\documentclass[twoside]{article}
\usepackage[accepted]{aistats2026}

\usepackage[utf8]{inputenc} %
\usepackage[T1]{fontenc}    %

\usepackage[USenglish]{babel}
\usepackage{amsmath}
\usepackage{amssymb}
\usepackage{amsthm}
\usepackage{thmtools}
\usepackage{mathtools}
\usepackage{comment}
\usepackage[shortlabels]{enumitem}
\usepackage{csquotes}
\usepackage{tikz}
\usepackage{tikz-qtree}
\usetikzlibrary{trees,positioning,shapes}
\usetikzlibrary{arrows.meta}
\usepackage{pgfplots}
\usepackage{booktabs}

\usepackage{chngcntr}
\usepackage{stmaryrd}
\usepackage{dsfont}  
\usepackage{scalerel,stackengine}  

\usepackage[page]{appendix}

\renewcommand{\appendixtocname}{Appendix Contents.}

\makeatletter
\let\oldappendix\appendices

\renewcommand{\appendices}{%
  \clearpage
  \renewcommand{\thesection}{\Roman{section}}
  \let\tf@toc\tf@app
  \addtocontents{app}{\protect\setcounter{tocdepth}{2}}
  \immediate\write\@auxout{%
    \string\let\string\tf@toc\string\tf@app^^J
  }
  \oldappendix
}%

\newcommand{\listofappendices}{%
  \begingroup
  \renewcommand{\contentsname}{\appendixtocname}
  \let\@oldstarttoc\@starttoc
  \def\@starttoc##1{\@oldstarttoc{app}}
  \tableofcontents
  \endgroup
}

\makeatother

\usepackage{hyperref}
\usepackage[capitalise,nameinlink]{cleveref}
\usepackage[retainorgcmds]{IEEEtrantools}

\usepackage{mleftright,xparse}

\pgfplotsset{compat=1.11}

\newlength\Origarrayrulewidth

\NewDocumentCommand\xDeclarePairedDelimiter{mmm}
{%
	\NewDocumentCommand#1{som}{%
		\IfNoValueTF{##2}
		{\IfBooleanTF{##1}{#2##3#3}{\mleft#2##3\mright#3}}
		{\mathopen{##2#2}##3\mathclose{##2#3}}%
	}%
}

\newcommand{\outcomment}[1]{}

\theoremstyle{plain}
\newtheorem{theorem}{Theorem}

\newtheorem{lemma}[theorem]{Lemma}
\newtheorem{proposition}[theorem]{Proposition}

\theoremstyle{definition}
\newtheorem{defenv}[theorem]{Definition}
\newtheorem{assumptionenv}[theorem]{Assumption}

\newtheorem{remenv}[theorem]{Remark}
\newtheorem{exenv}[theorem]{Example}

\newtheorem{notenv}[theorem]{Notation}

\newenvironment{remark}
  {\pushQED{\qed}\remenv}
  {\popQED\endremenv}
\newenvironment{definition}
  {\pushQED{\qed}\defenv}
  {\popQED\endremenv}
\newenvironment{example}
  {\pushQED{\qed}\exenv}
  {\popQED\endremenv}
\newenvironment{assumption}
  {\pushQED{\qed}\assumptionenv}
  {\popQED\endremenv}
\newenvironment{numbered_notation}
  {\pushQED{\qed}\notenv}
  {\popQED\endnotenv}

\newcommand{\bbE}{\mathbb{E}}

\newcommand{\bbN}{\mathbb{N}}

\newcommand{\bbP}{\mathbb{P}}

\newcommand{\bbR}{\mathbb{R}}
\newcommand{\bbS}{\mathbb{S}}

\newcommand{\bbZ}{\mathbb{Z}}

\newcommand{\bbone}{\mathds{1}}

\newcommand{\calA}{\mathcal{A}}

\newcommand{\calF}{\mathcal{F}}

\newcommand{\calH}{\mathcal{H}}
\newcommand{\calI}{\mathcal{I}}

\newcommand{\calL}{\mathcal{L}}

\newcommand{\calN}{\mathcal{N}}

\newcommand{\calP}{\mathcal{P}}
\newcommand{\calQ}{\mathcal{Q}}
\newcommand{\calR}{\mathcal{R}}
\newcommand{\calS}{\mathcal{S}}

\newcommand{\calV}{\mathcal{V}}
\newcommand{\calW}{\mathcal{W}}

\providecommand{\bfb}{\boldsymbol{b}}

\providecommand{\bfh}{\boldsymbol{h}}

\providecommand{\bfx}{\boldsymbol{x}}
\providecommand{\bfy}{\boldsymbol{y}}
\providecommand{\bfz}{\boldsymbol{z}}

\providecommand{\bfW}{\boldsymbol{W}}

\providecommand{\bftheta}{\boldsymbol{\theta}}

\providecommand{\bfSigma}{\boldsymbol{\Sigma}}

\newcommand{\diff}{\,\mathrm{d}}
\newcommand{\quot}[1]{\enquote{#1}}

\newcommand{\equalDef}{\coloneqq}
\newcommand{\defEqual}{\eqqcolon}
\newcommand{\equivDef}{:\Leftrightarrow}

\DeclareMathOperator{\sign}{sgn}

\DeclareMathOperator{\Cov}{Cov}

\DeclareMathOperator{\erf}{erf}
\DeclareMathOperator{\sigmoid}{sigmoid}
\DeclareMathOperator{\softplus}{softplus}
\DeclareMathOperator{\RBF}{RBF}
\DeclareMathOperator{\even}{even}
\DeclareMathOperator{\odd}{odd}

\newcommand{\eps}{\varepsilon}

\allowdisplaybreaks

\newcommand{\todo}[1]{\textcolor{blue}{[TODO: #1]}}

\newcommand{\ntk}{k^{\mathrm{NTK}}}  
\newcommand{\nngp}{k^{\mathrm{NNGP}}}

\newcommand{\kntk}{\kappa^{\mathrm{NTK}}}
\newcommand{\knngp}{\kappa^{\mathrm{NNGP}}}
\newcommand{\act}{\varphi}  
\newcommand{\scaledfn}[2]{#1_{\cdot #2}}
\newcommand{\acteven}{\act_{\mathrm{even}}}
\newcommand{\actodd}{\act_{\mathrm{odd}}}

\newcommand{\evenpart}[1]{\left(#1\right)_{\mathrm{even}}}
\newcommand{\oddpart}[1]{\left(#1\right)_{\mathrm{odd}}}

\newcommand{\feven}{f_{\mathrm{even}}}
\newcommand{\fodd}{f_{\mathrm{odd}}}
\newcommand{\geven}{g_{\mathrm{even}}}
\newcommand{\godd}{g_{\mathrm{odd}}}

\newcommand{\polyeq}{\simeq_{\mathrm{pol}}}

\newcommand{\s}{s}

\newcommand{\Eqref}[1]{Eq.~\eqref{#1}}

\stackMath
\newcommand\rwidehat[1]{%
\savestack{\tmpbox}{\stretchto{%
  \scaleto{%
    \scalerel*[\widthof{\ensuremath{#1}}]{\kern.1pt\mathchar"0362\kern.1pt}%
    {\rule{0ex}{\textheight}}
  }{\textheight}%
}{2.4ex}}%
\stackon[-6.9pt]{#1}{\tmpbox}%
}
\renewcommand{\widehat}{\rwidehat}

\setlist[enumerate]{nosep}
\setlist[itemize]{nosep}

\makeatletter
\newcommand\ackname{Acknowledgements}
\if@titlepage
   \newenvironment{acknowledgements}{%
       \titlepage
       \null\vfil
       \@beginparpenalty\@lowpenalty
       \begin{center}%
         \bfseries \ackname
         \@endparpenalty\@M
       \end{center}}%
      {\par\vfil\null\endtitlepage}
\else
   
\fi
\makeatother

\DeclareMathOperator{\sgn}{sgn}

\newcommand{\scal}[2]{\left\langle #1,#2\right\rangle}			

\newcommand{\w}[1]{\widehat{#1}}

\newcommand{\po}[2]{#1^{(#2)}}

\newcommand{\R}{\mathbb{R}}

\newcommand{\set}[1]{\{ #1\}}
\renewcommand*{\d}[1]{\, \operatorname{d}\! #1}
\newcommand{\E}{\mathbb E}
\newcommand{\Prob}{\mathbb{P}}

\xDeclarePairedDelimiter{\abs}{\lvert}{\rvert}
\xDeclarePairedDelimiter{\norm}{\lVert}{\rVert}
\newcommand{\setrange}[2]{\lbrace #1,\dots, #2\rbrace}
\xDeclarePairedDelimiter{\bra}{(}{)}
\xDeclarePairedDelimiter{\sbra}{[}{]}

\newcommand{\N}{\mathbb{N}}
\newcommand{\Z}{\mathbb{Z}}

\newcommand{\rand}[1]{Z^{#1}}

\newcommand{\ttt}{t}

\newcommand{\deven}[1]{\deg_{\text{ev}}\left(#1\right)}
\newcommand{\dodd}[1]{\deg_{\text{od}}\left(#1\right)}
\newcommand{\dseven}[1]{\deg_{\text{ev}}\left(#1\right)}
\newcommand{\dsodd}[1]{\deg_{\text{od}}\left(#1\right)}
\newcommand{\degree}[1]{\deg\left(#1\right)}
\newcommand{\dequal}{=_{\deg}}

\newcommand{\hdeven}[1]{\deg^{\text{Her}}_{\text{ev}}\left(#1\right)}
\newcommand{\hdodd}[1]{\deg^{\text{Her}}_{\text{od}}\left(#1\right)}

\newcommand{\hdegree}[1]{\deg^{\text{Her}}\left(#1\right)}
\newcommand{\hdequal}{=_{\deg}^{\text{Her}}}
\newcommand{\phdequal}{=_{\deg}^{\text{Pow,Her}}}

\newcommand{\funevenodd}{\calF_{\text{even/odd}}}
\newcommand{\Feven}{\calF_{\mathrm{even}}}
\newcommand{\Fodd}{\calF_{\mathrm{odd}}}
\newcommand{\eee}{\dseven{\w\act}}
\newcommand{\ooo}{\dsodd{\w\act}}

\usepackage{hyperref}
\usepackage{url}
\usepackage{multirow}
\usepackage[draft]{fixme}

\usepackage[round]{natbib}
\renewcommand{\bibname}{References}

\begin{document}

\begingroup
\renewcommand{\thefootnote}{\fnsymbol{footnote}}

\twocolumn[

\aistatstitle{Beyond ReLU: How Activations Affect Neural Kernels and Random Wide Networks}

\aistatsauthor{David Holzmüller\footnotemark[1] \And Max David Schölpple\footnotemark[1]}	

\aistatsaddress{INRIA \And University of Stuttgart}]

\runningauthor{David Holzmüller${}^*$, Max David Schölpple${}^*$}

\footnotetext[1]{Equal contribution.}
\endgroup
\setcounter{footnote}{0}

\begin{abstract}
In recent years, the neural tangent kernel (NTK) and neural network Gaussian process kernel (NNGP) have given theoreticians tractable limiting cases of fully connected neural networks. However, the property of these kernels are poorly understood for activation functions other than powers of the ReLU. 
Our main contribution is a characterization of the RKHS of these kernels for activation functions whose only non-smoothness is at zero. 
This extends existing theory to numerous commonly used activation functions such as SELU, ELU, or LeakyReLU.
Additionally, we analyze a broad set of special cases such as missing biases, two-layer networks, or polynomial activations. 
Our results show that a broad class of not infinitely smooth activations generate equivalent RKHSs at different network depths, depending only on the degree of the non-smoothness up to equivalence. On the other hand, the RKHS generated by polynomial activations depends on the network depth. 
Finally, we derive results for the smoothness of NNGP sample paths, characterizing the smoothness of infinitely wide neural networks at initialization.
\end{abstract}

\section{INTRODUCTION}

Despite great efforts, our theoretical understanding of when and why deep learning works remains limited, and much of it is distributed across fragmented case studies that do not integrate into a bigger theory. For example, the influence of activation functions on the training dynamics and generalization behavior of deep learning methods is unclear.

One of the most prominent approaches towards deep learning theory is the study of infinite-width limits of neural networks \citep{neal_priors_1996, jacot_neural_2018, yang2021tensor}. Different parametrizations of neural networks lead to different behavior at infinite width \citep{yang2021tensor}. Here, we will focus on the neural tangent parametrization (NTP), whose infinite-width limit leads to kernel behavior \citep{jacot_neural_2018}. While the kernel regime cannot model all aspects of neural network behavior \citep{yehudai2019power,ortiz2021can,vyas2022limitations,wenger2023disconnect}, it is one of the models most amenable to theoretical analysis. In particular, at initialization, the function represented by certain infinite-width neural networks is distributed like a Gaussian process with the so-called neural network Gaussian process (NNGP) kernel \citep{neal_priors_1996, daniely_toward_2016, lee_deep_2018, matthews_gaussian_2018}. The training then follows the dynamics of kernel gradient descent with the so-called neural tangent kernel \citep[NTK,][]{jacot_neural_2018, lee_wide_2019}.

To understand the behavior of neural networks in the kernel regime, including training dynamics and generalization, it is essential to understand the properties of the associated kernels themselves. In particular, the structure of the reproducing kernel Hilbert spaces (RKHSs) corresponding to these kernels is central to further theoretical analysis. Here, we limit ourselves to fully connected neural networks, which are sufficient to study the influence of activation functions, and are still practically relevant by themselves \citep{holzmuller2024better, gorishniy2025tabm, ye2024closer, erickson2025tabarena}. While these kernels have been analyzed for (powers of) the ReLU activation and infinitely smooth activation functions \citep{chen_deep_2020, bietti_deep_2021, vakili_information_2023}, not much is known about their structure for other activation functions.

\paragraph{Contribution}

Our main contribution is \Cref{thm:main_result}, illustrated in \Cref{rem:simple_case}, where we prove a general result analyzing the structure of the RKHS corresponding to the NTK and NNGP kernels on the sphere $\bbS^d$ of fully connected neural networks. Our proof, of which an overview is given in \Cref{sec:appendix:overview}, builds upon the analysis of \cite{bietti_deep_2021}, but generalizes existing results in multiple ways:
\begin{itemize}
\item We study a significantly larger class of activations. In particular, we provide exact characterizations for typical activation functions that are infinitely smooth everywhere except at zero. We also provide new results for polynomial activation functions, and we more precisely analyze the conditions under which results for infinitely smooth activation functions hold.
\item We study a more general class of fully connected neural networks, covering the cases of normally distributed bias, zero-initialized bias, and no bias. We also study all special cases arising for two-layer networks.
\item We analyze all cases for both NNGP and NTK (except discontinuous activation functions, for which the NTK is not well-defined).
\end{itemize}

As an additional contribution, in \Cref{theorem:path_smoothness}, we show how the results on the RKHS translate to the smoothness of functions sampled from the Gaussian process with the NNGP kernel, and hence the smoothness of infinitely wide neural networks at initialization.

After discussing preliminaries (\Cref{sec:preliminaries}), we introduce our main results in \Cref{sec:main_results} and \Cref{sec:path_smoothness}. \Cref{sec:implications} discusses some implications of our results on the broader infinite-width theory, such as the (non-)benefit of depth, training dynamics, and generalization. Related work is discussed in \Cref{sec:related_work}, before \Cref{sec:conclusion} concludes with a discussion of possible extensions.

\section{PRELIMINARIES} \label{sec:preliminaries}
We denote the $d$-dimensional unit sphere as $\bbS^{d}\subset \R^{d+1}$, and the square integrable functions over a measure space $(X,\calA, \mu)$ as $L_2(X) \equalDef L_2(\mu) = \{f \mid \int_X f^2 ~d\mu < \infty\}$. For $s \geq 0$, we use $H^s(X)$ to denote the Sobolev space of order $s$, which essentially consists of the functions $f$ for which all (fractional) derivatives up to order $s$ are in $L_2(X)$.

\paragraph{Neural kernels} 
For $d \geq 1$, we consider fully connected neural networks (NNs) in neural tangent parametrization (NTP) operating on inputs $\bfx \in \bbR^{d_0}, d_0 \equalDef d+1$. Typically, neural networks use componentwise-normalized input vectors $\bfx \in \bbR^{d_0}$ such that $\|\bfx\|_2 \approx \sqrt{d_0}$. Following related theoretical literature, we instead consider inputs on the sphere, $\bfx \in \bbS^d = \{\bfz \in \bbR^{d_0} \mid \|\bfz\|_2 = 1\}$, and compensate for this by rescaling the first layer by $\sqrt{d_0}$, which leads to equivalent learning dynamics with gradient descent.
We consider the following network architecture:

\begin{definition}\label{def:network}
	Let $L \geq 2$ be the number of layers, let $\sigma_w > 0$, and let $\sigma_b, \sigma_i \geq 0$.
	The network is the function $f_{\bftheta}: \bbR^{d_0} \to \bbR^{d_L}$ given by $f_{\bftheta}(\bfx^{(0)}) \equalDef \bfz^{(L)}$, where	
	\begin{IEEEeqnarray*}{+rCl+x*}
		\bfz^{(1)} & \equalDef & \sigma_w \bfW^{(1)} \bfx^{(0)} + \sigma_b \bfb^{(1)} \in \bbR^{d_1}\\
		\bfz^{(l)} & \equalDef & \frac{\sigma_w}{\sqrt{d_{l-1}}} \bfW^{(l)} \bfx^{(l-1)} + \sigma_b \bfb^{(l)} \in \bbR^{d_l}, \\
		\bfx^{(l-1)} & \equalDef & \act(\bfz^{(l-1)}) \in \bbR^{d_l}. \qquad (l \geq 2)
	\end{IEEEeqnarray*}
	We assume that all parameters are initialized independently as $\bfW^{(l)}_{jk} \sim \calN(0, 1)$ and $b^{(l)}_j \sim \calN(0, \sigma_i^2)$.
\end{definition} 
\cref{def:network} subsumes the three cases of
\begin{itemize}
	\item bias-free networks by setting $\sigma_b = 0$,
	\item zero-initialized biases by setting $\sigma_b \neq 0$, $\sigma_i = 0$,
	\item randomly initialized biases by setting $\sigma_b, \sigma_i \neq 0$.
\end{itemize}
For $\sigma_i^2=1$, the recursive formula for NTK and NNGP kernels are derived in \cite{lee_wide_2019}.
We derive formulas for the different possible values of $\sigma_i,\sigma_b,\sigma_w$ in \cref{lemma:limit_ntk_formula}.
For the NNGP-kernel, the recursion is given by
	\begin{align*}
		\nngp_1(\bfx, \bar{\bfx}) \equalDef&\, \sigma_b^2 \sigma_i^2 + \sigma_w^2 \langle \bfx, \bar{\bfx}\rangle \\
		\nngp_{L}(\bfx, \bar{\bfx})  =&\, \sigma_b^2 \sigma_i^2 + \sigma_w^2 \bbE_{(u, v) \sim \bfSigma_{L-1}(\bfx, \bar{\bfx})} [\act(u) \act(v)]\\
		\bfSigma_{L}(\bfx, \bar{\bfx})  
		=&\, \begin{pmatrix}
			\nngp_{L}(\bfx, \bfx) & \nngp_{L}(\bfx, \bar{\bfx}) \\
			\nngp_{L}(\bar{\bfx}, \bfx) & \nngp_{L}(\bar{\bfx}, \bar{\bfx})
		\end{pmatrix}
		~,
	\end{align*}
	For the NTK, we refer to \cref{lemma:limit_ntk_formula}.

To analyze NTK and NNGP kernels, we restrict them to the sphere $\bbS^d$, where they are dot-product kernels due to the rotation-invariance induced by the Gaussian initialization of the weight matrices $\bfW^{(l)}$.
\paragraph{Dot-product kernels on the sphere} Let $\kappa: [-1, 1] \to \bbR$ such that $k \equalDef k_{\kappa, d}: \bbS^d \times \bbS^d \to \bbR, (\bfx, \bfx') \mapsto \kappa(\langle \bfx, \bfx' \rangle)$ is a (positive semidefinite) kernel. Then, $k$ is called a \emph{dot-product kernel}. 
We leverage the theory of dot-product kernels on spheres to study the RKHS $\calH_k$ associated with $k$ \citep[see e.g.][]{bietti_deep_2021, hubbert_sobolev_2023, haas_mind_2023}. In particular, these RKHSs can be characterized through the eigenvalues of the associated integral operator $T_k: L_2(\bbS^d) \to L_2(\bbS^d)$ given by  %
\begin{IEEEeqnarray*}{+rCl+x*}
(T_k f)(\bfx) = \int_{\bbS^d} k(\bfx, \bfx') f(\bfx') \diff \bfx'~.
\end{IEEEeqnarray*}
It can be diagonalized as
\begin{IEEEeqnarray*}{+rCl+x*}
T_k f & = & \sum_{l=0}^\infty \sum_{i=1}^{N_{l,d}} \mu_l Y_{l,i} \langle Y_{l,i}, f \rangle_{L_2(\bbS^d)}~,
\end{IEEEeqnarray*}
where $\{Y_{l,1}, \hdots, Y_{l, N_{l, d}}\}$ is an arbitrary orthonormal basis of the space of spherical harmonics (a subset of polynomials) of degree $l$ in $L_2(\bbS^d)$, cf.~\cite{mueller2006spherical}, but the details are not important in the following. The eigenfunctions $Y_{l, i}$ are the same for all dot-product kernels but the eigenvalues $\mu_l = \mu_l(k) = \mu_l(\kappa, d)$ can be different and characterize the RKHS $\calH_k$ associated with $k$:
\begin{IEEEeqnarray*}{+rCl+x*}
\calH_k = \left\{ \sum_{l=0}^\infty \sqrt{\mu_l} \sum_{i=1}^{N_{l,d}} a_{l,i} Y_{l,i} ~\middle|~ \sum_{l=0}^\infty \sum_{i=1}^{N_{l,d}} a_{l,i}^2 < \infty\right\}~.
\end{IEEEeqnarray*}
Its inner product is
\begin{IEEEeqnarray*}{+rCl+x*}
& & \left\langle \sum_{l=0}^\infty \sqrt{\mu_l} \sum_{i=1}^{N_{l,d}} a_{l,i} Y_{l,i}, \sum_{l=0}^\infty \sqrt{\mu_l} \sum_{i=1}^{N_{l,d}} b_{l,i} Y_{l,i}\right\rangle_{\calH_k} \\
& = & \sum_{l=0}^\infty \sum_{i=1}^{N_{l,d}} a_{l,i}b_{l,i}~.
\end{IEEEeqnarray*}
Hence, the RKHS as a set is characterized by 1) the asymptotic behavior of the eigenvalues and 2) the set of indices for which the eigenvalues are nonzero. 
By Proposition 3 in \cite{schoelpple_inside_2025}, if two RKHSs $\calH_1, \calH_2$ are equal as sets, then their norms are equivalent. In this case, we call $\calH_1, \calH_2$ \emph{equivalent}, and write $\calH_1 \cong \calH_2$. Polynomial eigenvalue decays lead to RKHSs equivalent to Sobolev spaces:
\begin{lemma}[Sobolev spaces on the sphere]\label{lemma:sobolev_spherical}
Let $d\in\N$ and $s>d/2$. The RKHS $\calH_k$ corresponding to the dot-product kernel $k$ on $\bbS^d$ is equivalent to the Sobolev space $H^s(\bbS^d)$ if and only if there exist $c, C > 0$ with $c(l+1)^{-2s} \leq \mu_l \leq C(l+1)^{-2s}$ for all $l \geq 0$. %
\end{lemma}
\begin{proof}
Depending on the definition of $H^s(\bbS^d)$, this follows either directly or from classical theory, see \cite{hubbert_sobolev_2023} and Lemma B.1 in \cite{haas_mind_2023}.
\end{proof}
Sobolev spaces on the sphere consist of those functions $f$ that locally, when composed with a smooth chart $\varphi: U \subseteq \bbR^d \to V \subseteq \bbS^d$, are Sobolev functions on $\bbR^d$, e.g.\ \cite{haas_mind_2023}, p.\ 34. %

\section{THE STRUCTURE OF NEURAL KERNELS} \label{sec:main_results}

Before we introduce our main theorem, we need to impose some assumptions on the involved activation functions.

\begin{definition}[Functions with polynomially bounded derivatives] \label{def:sinfty}
Let $I \subseteq \bbR$ be an interval. We define the set $\calS^{(\infty)}(I)$ to contain all functions $\varphi_I \in C^\infty(I)$ with polynomially bounded derivatives. In other words: For all $m \in \bbN_0$, there exist $a_m, b_m, q_m > 0$ such that for all $x \in I$, $|\varphi_I^{(m)}(x)| \leq a_m |x|^{q_m} + b_m$.
\end{definition}

Polynomial boundedness for $m \in \{0, 1\}$ is required to get the NTK formulas from \cite[Theorem 7.2, Box 1]{yang_tensor_programs_II_ntk_for_any_architecture}; for higher derivatives we only need $\varphi_I^{(m)} \in L_2(\calN(0, \sigma^2))$ for all $\sigma^2 > 0$.

\begin{assumption}[Activation function] \label{ass:act}
We assume that $\varphi: \bbR \to \bbR$ is of the form\footnote{The definition at zero ensures that $\varphi$ is continuous when $\varphi_+$ and $\varphi_-$ allow, and that its even/odd parts are zero everywhere whenever they are zero almost everywhere.}
\begin{IEEEeqnarray*}{+rCl+x*}
\varphi(x) & = & \begin{cases}
\varphi_+(x) &, x > 0 \\
\varphi_-(x) &, x < 0 \\
\frac{1}{2} (\lim_{x \searrow 0}\varphi_+(x) + \lim_{x \nearrow 0} \varphi_-(x)) &, x = 0
\end{cases}
\end{IEEEeqnarray*}
with $\varphi_- \in \calS^{(\infty)}((-\infty, 0))$ and $\varphi_+ \in \calS^{(\infty)}((0, \infty))$ and that $\varphi$ is not the zero function. 
\end{assumption}

\begin{restatable}[All common activation function satisfy \Cref{ass:act}]{proposition}{propActAssumptions}  \label{prop:act_assumptions}
\leavevmode
\begin{enumerate}[(a)]
\item If $f$ is constructed by addition, multiplication, and composition of polynomials, $\sigmoid$, $\tanh$, $\softplus$, $\sin$, $\cos$, $\RBF$ and $\Phi$, then $f \in \calS^{(\infty)}(\bbR)$. Moreover, the functions $g(x) = \exp(ax)$ are in $\calS^{(\infty)}((-\infty, 0])$ for every $a \geq 0$.
\item \Cref{ass:act} is satisfied for all activation functions from \Cref{table:acts}. %
It is also satisfied for all $\phi \in \calS^{(\infty)}(\bbR)$.
\end{enumerate}
\end{restatable}

We prove \Cref{prop:act_assumptions} in \Cref{sec:integrability}.

\newcommand{\smoothness}[1]{\operatorname{smoothness}(#1)}

\begin{definition}[Smoothness of an activation function]\label{def:smoothness_of_an_activation_function}
For an activation function $\varphi$ as in \Cref{ass:act}, we define its smoothness as
\begin{IEEEeqnarray*}{+rCl+x*}
\smoothness{\act}  \equalDef  \inf&\{&m \in \bbN_0 \mid  \lim_{t \searrow 0} \act^{(m)}(t) \\
& &\neq \lim_{t \nearrow 0} \act^{(m)}(t)\} \in \bbN_0 \cup \{\infty\}. & \qedhere
\end{IEEEeqnarray*}
\end{definition}

\begin{example}[Smoothness of the ReLU activation]
	The ReLU activation function $\act(x) = \max\{0,x\}$ fulfills \cref{ass:act} and decomposes as $\act_+(x) = x\ , x>0$, $\act_- (x) = 0 \ , x<0$. Hence we have
	\begin{align*}
		\lim_{t \searrow 0} \act(t) &= 0 = \lim_{t \nearrow 0} \act(t) \\
		 \text{but} \qquad \lim_{t \searrow 0} \act'(t) & = 1 \neq 0 = \lim_{t \nearrow 0} \act'(t)~,
	\end{align*}
	implying $\smoothness{\mathrm{ReLU}} = 1$.
\end{example}

\begin{table}
	\centering
	\caption{Activation functions and their smoothness (\cref{def:smoothness_of_an_activation_function}). 
		Here, $\calP_m$ means that the function is a polynomial of degree $m$. $\calP_{-\infty}$ means that the function is zero. For a more extensive overview, we refer to \cite{dubey2022activation}.
		In the typical cases i) of \cref{thm:main_result}, a smoothness $s \in (1, \infty)$ of $\act$ yields $\calH_{\nngp_L} \cong H^{d/2+s+1/2}(\bbS^d)$ and $\calH_{\ntk_L} \cong H^{d/2 +s -1/2}$.
		} \label{table:acts}
	\resizebox{\columnwidth}{!}{
	\begin{tabular}{cccccc}
		\toprule
		Activation & \multicolumn{2}{c}{Formula} & \multicolumn{3}{c}{Smoothness $s$} \\ %
		& $x < 0$ & $x \geq 0$ & $\varphi$ & $\acteven$ & $\actodd$ \\
		\midrule
		ReLU & 0 & $x$ & 1 & 1 & $\calP_1$ \\
		LeakyReLU ($\eps \neq -1$) & $-\varepsilon x$ & $x$ & 1 & 1 & $\calP_1$ \\
		SELU & $\lambda\alpha (e^x - 1)$ & $\lambda x$ & 1 & 1 & 2 \\
		ELU ($\alpha \neq 1$) & $\alpha(e^x - 1)$ & $x$ & 1 & 1 & 2 \\
		ELU ($\alpha = 1$), CELU & $\alpha(e^{x/\alpha} - 1)$ & $x$ & 2 & 3 & 2 \\
		RePU, power $m$ even & $0$ & $x^m$ & $m$ & $\calP_m$ & $m$ \\
		RePU, power $m$ odd & $0$ & $x^m$ & $m$  & $m$ & $\calP_m$ \\
		Heaviside & \multicolumn{2}{c}{ $\frac 12 \mathds{1}_{\{0\}}(x) +\mathds{1}_{\R_{>0}}(x)$} & $0$ & $\calP_0$ & $0$ \\
		$\tanh$ & \multicolumn{2}{c}{$(e^x - e^{-x})/(e^x + e^{-x})$} & $\infty$ & $\calP_{-\infty}$ & $\infty$ \\
		Sigmoid & \multicolumn{2}{c}{$1/(1 + e^{-x})$} & $\infty$ & $\calP_{0}$ & $\infty$ \\
		GeLU &  \multicolumn{2}{c}{$\frac 12 x (1+\erf(x/\sqrt 2))$} & $\infty$& $\infty$ & $\calP_1$\\
		SiLU & \multicolumn{2}{c}{$x/(1+e^{-x})$} &$ \infty $& $ \infty $& $\infty$\\
		RBF &  \multicolumn{2}{c}{$\exp(-x^2)$} & $\infty$ & $\infty$ & $\calP_{-\infty}(\bbS^d)$. \\
		\bottomrule
	\end{tabular}
	}
\end{table}
\begin{remark}[Location of the non-smoothness]
We conjecture that our analysis could be extended to the case of non-smoothnesses at $b \ne 0$ with similar results by extending a part of our proofs, cf.\ \cref{rem:restricting_analysis_to_origin}.
\end{remark}
Our main result, \cref{thm:main_result}, contains multiple cases. 
The \quot{regular} cases are discussed afterward in \Cref{rem:simple_case}. %
In some special bias-free cases we observe a phenomenon of parity: The even/odd parts of the functions contained in the RKHSs $\calH_k$ of the neural kernels depend on the even/odd parts of the activation $\act$, commonly defined as
\begin{align*}
	\acteven(x) &\equalDef \frac{\act(x) +\act(-x)}{2}~,\\
	 \actodd(x) &\equalDef \frac{\act(x)-\act(-x)}{2}~,
\end{align*}
fulfilling $\act = \acteven + \actodd$.
For example, $\mathrm{ReLU}_{\mathrm{even}}(x) = \frac{1}{2}|x|$ and $\mathrm{ReLU}_{\mathrm{odd}}(x) = \frac{1}{2}x$. In the bias-free case ($\sigma_b^2 = 0$), the function $f_{\act}$ represented by a two-layer network with activation $\varphi$ can be written as $f_{\act} = f_{\acteven} + f_{\actodd}$, where $f_{\acteven}$ is even and $f_{\actodd}$ is odd. Hence, $(f_{\mathrm{ReLU}})_{\mathrm{odd}} = f_{\mathrm{ReLU}_{\mathrm{odd}}}$ is a linear function, which explains the results of \cite{bietti_inductive_2019} where the odd eigenvalues satisfy $\mu_1 > 0$, $\mu_3 = \mu_5 = \hdots = 0$. A much more general version of this calculation is performed in \Cref{prop:even_odd_kernels}.

By $\calP_m$ we denote the polynomials of degree at most $m$. 
 We additionally define 
\begin{align*}
	\Feven &\equalDef \{ f: \bbS^d \to \bbR \mid f \text{ is even}\}~,\\
	\Fodd &\equalDef \{f:\bbS^d \to \bbR \mid f \text{ is odd}\}~.
\end{align*}
\begin{restatable}[Main result, summary below]{theorem}{thmMainResult} \label{thm:main_result}
	Let the activation $\act$ fulfill \cref{ass:act} and let $s\equalDef \smoothness{\act}$.
	\\ 
	\textbf{NNGP:}
	\begin{enumerate}[i)]
		\item Case $\sigma_b^2\sigma_i^2>0$ or both $L\ge 3$ and $\act$ is neither even nor odd.
		\begin{enumerate}[a)]
			\item If $s=0$, then $\calH_{\nngp_L} \cong H^{d/2 +2^{1-L}}(\bbS^d)$.
			\item If $ 1\le  s <\infty$, then $\calH_{\nngp_L} \cong H^{d/2 + s + 1/2}(\bbS^d)$.
			\item If $ s=\infty$ and $\act$ is not a polynomial, then $\calH_{\nngp_L} \subset H^t(\bbS^d)$ for all $t\in\bbR$ and $\calH_{\nngp_L}$ contains all polynomials. 				
			\item If $s=\infty$ and $\act$ is a polynomial of degree $m$, then $\calH_{\nngp_L} \cong 
			\calP_{m^{L-1}}$.
		\end{enumerate}
		\item Case $\sigma_b^2\sigma_i^2 =0$ and $\act$ is even or odd. 
		\begin{enumerate}[a)]
			\item If $s=0$, then $\calH_{\nngp_L} \cong H^{d/2 +2^{1-L}}(\bbS^d) \cap \funevenodd$.
			\item If $ 1\le  s <\infty$, then $\calH_{\nngp_L} \cong H^{d/2 + s + 1/2}(\bbS^d)\cap \funevenodd$.
			\item If $ s=\infty$ and $\act$ is not a polynomial, then $\calH_{\nngp_L} \subset H^t(\bbS^d)\cap\funevenodd$ for all $t\in\bbR$ and $H_{\nngp_L}$ contains all even/odd polynomials.
			\item If $s=\infty$ and $\act$ is a polynomial of degree $m$, then 
$\calH_{\nngp_L} \cong \calP_{m^{L-1}} \cap \funevenodd$.
		\end{enumerate}
		\item 
		Case $\sigma_b^2\sigma_i^2 =0$ and $L=2$. We have
		\begin{align}
			\label{eq:nngp_complicated}
			\calH_{\nngp_L} \cong \calH_{\nngp_{\acteven,L}} \oplus \calH_{\nngp_{\actodd,L}}
		\end{align}
		where the RKHSs 
		$\calH_{\nngp_{\acteven,L}}$ and $\calH_{\nngp_{\actodd,L}}$ of the even/odd activation functions $\acteven$ and $\actodd$ can be found by Case ii) with corresponding $s\equalDef \smoothness{\acteven}$ respectively $s\equalDef \smoothness{\actodd}$.	
	\end{enumerate}
	\textbf{NTK:}
	\begin{enumerate}[i)]
		\item Case $\sigma_b^2>0$ or both $L\ge 3$ and $\act$ is neither even nor odd.
		\begin{enumerate}[a)]
			\item If $ 1\le  s <\infty$, then $\calH_{\ntk_L} \cong H^{d/2 + s - 1/2}(\bbS^d)$.
			\item If $ s=\infty$ and $\act$ is not a polynomial, then $\calH_{\ntk_L} \subset H^t(\bbS^d)$ for all $t\in \bbR$ and $\calH_{\ntk_L}$ contains all polynomials.
			\item If $s=\infty$ and $\act$ is a polynomial of degree $m$, then 
			$\calH_{\ntk_L} \cong \calP_{m^{L-1}}$. 
		\end{enumerate}
		\item Case $\sigma_b^2=0$ and $\act$ is even or odd. 
		\begin{enumerate}[a)]
			\item If $ 1\le  s <\infty$, then $\calH_{\ntk_L} \cong H^{d/2 + s - 1/2}(\bbS^d)\cap \funevenodd$.
			\item If $ s=\infty$ and $\act$ is not a polynomial, then $\calH_{\ntk_L} \subset H^t(\bbS^d)\cap\funevenodd$ for all $t\in \bbR$ and $H_{\ntk_L}$ contains all even/odd polynomials.
			\item 
			If $s=\infty$ and $\act$ is a polynomial of degree $m$, then  $\calH_{\ntk_L}= \calP_{m^{L-1}}\cap \funevenodd$.
		\end{enumerate}
		\item 
		Case $\sigma_b^2 =0$ and $L=2$. We have
		\begin{align}
			\label{eq:ntk_complicated}
			\calH_{\ntk_L}\cong \calH_{\ntk_{\acteven,L}} \oplus \calH_{\ntk_{\actodd,L}}
		\end{align}
		where the RKHSs $\calH_{\ntk_{\acteven,L}}$ and $\calH_{\ntk_{\actodd,L}}$ of the even/odd activation functions $\acteven$ and $\actodd$ and can be found by Case ii) with corresponding $s\equalDef \smoothness{\acteven}$ respectively $s\equalDef \smoothness{\actodd}$.
	\end{enumerate}
Note that $H^t(\bbS^d) \cap \Feven$ and $H^t(\bbS^d) \cap \Fodd$ are closed subspaces of $H^t(\bbS^d)$ for any $t$. We equip them with the restricted norm.
If in the sub-cases iii) $\acteven =0 $ or $\actodd =0$ occurs, the RKHS corresponding to that even/odd part in \cref{eq:nngp_complicated} respectively \cref{eq:ntk_complicated} is $\calP_{-\infty}=\{0\}$.
\end{restatable}
We show a proof sketch in \Cref{sec:proof_sketch}, a more detailed proof overview in \Cref{sec:appendix:overview}, and the proof for deriving this version of the theorem in \Cref{sec:appendix:eigenvalue_decay}.
The smoothness of many common activation functions, as well as the smoothness of their even/odd parts, can be found in \cref{table:acts}.
Note that the results of \Cref{thm:main_result} can be expressed in terms of the eigenvalues $\mu_l$ of the integral operator instead of the RKHSs as discussed in \Cref{sec:preliminaries}; the eigenvalue-based formulation can be found in \Cref{thm:kernel_ev_rates}.

\begin{figure}
\includegraphics[width=\columnwidth]{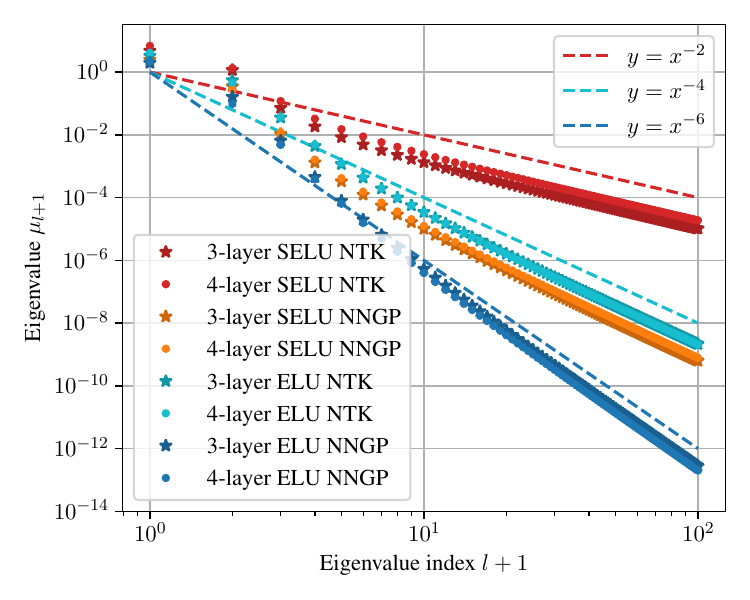}
\caption{\textbf{Eigenvalues $\mu_l$ of different neural kernels on $\bbS^2$.} We use $\sigma_w=\sigma_b=\sigma_i=1$. We use a custom method to numerically compute dual activations, described in \Cref{sec:appendix:activation_quadrature}. %
} \label{fig:deep_spectra}
\end{figure}

\begin{remark}[General takeaway] \label{rem:simple_case}
\Cref{thm:main_result} is most easy to understand for the case where
\begin{itemize}
\item the NN has $L \geq 2$ layers, contains biases ($\sigma_b^2 > 0$), and in the NNGP case the biases are initialized with nonzero variance ($\sigma_i^2 > 0$), or
\item the NN has $L \geq 3$ layers and the activation function $\varphi$ is neither even nor odd.
\end{itemize}
In this case, the decay of the eigenvalues $\mu_l$ and the corresponding RKHS depend on the smoothness $s$ of the activation function $\varphi$ as follows:
\begin{itemize}
\item If the activation function is discontinuous ($s=0$), the RKHS of the NNGP kernel is equivalent to a Sobolev space of smoothness $d/2+2^{1-L}$, while the NTK is not defined. 
\item If the activation function has finite smoothness $1 \leq s < \infty$, the RKHSs are equivalent to Sobolev spaces of order $d/2+s-1/2$ for the NTK and $d/2+s+1/2$ for the NNGP. This applies to ReLU, LeakyReLU, ELU ($\alpha \neq 1$) and SELU with $s=1$ and to CELU with $s=2$. \Cref{fig:deep_spectra} shows that the corresponding eigenvalue decays $\mu_l = d+2s-1$ for the NTK and $\mu_l = d+2s+1$ (cf. \Cref{lemma:sobolev_spherical}) are attained in numerical experiments.
Rectified power unit activations $\max\{0,x\}^s$ have smoothness $s\in\N$, yielding a family of Sobolev spaces. Our results show that this is due to the behavior at zero and that the superlinear growth of these functions for $s \geq 2$ is not necessary.
\item If the activation function is infinitely smooth ($s=\infty$) but not a polynomial, all $\mu_l$ are nonzero but decay faster than any inverse polynomial. Hence, the RKHSs are contained in all Sobolev spaces, but the kernels are universal. This applies for example to the GELU, SiLU/Swish, Mish, softplus, sigmoid, and tanh activation functions (though tanh is an odd function, so the conclusion only holds with biases).
\item If the activation function is a polynomial, for NTK and NNGP only finitely many eigenvalues $\mu_l$ are nonzero and the RKHSs only contain polynomials, whose maximum degree grows exponentially with the depth of the network. %
\qedhere
\end{itemize}
\end{remark}

\begin{figure}
\includegraphics[width=\columnwidth]{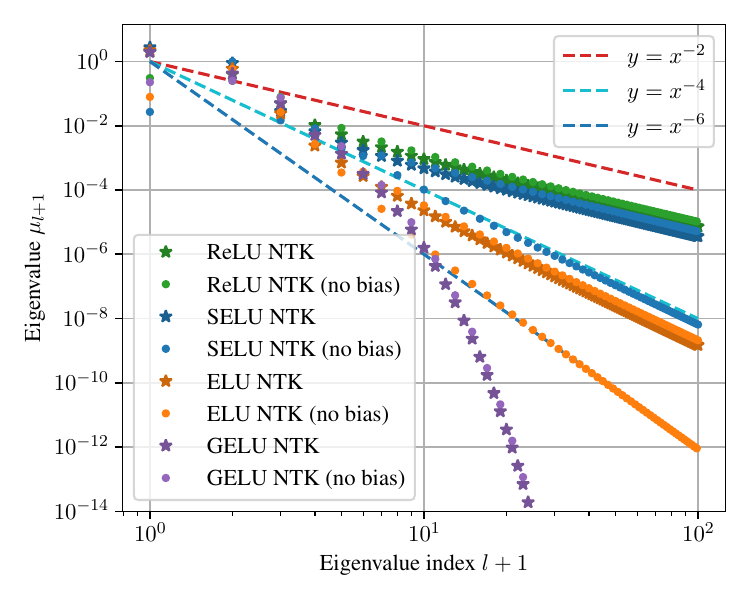}
\caption{\textbf{Eigenvalues $\mu_l$ of different two-layer NTKs on $\bbS^2$.} We use $\sigma_w=\sigma_b=\sigma_i=1$ for the case with biases and set $\sigma_b=\sigma_i=0$ for the no-bias case. We use a custom method to numerically compute dual activations, described in \Cref{sec:appendix:activation_quadrature}.
} \label{fig:shallow_spectra}
\end{figure}

\begin{example}[Special cases]
As mentioned above, two-layer networks without bias ($\sigma_b^2 = 0$ for NTK or $\sigma_b^2\sigma_i^2 = 0$ for NNGP) can have special behavior. Following \Cref{table:acts}, we obtain for CELU the NTK RKHS $(H^{d/2 + 5/2} \cap \Feven) \oplus (H^{d/2 + 3/2} \cap \Fodd)$, so the even parts of functions are smoother than the odd parts of functions. For ELU ($\alpha \neq 1$) and SELU, we obtain $(H^{d/2 + 1/2} \cap \Feven) \oplus (H^{d/2 + 3/2} \cap \Fodd)$. For ReLU and LeakyReLU, the odd parts are just linear functions. For RePU, one of the two parts are polynomials up to a certain degree. \Cref{fig:shallow_spectra} shows the corresponding behavior of the eigenvalues $\mu_l$ for the two-layer NTKs for different activations: In the bias-free case, the odd SELU eigenvalues decay faster than the even ones, while the even ELU ($\alpha = 1$) eigenvalues decay faster than the odd ones. For bias-free ReLU and GELU, the eigenvalues $\mu_l$ for odd $l \geq 3$ are zero.

The other special case is for even or odd activations and arbitrarily deep networks without bias. For example, $\tanh$ is an odd activation, hence the corresponding NNGP and NTK RKHSs will only contain odd functions.
\end{example}

\begin{remark}[NTKs and NNGPs of non-smooth activations are equivalent to Matérn kernels]
\cite{chen_deep_2020} and \cite{bietti_deep_2021} showed that on $\bbS^d$, the RKHS of the deep ReLU NTK is equivalent to that of the Laplace kernel. Theorem 1 in \cite{vakili2021uniform} further shows that the RKHSs of deep RePU NTKs and NNGPs are equivalent to the ones of Matérn kernels, which are equivalent to Sobolev spaces. Specifically, for the Matérn kernel $k_\nu$ of order $\nu > 0$ on the sphere $\bbS^d$, it holds that $\calH_{k_\nu} \cong H^{d/2 + \nu}(\bbS^d)$ (see also Proposition 5.2 (c) in \citealt{hubbert_sobolev_2023}). For other not-infinitely-smooth activations, our \Cref{thm:main_result} shows that their RKHSs are generally also equivalent to Sobolev spaces, and hence the corresponding kernels are equivalent to Matérn kernels.
\end{remark}

\Cref{sec:implications} discusses implications of our results, for example, on the equivalence of RKHSs for different network depths.

\section{PROOF SKETCH FOR THE MAIN THEOREM} \label{sec:proof_sketch}

Here, we provide a short overview of our main theorem. A more detailed overview can be found in \Cref{sec:appendix:overview}.
We obtain \cref{thm:main_result} by case distinction from \cref{thm:kernel_ev_rates}, the central technical result of this paper, which yields the eigenvalue asymptotics of the neural kernel $k$ in dependence of the smoothness of the activation function $\act$. 
Those eigenvalue asymptotics then directly yield the equivalent Sobolev space $\calH_k \cong H^s(\bbS^d)$ by \cref{lemma:sobolev_spherical}.
\cref{thm:kernel_ev_rates} is the major theoretical effort of this paper. It leverages \cref{thm:bietti_bach_adapted}, a variant of \citet[Theorem 7]{bietti_deep_2021}, which reduces the question of the eigenvalue decay to an investigation of the asymptotic behavior at $\{\pm 1\}$ of the function $\kappa:[-1,1]\to \R$ satisfying $$k(x, x') = \kappa(\langle x, x'\rangle)~.$$ This function $\kappa$ is best described with 
\emph{dual activation functions} (\cref{definition:dual_and_rescaled_activation_function}, \cref{sec:dual_activations}).
The dual activation function $\widehat\act:[-1,1]\to \R$ has been introduced in \citet{daniely_toward_2016}
and is given by a quadratic form: $\widehat\act = b(\act, \act)$, where
\begin{equation*}
	b(f,g)(t) = \E_{(u,v)\sim \calN(0,\Sigma_t)}[f(u)g(v)], \quad \Sigma_t = \begin{pmatrix}1 & t \\ t & 1\end{pmatrix}
\end{equation*}
describes the infinite-width behavior of the activation in a single hidden layer.
Recursively, the function $\knngp_l$ corresponding to the NNGP kernel of an $l$-layer network for $l\ge 2$ is then given by 
\begin{align}
		\knngp_1(t) & \equalDef  \sigma_b^2 \sigma_i^2 + \sigma_w^2 t \nonumber\\
		\knngp_{l}(t) & \equalDef  \sigma_b^2 \sigma_i^2 + \sigma_w^2 \widehat{\scaledfn{\act}{\sqrt{\alpha_{l-1}}}}(\knngp_{l-1}(t)/\alpha_{l-1}),
		\label{eq:deep_recursion_main_part}
\end{align}
where $\scaledfn{\act}{\sqrt{\alpha_{l-1}}}$ means that $\act$ rescales its input, which is explained in \Cref{definition:dual_and_rescaled_activation_function} and \Cref{lemma:restriction_to_the_unit_sphere}, but is not important as it does not change the smoothness at zero.
This recursion alongside with the corresponding formula for the NTK kernel $\kntk_l$ can be found in \cref{lemma:restriction_to_the_unit_sphere}.
The analysis of the asymptotic behavior of $\knngp_l$ is built on the analysis of the asymptotic behavior of the dual activation $\widehat\act$ at $\{\pm 1\}$. The following theorem is hence a central cornerstone for the analysis of dual activations that may be useful beyond the specific structure of fully-connected neural networks.

\begin{theorem}[\cref{thm:smoothed_dual_activation_smooth_decomposition} reformulated]	\label{thm:boundary_behavior_of_dual_main_part}
	Let $\act$ be an activation function of finite smoothness $s$. Then, there exists a constant $b_s>0$ depending only on $s$ such that for $\tau \in \{\pm 1\}$ and $t \in (0, 2)$
	\begin{align*} 
		\widehat\act(\tau (1-t)) =& \big( \act^{(s)}(0+) - \act^{(s)}(0-) \big) (-\tau)^{s+1} b_s t^{s+1/2} \\
		&+ p_\tau(t) + q_\tau(t)~,
	\end{align*}
	where $\act^{(s)}(0+), \act^{(s)}(0-)$ are right- and left-sided limits, $p_\tau(t)$ is a polynomial and $q_\tau(t)$ fulfills $q_\tau^{(n)}(t) = o(t^{(s-n)})$ for all $n\in \bbN_0$.
\end{theorem}
We prove \Cref{thm:boundary_behavior_of_dual_main_part} by decomposing
\begin{IEEEeqnarray*}{+rCl+x*}
\act(x) = \sum_{k=s}^{K-1} (\act^{(s)}(0+) - \act^{(s)}(0-)) s_k(x) + r(x)~,
\end{IEEEeqnarray*}
where $s_k(x) \equalDef \frac{1}{2k!} \sgn(x)x^k$ are analytically tractable functions of smoothness $k$ and $r$ is a remainder term of smoothness $K$.
The analysis of \citet{bietti_deep_2021} is similar to computing $\widehat s_1 = b(s_1, s_1)$, and we introduce multiple new arguments to study higher-order terms $b(s_k, s_k)$, and also mix-terms $b(s_k, s_m)$, $b(s_k, r)$, and $b(r, r)$. While \citet{bietti_deep_2021} use arguments from \citet{chen_deep_2020} based on complex analysis to control the regularity of derivatives for $q_\tau$, we circumvent this part through direct computations since it is not obvious whether functions like $b(r, r)$ are analytic.

While \cref{thm:boundary_behavior_of_dual_main_part} is only for dual activations and not the final kernels, it already provides the correct functional form ``non-integer power plus remainder terms'' needed for applying the main theorem of \citet{bietti_deep_2021}, which says that the eigenvalue decay depends on the smallest non-integer power.
In \cref{sec:appendix:boundary}, we develop a calculus that investigates how this functional form is preserved for sums, producs, and compositions of functions. 
To obtain our main technical result, \cref{thm:kernel_ev_rates}, we apply this calculus in all cases together with \Cref{prop:even_odd_kernels} for handling the bias-free special cases and a separate computation for polynomial activations in \Cref{sec:degree}.

\section{SMOOTHNESS OF NNGP SAMPLE PATHS} \label{sec:path_smoothness}
The following result connects RKHSs of kernels on the sphere to the smoothness of the paths of their Gaussian process.
\begin{restatable}{theorem}{thmPathSmoothness}\label{theorem:path_smoothness}
Let $k$ be a dot-product kernel on $\bbS^d$ whose RKHS is equivalent to a Sobolev space $H^{d+\alpha}(\bbS^d)$, $\alpha> 0$. Let $X$ be a Gaussian process on $\bbS^d$ with zero mean and covariance kernel $k$. 
\begin{enumerate}[i)]
\item  For any $\eps \ge 0$ we have $P(Y \in H^{d/2+\alpha+\eps}(\bbS^d)) = 0$ for any version $Y$ of $X$.
\item For any $0< \eps<\alpha$, there exists a version $Y$ of $X$ with $\Prob\left( Y \in H^{d/2+\alpha -\eps}\right) =1$. 
\end{enumerate}
\end{restatable}
The requirement $\eps < \alpha$ in \Cref{theorem:path_smoothness} ii) stems from the underlying theory, which requires the investigated spaces to be RKHSs. For Sobolev spaces $H^s(\bbS^d)$ is equivalent to $s>d/2$.
\Cref{theorem:path_smoothness} is proven in \Cref{sec:appendix:path_smoothness}.
\begin{example}[Application to infinite-width NNs]
Consider a network with biases, that is $\sigma_b^2 \sigma_i^2 > 0$, and an activation of smoothness $1 \leq s < \infty$. From \Cref{thm:main_result} we know $\calH_{\ntk_L} \cong H^{d/2 + s - 1/2}(\bbS^d)$ and $\calH_{\nngp_L} \cong H^{d/2 + s + 1/2}(\bbS^d)$. If $s + 1/2 > d/2$, by \Cref{theorem:path_smoothness}, the NNGP sample paths are in $H^{s + 1/2 - \varepsilon}(\bbS^d)$ but not $H^{s+1/2+\varepsilon}(\bbS^d)$ for any $\varepsilon > 0$. We conjecture that randomly initialized finite-width networks are only in $H^s(\bbS^d)$ but not $H^{s+\varepsilon}(\bbS^d)$, so the infinite-width limit would gain an extra half-order of smoothness.
\end{example}

\begin{example}[ReLU sample paths]
For the ReLU activation, in the case with biases ($\sigma_b^2\sigma_i^2 > 0$), we know from \Cref{thm:main_result} as well as \cite{bietti_deep_2021} that $\calH_{\nngp_L} \cong H^{d/2 + 3/2}(\bbS^d)$. We have $H^{d/2 + 3/2}(\bbS^d) = H^{d + \alpha}(\bbS^d)$ with $\alpha = 3/2 - d/2$. Therefore, by \Cref{theorem:path_smoothness}, the sample paths of the ReLU NNGP essentially have the smoothness $d/2 + \alpha = 3/2$, assuming $\alpha > 0$ or, equivalently, $d \in \{1, 2\}$. 
\end{example}

This phenomenon is reminiscent of how, in a suitable small-step limit, a random walk of increasingly small step size converges to the Brownian motion, which essentially has paths of smoothness  $1/2$:
\begin{remark}[Intuition: analogy to random walks]\label{rem:intuition}
	Let $g_n: [0,1] \to \R$ be the random walk of $n$ steps, defined as $g_n(t) = \frac{1}{\sqrt n} \sum_{i=1}^{\lfloor nt \rfloor} Z_i $ with i.i.d.\ coin tosses $Z_i$.
	Then $g_n$ has $n$ non-smoothnesses of degree $s=0$ which are increasingly dense in $[0,1]$, and converges weakly to a standard Brownian motion on $[0,1]$. 
	The paths of a Brownian motion are essentially of (Hölder) smoothness $1/2$; they gain half an order of smoothness. Intuitively, one might think of this that as the non-smoothnesses become dense, they also become smaller and in this sense less severe, however covering the whole interval. 
			
	This is similar to what can be observed in a two-layer neural network: 
	Consider a shallow network with bias $\sigma_b^2\sigma_i^2>0$ and the Heaviside activation function $\varphi(x) = \mathds{1}_{\mathbb{R}_{\ge 0}}(x)$ on the one-dimensional sphere $\mathbb{S}^1$.
	The network of width $n$ has at most $n$ non-smoothnesses, and as $n$ increases, these non-smoothnesses will be increasingly dense in $\mathbb{S}^1$. So, in analogy to the random walk, the finite networks are functions of smoothness $0$ with increasingly small distances between the non-smoothnesses, and the jumps in these points shrink as $n$ grows. 
	The NNGP-RKHS of the infinite-width network is $\calH_{\nngp_L}= H^{3/2}(\mathbb{S}^1) $ by \cref{thm:main_result}. The path smoothness theorem \ref{theorem:path_smoothness} shows that the sample paths of the NNGP, that is the infinite width networks at initialization, are essentially of smoothness $1/2$. We observe the same effect: the paths gain half an order of smoothness compared to the finite networks. 
	
	The same intuition applies for activations of higher smoothness $s\ge 1$.

	For the NTK, essentially the same intuition can be envoked. Since the NTK includes derivatives of the activation function, the path paths are one level rougher, they have smoothness $s-1/2$. Note that the Heaviside activation does not allow an NTK, since it is not weakly differentiable.
\end{remark}

\section{POSSIBLE EXTENSIONS}\label{sec:extensions}
While this work considers neural networks consisting only of feedforward-layers, we briefly discuss possible extensions to more complex architectures, 
by sketching the required steps to include layer normalization and residual layers.
In a similar style, adapting our argumentation to many other architectures might be possible.

Our core strategy is to write the kernels as sums, products, and compositions of dual activations and linear functions. 
Our results from \cref{sec:dual_activations} prove the boundary behavior of dual activations, and \cref{sec:appendix:boundary} introduces a calculus for how this boundary behavior behaves under sums, products, and compositions. Hence, as long as other kernels can be written in the same form, our tools should be applicable to them, and in this case one can proceed similarly as we did to obtain \cref{thm:main_result}.

Note that a network with complex architecture remains (in distribution, in the case of finite width) rotationally invariant if the first operation in the network is a linear layer with Gaussian initialization.

\paragraph{Residual layers}
Residual layers fit well into the framework we considered:  
Similar recursive fromulas as presented in \cref{lemma:limit_ntk_formula} can be derived for a network containing residual layers using  Appendix E of \citet{yang_tensor_programs_II_ntk_for_any_architecture}.
Then, it may be possible to imitate the analysis done in the further course of the \cref{sec:appendix:neural_kernel_proofs} 
for this recursion formula, which would work similar to the analysis of the NNGP-kernel.

\paragraph{Layer normalization}
For simplicity, let us consider RMSNorm
\citep{zhang2019root}, which is popular in Transformers, without parameters.
Let the $l$-th post-activation layer for $l\ge 2$ be normalized to a unit vector as $\tilde \bfx^{(l-1)} \equalDef \act(\bfz^{(l)}) / \norm{\act(\bfz^{(l)})}_2$. In a recursive decomposition of the kernel as in \cref{lemma:limit_ntk_formula}, this leads to a linear rescaling of the kernel at each layer.
In a standard feedforward-network as in \cref{def:network}, the norm of the post activation vector \emph{without layer normalization} converges almost surely to a constant value, namely $\sqrt{k^{\text{NNGP}}_{l}(\bfx,\bfx)}$, by \cref{lemma:limit_ntk_formula}. Our analysis shows that the scale of that value does not influence the RKHSs associated to the neural network -- as an illustrative example, the RKHSs do not change when changing the activation function from $\varphi$ to $\lambda\varphi$ for $\lambda\ne 0$. 

\section{IMPLICATIONS} \label{sec:implications}

Our results can be used to obtain a more complete understanding of neural networks in the kernel regime. Below, we list some non-exhaustive ways in which our results could complement existing theory.

\paragraph{Deep vs.\ shallow networks} By showing that the RKHSs of deep and shallow ReLU NTKs are equivalent, \cite{bietti_deep_2021} concluded that the NTK regime is not sufficient to model the benefits of depth for practical neural networks. Our results show that this holds for a much larger class of not infinitely smooth activation functions: Whenever $s \in [1, \infty)$, all depths $\geq 2$ (or $\geq 3$ in the bias-free case) yield equivalent RKHSs for the NTK as well as for the NNGP. On the other hand, our results show that the RKHSs are not equivalent for polynomial activations, while the situation for other infinitely smooth activations remains unclear, as more precise results are only known in the two-layer case \citep{murray_characterizing_2023}.

\paragraph{Training dynamics} Regarding theoretical modeling, the dynamics of gradient flow are difficult to study for ReLU activations since their derivative is discontinuous at zero. The use of smoother but not infinitely smooth activations such as CELU can enable theoretical studies of gradient flow in a setting where the RKHS is well-known, without limiting results to powers of ReLU as in previous work \citep{vakili_information_2023}. For example, the gradient flow analysis of \cite{bowman2022spectral} yields finite-time bounds for the deviation from the infinite-width training trajectory but requires activation functions smoother than ReLU. The use of semi-smooth activation functions might also be interesting to the NTK analysis of physics-informed neural networks \citep{wang2022and}, since their training involves taking higher-order derivatives of neural networks. 
The faster eigenvalue decay for smoother activations can lead to slower training dynamics \citep[e.g.,][]{raskutti2014early, cao_towards_2019}.
Compared to the pure kernel case, gradient flow on infinite-width NNs in the kernel regime also has to learn to remove the random initial function \citep{lee_wide_2019}. Thanks to \Cref{theorem:path_smoothness}, we now know the smoothness of this function, which makes it possible to derive convergence rates for the regression of this function with the NTK. To this end, we refer to the proof of Theorem G.5 in \cite{haas_mind_2023} for convergence rates with regularization in the noisy case, and to Theorem 3.3 in \cite{le2006continuous} for an approach towards convergence rates of interpolation.

\paragraph{Generalization} Our main result allows the application of generalization results to more activation functions. For example, the inconsistency results of \cite{haas_mind_2023} for overfitting with NTKs and NNGPs apply whenever the RKHS is a Sobolev space, and hence by our results, they apply to a wide range of finitely smooth activation functions.
Our results also suggest that the spectral bias of neural networks \citep{rahaman2019spectral} can be influenced by changing the smoothness of the activation function. Finally, while \cite{simon2022reverse} show that a large class of kernels can be represented as NTKs and NNGPs, our results help to elucidate the properties of activation functions that realize these kernels and consider networks of any depth.

\section{RELATED WORK} \label{sec:related_work}

\paragraph{Neural kernels} Multiple results have shown that in typical settings, infinite-width neural networks behave like Gaussian processes at initialization \citep{neal_priors_1996, daniely_toward_2016, lee_deep_2018, matthews_gaussian_2018}. The associated covariance function is known as the neural network Gaussian process (NNGP) kernel or random features kernel. \cite{jacot_neural_2018} discovered that the training dynamics of such infinite-width neural networks can be described by a different kernel, the so-called neural tangent kernel (NTK). The same observation has been made by \cite{lee_wide_2019}. Follow-up work has generalized the scope of these results \citep[e.g.,][]{arora_exact_2019, yang_scaling_limits_of_wide_neural_networks, yang_tensor_programs_II_ntk_for_any_architecture}. These works do not yet give deep insights into the nature of the NTK and NNGP kernels.

\paragraph{Smoothness of neural kernels} \cite{belkin_understand_2018} noted similarities in generalization behavior between the Laplace kernel and neural networks. A follow-up line of work managed to characterize the RKHS of different NTKs on $\bbS^d$. \cite{bietti_inductive_2019} derived the structure of the RKHS of the two-layer ReLU NTK without bias ($\sigma_b^2 = 0$). \cite{basri_convergence_2019} analyzed two-layer ReLU NTKs with and without bias. \cite{geifman_similarity_2020} showed that the two-layer ReLU NTK with bias is equivalent to the Laplace kernel, and showed that the RKHS of the deep ReLU NTK is at least as large as the RKHS of the shallow one. The equivalence of these RKHSs was then shown by \cite{chen_deep_2020} for $\sigma_b^2 > 0, \sigma_i = 0, L \geq 2$. Simultaneously, \cite{bietti_deep_2021} characterized the structure of the deep ReLU NTK and NNGP without bias ($\sigma_b^2 = 0, L \geq 3$) and also showed some results for step activations as well as infinitely differentiable activations. \cite{vakili_information_2023} extended this analysis to RePU (integer powers of the ReLU activation). Finally, \cite{haas_mind_2023} formally established the connection of these RKHSs to Sobolev spaces on the sphere. \cite{murray_characterizing_2023} prove further spectral properties of the NTK, including a characterization of the RKHS of two-layer NNs for tanh and RBF activation functions that is more precise than our result for these activations. \cite{scetbon_spectral_2021} prove some results about spectral properties of dot-product kernels, with weak results concerning the NTK. However, to the best of our knowledge, we are the first to provide a characterization of the RKHS for many non-smooth activation functions like LeakyReLU, SELU, or ELU.

The NTK has also been analyzed for other architectures such as residual networks \citep{belfer2024spectral} and convolutional residual networks \citep{barzilai_kernel_2023}. \cite{dandi_understanding_2021} investigate the contributions of individual layers towards the properties of the NTK.
\cite{simon2022reverse} show that a large class of dot-product kernels on the sphere can be realized as NTK or NNGP kernels through suitable activation functions.
While most results analyze the NTK on $\bbS^d$, \cite{lai2023generalization} analyze the eigenvalue decay on $\bbR$ and \cite{li2024eigenvalue} leverage known results on $\bbS^d$ to obtain decay rates for $\bbR^{d+1}$. %
In terms of activation functions, these results do not go beyond the known results for $\bbS^d$.

\paragraph{Spectral properties} \cite{panigrahi_effect_2019} investigate minimum eigenvalues of NTK matrices and find a dependence on the smoothness of the activation function. %
\cite{nguyen_tight_2021} and \cite{karhadkar2024bounds} provide more bounds for minimum eigenvalues. These results are related to the structure of the RKHS through kernel matrix concentration inequalities but do not reveal the full structure of the RKHS.

\section{CONCLUSION} \label{sec:conclusion}

We have shown general results for the structure of the RKHS for fully connected neural networks with different activation functions, as well as results for the sample path smoothness of randomly initialized wide neural networks. We refer the reader to \Cref{sec:appendix:overview} for a short overview of the central techniques and objects of our main proof.

\paragraph{Possible extensions} Our work offers several possibilities for extensions. First, following \cite{li2024eigenvalue}, an extension of our results from $x \in \bbS^d$ to $x \in \bbR^{d+1}$ could be possible. Second, one could study different activation functions in different layers; we strongly conjecture that the behavior of the NTK and NNGP will be determined by the least smooth activation function. Third, while we study non-smoothness in zero, it might be possible to apply our proof technique to functions that are non-smooth in other points---the main obstacle is to find reference functions $s_k$ with the given non-smoothness that are similarly amenable to analysis as the functions $s_k$ we use in the appendix. Fourth, a more precise characterization of the RKHS for infinitely smooth non-polynomial activations is still open but might require stronger tools \citep[e.g.,][]{minh_mercers_2006,Azevedo_ev_2014,murray_characterizing_2023}. Finally, our results might be extensible to analyses for residual networks \citep[e.g.,][]{belfer2024spectral, tirer2022kernel}, convolutional networks \citep[e.g.,][]{geifman2022spectral, barzilai_kernel_2023}, or more general architectures including transformers \citep{yang_tensor_programs_II_ntk_for_any_architecture}.

\section*{ACKNOWLEDGMENTS}

The authors thank Ingo Steinwart, Tizian Wenzel, Jens Wirth, and Alberto Bietti for valuable discussions.

Max Sch\"olpple was funded by the Deutsche Forschungsgemeinschaft (DFG) in the project STE 1074/5-1, within the DFG priority programm SPP 2298 ``Theoretical Foundations of Deep Learning'', and by the International Max Planck Research School for Intelligent Systems (IMPRS-IS).

\section*{CONTRIBUTION STATEMENT}
D.H.\ conceived the project and the high-level proof strategy. D.H.\ and M.S.\ contributed to proof writing and paper writing.

\renewcommand{\bibname}{References}
\bibliographystyle{plainnat}
\bibliography{2020_ntk_noise}

\section*{Checklist}

\begin{enumerate}

  \item For all models and algorithms presented, check if you include:
  \begin{enumerate}
    \item A clear description of the mathematical setting, assumptions, algorithm, and/or model. [Not Applicable]
    \item An analysis of the properties and complexity (time, space, sample size) of any algorithm. [Not Applicable]
    \item (Optional) Anonymized source code, with specification of all dependencies, including external libraries. [Not Applicable]
  \end{enumerate}

  \item For any theoretical claim, check if you include:
  \begin{enumerate}
    \item Statements of the full set of assumptions of all theoretical results. [Yes] This paper is dedicated to theoretical results.
    \item Complete proofs of all theoretical results. [Yes]
    \item Clear explanations of any assumptions. [Yes]     
  \end{enumerate}

  \item For all figures and tables that present empirical results, check if you include:
  \begin{enumerate}
    \item The code, data, and instructions needed to reproduce the main experimental results (either in the supplemental material or as a URL). [Yes]
    \item All the training details (e.g., data splits, hyperparameters, how they were chosen). [Not Applicable]
    \item A clear definition of the specific measure or statistics and error bars (e.g., with respect to the random seed after running experiments multiple times). [Not Applicable]
    \item A description of the computing infrastructure used. (e.g., type of GPUs, internal cluster, or cloud provider). [Not Applicable]
  \end{enumerate}

  \item If you are using existing assets (e.g., code, data, models) or curating/releasing new assets, check if you include:
  \begin{enumerate}
    \item Citations of the creator If your work uses existing assets. [Not Applicable]
    \item The license information of the assets, if applicable. [Not Applicable]
    \item New assets either in the supplemental material or as a URL, if applicable. [Not Applicable]
    \item Information about consent from data providers/curators. [Not Applicable]
    \item Discussion of sensible content if applicable, e.g., personally identifiable information or offensive content. [Not Applicable]
  \end{enumerate}

  \item If you used crowdsourcing or conducted research with human subjects, check if you include:
  \begin{enumerate}
    \item The full text of instructions given to participants and screenshots. [Not Applicable]
    \item Descriptions of potential participant risks, with links to Institutional Review Board (IRB) approvals if applicable. [Not Applicable]
    \item The estimated hourly wage paid to participants and the total amount spent on participant compensation. [Not Applicable]
  \end{enumerate}

\end{enumerate}

\onecolumn

\begin{appendices}

\listofappendices

\numberwithin{theorem}{section}
\numberwithin{lemma}{section}
\numberwithin{corollary}{section}
\numberwithin{proposition}{section}
\numberwithin{exenv}{section}
\numberwithin{remenv}{section}
\numberwithin{defenv}{section}

\counterwithin{figure}{section}
\counterwithin{table}{section}
\counterwithin{equation}{section}

\crefalias{section}{appendix}
\crefalias{subsection}{appendix}

\newpage

\section{OVERVIEW AND NOTATION} \label{sec:appendix:overview}

\Cref{sec:appendix:overview} -- \ref{sec:appendix:neural_kernel_proofs} are dedicated to the proof of the main theorem, \Cref{thm:main_result}. They build on each other and should be read in order. \Cref{sec:integrability} and \Cref{sec:appendix:path_smoothness} contain smaller proofs that are independent of the other appendices. \Cref{sec:appendix:activation_quadrature} derives a way to numerically compute dual activations for our experiments.

\paragraph{Analysis through boundary behavior.} As mentioned before, the NTK and NNGP kernels on the sphere are dot-product kernels, i.e., they are of the form $k(x, x') = \kappa(\langle x, x'\rangle)$. To obtain eigenvalue decays for $k$, we build on the results of \cite{bietti_deep_2021}, which requires to study the behavior of $\kappa(t)$ for $t \to 1$ and $t \to -1$. To unify these analyses, we let $\tau \in \{-1, 1\}$ and study the behavior of $\kappa_\tau: (0, 2) \to \bbR, t \mapsto \kappa(\tau(1-t))$ for $t \searrow 0$. We first introduce function classes with certain boundary behavior that will be central in simplifying computations later on:

\begin{definition}[Function classes with controlled boundary behavior] \label{def:boundary_function_classes}
Let $\alpha, \beta, \gamma \in \bbR$. We define sets of functions $\calP_{\alpha, \beta}$, $\calR_{\gamma}$, and $\calQ_{\alpha, \beta}$ as follows:
\begin{enumerate}[(a)]
\item \textbf{Sums of high-enough powers:} We define $\calP_{\alpha, \beta}$ as the set of functions $f: (0, 2) \to \bbR$ of the form
\begin{IEEEeqnarray*}{+rCl+x*}
f(t) & = & \sum_{i=1}^n a_i t^{\alpha_i} \quad \text{with} \quad n \in \bbN_0, a_i, \alpha_i \in \bbR, \begin{cases}
\alpha_i > \alpha &\text{for all $i$ with } \alpha_i \in \bbZ \\
\alpha_i > \beta &\text{for all $i$ with } \alpha_i \not\in \bbZ
\end{cases}
\end{IEEEeqnarray*}
Note the use of $>$, which will be convenient later.
\item \textbf{Negligible remainder functions:} We define $\calR_\gamma$ as the set of $C^\infty$-functions $f: (0, 2) \to \bbR$ such that for all $m \in \bbN_0$, the $m$-th derivative $f^{(m)}$ satisfies
\begin{IEEEeqnarray*}{+rCl+x*}
|f^{(m)}(t)| & \leq & O(t^{\gamma-m}) \quad \text{ for }t \searrow 0~.
\end{IEEEeqnarray*}
Intuitively, $\calR_\gamma$ contains functions that are at least as benign as $t \mapsto t^\gamma$, in terms of the behavior of all derivatives for $t \to 0$.
\item \textbf{Sums of powers plus negligible remainder:} We define $\calQ_{\alpha, \beta} \equalDef \bigcap_{\gamma \in \bbR} (\calP_{\alpha, \beta} + \calR_\gamma)$. In other words, $\calQ_{\alpha, \beta}$ contains those functions that are sums of powers plus an arbitrarily \quot{high-order} remainder. Note for $f\in \calQ_{\alpha,\beta}$ that in general, the number $n$ of terms in the sum of powers when writing $f$ as $f\in \calP_{\alpha,\beta} + \calR_\gamma$ will depend on the imposed order $\gamma$ of the remainder.
\end{enumerate}
We will mostly use $\calQ_{\alpha, \beta}$ in the paper. Instead of $f \in \calQ_{\alpha, \beta}$, we also write $f(t) = \calQ_{\alpha, \beta}(t)$ analogous to common $O$-notation. We often track coefficients of leading terms separately, e.g., writing $g(t) = at^\alpha + bt^\beta + \calQ_{\alpha, \beta}(t)$.
\end{definition}

Furthermore, we require refined adapted asymptotic notation for sequences, as we need the asympototic behavior of a sequence $(a_n)$ but also need to know wether there are elements $a_n =0$, which is crucial to compare Sobolev spaces based on eigenvalues, cf.\ \cref{lemma:sobolev_spherical}.
\begin{definition}[Asymptotic notation]
For functions $\kappa$, we use standard asymptotic notation like $\kappa(t) = O(t^{1/2})$ for $t \searrow 0$. For sequences $(a_n)_{n \in I}, (b_n)_{n \in I} \subseteq \bbR_{\ge 0}$ indexed by $I \subseteq \bbN_0$, we use an index \quot{$\forall n$} for asymptotic notation to denote that it should hold for all $n$ and not only for almost all $n$, and to indicate the variable $n$ that the constant is independent of. Specifically,
\begin{IEEEeqnarray*}{+rClCl+x*}
a_n &=& O_{\forall n}(b_n) & \equivDef & \exists C > 0: \forall n \in I: a_n \leq C b_n \\
a_n &=& \Omega_{\forall n}(b_n) & \equivDef & b_n = O_{\forall n}(a_n) \\
a_n &=& \Theta_{\forall n}(b_n) & \equivDef & a_n = O_{\forall n}(b_n)\text{ and } b_n = O_{\forall n}(a_n) \\
a_n &=& o_{\forall n}(b_n) & \equivDef & a_n = O_{\forall n}(b_n) \text{ and }\forall C > 0: \exists n_0 \in \bbN_0: \forall n \geq n_0: a_n \leq Cb_n~.
\end{IEEEeqnarray*}
Using this notation, when we write $\mu_l = \Theta_{\forall l}((l+1)^{-d-2s})$, it implies that all $\mu_l$ are positive. We write $l+1$ and not $l$ to make the right-hand side well-defined for all $l \in \bbN_0$. The constant $C$ implied in this notation can depend on $d$ and $s$.
\end{definition}

A key element of our analysis is the following result, formulated using the previously defined boundary function classes:

\begin{restatable}[Adaptation of Theorem 7 in arXiv v4 of \citealt{bietti_deep_2021}]{theorem}{thmBiettiBachAdapted} \label{thm:bietti_bach_adapted}
Let $\kappa: [-1, 1] \to \bbR$ be a function that is smooth on $(-1,1)$ such that  $k_{\kappa, d}(\bfx, \bfx') = \kappa(\langle \bfx, \bfx' \rangle)$ is a positive semi-definite kernel on all spheres $\bbS^d, d \in \bbN_{\geq 1}$. Suppose that there exists $0 < \beta \in \bbR \setminus \bbZ$ and $b_{-1}, b_1 \in \bbR$ such that for $\tau \in \{-1, 1\}$,
\begin{IEEEeqnarray*}{+rCl+x*}
\kappa(\tau(1-t)) = b_\tau t^\beta + \calQ_{-1, \beta}(t)~.
\end{IEEEeqnarray*}
Then, for a given dimension $d \in \bbN_{\geq 1}$, the eigenvalues $\mu_l = \mu_l(\kappa, d)$ as defined in \Cref{sec:preliminaries} satisfy:
\begin{enumerate}[(a)]
\item For $l \in \bbN_0$ even, if $b_{-1} \neq -b_1$, then $\mu_l = \Theta_{\forall l}((l+1)^{-d-2\beta})$.
\item For $l \in \bbN_0$ even, if $b_{-1} = -b_1$, then  $\mu_l = o_{\forall l}((l+1)^{-d-2\beta})$.
\item For $l \in \bbN_0$ odd, if $b_{-1} \neq b_1$, then  $\mu_l = \Theta_{\forall l}((l+1)^{-d-2\beta})$.
\item For $l \in \bbN_0$ odd, if $b_{-1} = b_1$,  then $\mu_l = o_{\forall l}((l+1)^{-d-2\beta})$.
\end{enumerate}

\end{restatable}

This leads to the following strategy, detailed below, for analyzing the eigenvalues of neural kernels on the sphere:
\begin{enumerate}[(1)]
\item Write neural kernels on the sphere as a composition of dual activation functions.
\item Analyze the boundary behavior of dual activation functions.
\item Analyze the boundary behavior for sums, products, and compositions of functions, which will be done in \Cref{prop:rules_for_Q}.
\item Analyze the behavior of even and odd functions to obtain stronger results for the special cases (b) and (d) in \Cref{thm:bietti_bach_adapted}.
\item Assemble everything.
\end{enumerate}

\paragraph{Dual activations and neural kernels.} To apply the theorem above, similar to \cite{bietti_deep_2021}, we derive a convenient expression for $\kappa$ using \emph{dual activation functions} introduced by \cite{daniely_toward_2016}. %

\begin{definition}[Dual and rescaled activation functions]\label{definition:dual_and_rescaled_activation_function}
For a function $\varphi \in L_2(\calN(0, 1))$, we follow \cite{daniely_toward_2016} and define the dual activation
\begin{IEEEeqnarray*}{+rCl+x*}
\widehat \varphi: [-1, 1] \to \bbR, t \mapsto \bbE_{(u, v) \sim \calN(0, \Sigma_t)}[\varphi(u)\varphi(v)], \qquad \Sigma_t \equalDef \begin{pmatrix}
1 & t \\
t & 1
\end{pmatrix}~.
\end{IEEEeqnarray*}
Moreover, for a function $\act: \bbR \to \bbR$ and $a \in \bbR$, we define the rescaled function
\begin{IEEEeqnarray*}{+rCl+x*}
\scaledfn{\act}{a}: \bbR \to \bbR, \scaledfn{\act}{a}(x) & \equalDef & \act(ax)~. & \qedhere
\end{IEEEeqnarray*}
\end{definition}

Our following result, derived in \cref{subsec:neural_kernels}, shows that the NTK and NNGP kernels restricted to the sphere can be expressed using dual rescaled activations:

\begin{restatable}[Neural kernels on the unit sphere]{lemma}{lemmaKernelsSphere} \label{lemma:restriction_to_the_unit_sphere}
	Let the activation function $\act:\R\to\R$ fulfill \cref{ass:act}.
	Consider a neural network $f_\theta:\R^{d+1}\to \R$ initialized as in \cref{def:network}.
	For $l\ge 2$ and $t \in [-1, 1]$  we recursively define
	\begin{IEEEeqnarray*}{+rCl+x*}
		\alpha_1 & \equalDef & \sigma_b^2 \sigma_i^2 + \sigma_w^2 \\
		\alpha_{l} & \equalDef & \sigma_b^2 \sigma_i^2 + \sigma_w^2 \widehat{\scaledfn{\act}{\sqrt{\alpha_{l-1}}}}(1) \\
		\knngp_1(t) & \equalDef & \sigma_b^2 \sigma_i^2 + \sigma_w^2 t \\
		\knngp_{l}(t) & \equalDef & \sigma_b^2 \sigma_i^2 + \sigma_w^2 \widehat{\scaledfn{\act}{\sqrt{\alpha_{l-1}}}}(\knngp_{l-1}(t)/\alpha_{l-1}) \\
		\kntk_1(t) & \equalDef & \sigma_b^2 (1 - \sigma_i^2) + \knngp_1(t)\\
		\kntk_{l}(t)  &\equalDef & \sigma_b^2(1-\sigma_i^2) + \knngp_{l}(t) + \sigma_w^2 \kntk_{l-1}(t) \widehat{\scaledfn{(\act')}{\sqrt{\alpha_{l-1}}}}(\knngp_{l-1}(t)/\alpha_{l-1})~.
	\end{IEEEeqnarray*}
	Then for all $l\ge 1,\bfx, \bfx' \in \bbS^{d}$ we have $\alpha_l> 0$ and
	\begin{align*}
		\nngp_l(\bfx, \bar\bfx) &= \knngp_l(\langle \bfx, \bar\bfx\rangle)\\
		\ntk_l(\bfx, \bar\bfx) &= \kntk_l(\langle \bfx, \bar\bfx\rangle)~.
	\end{align*}
\end{restatable}

\paragraph{Boundary behavior of dual activations.} 

To study the behavior of dual activations at $\pm 1$, we first need a definition:
\begin{definition}[Reference activations] \label{def:special_acts}
For $k \in \bbN_0$, we define the activation functions $\s_k: \bbR \to \bbR$ by
\begin{IEEEeqnarray*}{+rCl+x*}
\s_k(x) & \equalDef & \frac{1}{2k!} \sgn(x) x^k~, \qquad 
\sgn(x) \equalDef \begin{cases}
-1 &, x < 0 \\
0 &, x = 0 \\
1 &, x > 0~.
\end{cases}
\end{IEEEeqnarray*}
Let $\act$ be an activation function as in \Cref{ass:act}. For $k \in \bbN_0$, we define the coefficients
\begin{IEEEeqnarray*}{+rCl+x*}
\Delta_k(\act) \equalDef \act^{(k)}(0+) - \act^{(k)}(0-)~,
\end{IEEEeqnarray*}
where $\act^{(k)}(0+)$ is the right-sided limit of the $k$-th derivative of $\act$ at zero, and $\act^{(k)}(0-)$ is the left-sided limit.
\end{definition}
The motivation behind \Cref{def:special_acts} is given in \Cref{lemma:smoothing_k_times_at_zero}, which says that we can decompose
\begin{IEEEeqnarray*}{+rCl+x*}
\varphi = \sum_{k=0}^{m-1} \Delta_k(\varphi) s_k + \varphi_m~, \IEEEyesnumber \label{eq:act_decomposition}
\end{IEEEeqnarray*}
where $\varphi_m$ is $m$ times pseudo-differentiable (see \Cref{def:pseudo-derivative}). This decomposition is central to our proof. 

	\begin{remark}\label{rem:restricting_analysis_to_origin}
	The reason why our analysis is restricted to a non-smoothness at the origin
	is that the “reference activations” $s_k(x) = \frac{1}{2k!} \mathrm{sgn}(x) x^k$ introduced in \cref{def:special_acts} have their only non-smoothness in the origin. 
	For these reference activations, we can study their dual activation through analytic means (e.g., Lemmata \ref{lemma:s_0_dual} and \ref{lemma:asymptotics_at_the_boundary_of_hat_s}). Generalizing our results to non-smoothness in another point $b$ would require finding a similar family of reference functions that have their non-smoothness in $b$ while still allowing similar calculations for the dual activation. This is non-trivial, but if it can be done, the rest of our analysis should still apply in the same way.
	\end{remark}

As in \cite{daniely_toward_2016}, we use $h_n$ to denote the $n$-th normalized probabilist's Hermite polynomial, which means that $h_n$ is a polynomial of degree $n$ and $(h_n)_{n \in \bbN_0}$ forms an orthonormal basis of the Hilbert space $L_2(\calN(0, 1))$. For a function $f \in L_2(\calN(0, 1))$ and $n \in \bbN_0$, we define its $n$-th Hermite coefficient
\begin{IEEEeqnarray*}{+rCl+x*}
a_n(f) \equalDef \langle f, h_n \rangle_{L_2(\calN(0, 1))}~,
\end{IEEEeqnarray*}
such that we have the following Hermite expansion in $L_2(\calN(0, 1))$:
\begin{IEEEeqnarray*}{+rCl+x*}
f & = & \sum_{n=0}^\infty a_n(f) h_n~.
\end{IEEEeqnarray*}
\citet[]{daniely_toward_2016} show that the Hermite series of $\varphi$ is related to the Maclaurin series of $\widehat{\varphi}$ via 
\begin{align}
\widehat \varphi(t) = \sum_{n=0}^\infty a_n(\varphi)^2 t^n~. \label{eq:hermite_dual_activation_connection}
\end{align}
This also demonstrates that the \quot{dualization} of activation functions should be thought of as a (function-valued) quadratic form and not dualization. Using our decomposition in Eq.~\eqref{eq:act_decomposition}, we therefore obtain
\begin{IEEEeqnarray*}{+rCl+x*}
\widehat \varphi(t) & = & \sum_{i=0}^m \sum_{j=0}^m \sum_{n=0}^\infty \Delta_i(\varphi) \Delta_j(\varphi) a_n(s_i) a_n(s_j) t^n \\
&& ~+~ \sum_{i=0}^m \sum_{n=0}^\infty 2 \Delta_i(\varphi) a_n(s_i) a_n(\varphi_m) t^n + \sum_{n=0}^\infty a_n(\varphi_m)^2 t^n~. \IEEEyesnumber \label{eq:dual_act_decomposition}
\end{IEEEeqnarray*}
We will treat the first term involving $a_n(s_i) a_n(s_j)$ analytically, while using the smoothness of $\varphi_m$ to obtain fast decay rates for $a_n(\varphi_m)$, which will imply that the behavior at the boundary is dominated by the first term. In total, we obtain the following result at the end of \Cref{sec:appendix:general_activations}:

\begin{restatable}[Boundary behavior of dual activations]{theorem}{thmDualBoundary} \label{thm:smoothed_dual_activation_smooth_decomposition}
Let $\act$ be an activation function as in \Cref{ass:act} and let $m \in \bbN_0$ such that $\Delta_k(\act) = 0$ for all $k < m$ (see \Cref{def:special_acts}). Then, there exists $b_m > 0$ depending only on $m$ such that for $\tau \in \{\pm 1\}$,
\begin{IEEEeqnarray*}{+rCl+x*}
\widehat{\act}(\tau(1-t)) & = & \Delta_m(\act)^2 (-\tau)^{m+1} b_m t^{m+1/2}  + \calQ_{-1, m+1/2}(t)
\end{IEEEeqnarray*}
holds. %
\end{restatable}
To deal with the derivative of the activation in the NTK, we can use the convenient formula $\widehat{\varphi'} = \widehat \varphi'$ from \cite{daniely_toward_2016}, of which a generalized formulation (for pseudo-derivatives, see \Cref{def:pseudo-derivative}) is proved in \Cref{lemma:hermite_derivative}.

\paragraph{Even and odd functions.} Let $I \subseteq \bbR$ be symmetric around zero and let $f: I \to \bbR$ be a function. We say that $f$ is even if $f(x) = f(-x)$ for all $x \in I$, and that $f$ is odd if $f(x) = -f(-x)$ for all $x \in I$. Every function $f: I \to \bbR$ can be decomposed in its even and odd part
\begin{IEEEeqnarray*}{+rCl+x*}
\feven(x) = \frac{f(x) + f(-x)}{2}, \qquad \fodd(x) = \frac{f(x) - f(-x)}{2}~.
\end{IEEEeqnarray*}
We then have $f = \feven + \fodd$, $\feven$ is even, and $\fodd$ is odd.

For $n \in \bbZ$, we define
\begin{IEEEeqnarray*}{+rCl+x*}
\even(n) \equalDef \begin{cases}
1 &, n\text{ is even} \\
0 &, n\text{ is odd,}
\end{cases} \qquad \odd(n) \equalDef \begin{cases}
0 &, n\text{ is even} \\
1 &, n\text{ is odd,}
\end{cases}
\end{IEEEeqnarray*}

The Hermite polynomials $h_n$ are even for even $n$ and odd for odd $n$. 

We show in \Cref{prop:even_odd_kernels} that two-layer bias-free neural kernels satisfy an even-odd decomposition, and deeper bias-free neural kernels can be even/odd whenever the activation function is even/odd.

\section{BOUNDARY BEHAVIOR} \label{sec:appendix:boundary}

In the following, we will prove some rules for algebraic manipulations with the function classes $\calP_{\alpha, \beta}, \calR_{\gamma}, \calQ_{\alpha, \beta}$ from \Cref{def:boundary_function_classes}. While we mostly care about $\calQ_{\alpha, \beta}$ later, we will first prove rules for $\calP_{\alpha, \beta}$ and $\calR_\gamma$ as a simpler intermediate step, before providing rules for $\calQ_{\alpha, \beta}$ in \Cref{prop:rules_for_Q}, which are summarized in \Cref{table:rules_for_P}:

\begin{lemma}[Rules for $\calP_{\alpha, \beta}$ and $\calR_\gamma$] \label{prop:properties_of_P_and_R}
\leavevmode
\begin{enumerate}[(a)]
\item For $\alpha \in \bbR$, we have $\calP_{\alpha, \alpha} \subseteq \calR_\alpha$. 
\item Let $\gamma_1, \gamma_2 \in \bbR$. Let $f_1, f_2: (0, 2) \to \bbR$ with $f_i \in \calR_{\gamma_i}$. Then, $f_1 \cdot f_2 \in \calR_{\gamma_1 + \gamma_2}(t)$.
\item If $f_i \in \calP_{\alpha_i, \beta_i}$, $\alpha_i, \beta_i \in \bbR$, $\alpha_i \leq \beta_i$, then
\begin{IEEEeqnarray*}{+rCl+x*}
f_1 \cdot f_2 \in \calP_{\alpha_1 + \alpha_2, \min\{\alpha_1 + \beta_2, \beta_1 + \alpha_2\}}~.
\end{IEEEeqnarray*}
\item Let $A, B > 0$, $g_1: (0, 2) \to (0, 2)$, $g_2: (0, 2) \to \bbR$. Suppose that $g_2 \in \calR_{\gamma_2}, \gamma_2 \in \bbR$ and that $g_1(t) = at^\alpha + \calR_{\gamma_1}(t)$, $\gamma_1 > \alpha > 0$ and $a > 0$. Then,
\begin{IEEEeqnarray*}{+rCl+x*}
g_2 \circ g_1 & \in & \calR_{\alpha \gamma_2}~.
\end{IEEEeqnarray*} 
\item Let $J \subseteq \bbR$ be an interval, let $g_1: (0, 2) \to J$, and let $g_2 \in C^\infty(J)$ with $0 \in J$. Moreover, suppose that $g_1(t) = a t^\alpha + b t^\beta + \calP_{\alpha, \beta}(t) + \calR_\gamma(t)$ with $\lim_{t \searrow 0} g_1(t) = 0$, $0 \leq \alpha \leq \beta \leq \gamma$, $\alpha \in \bbN_0$, $\beta,\gamma \in \bbR$ and $a, b \in \bbR$. Then,
\begin{IEEEeqnarray*}{+rCl+x*}
g_2(g_1(t)) & = & g_2(0) + g_2'(0)\alpha a t^\alpha + g_2'(0) \beta bt^\beta + \calP_{\alpha, \beta}(t) + \calR_\gamma(t)~.
\end{IEEEeqnarray*}
\item Let $g_1: (0, 2) \to (0, \infty)$ such that $g_1(t) = at^\alpha + \calP_{\alpha, \alpha}(t) + \calR_\gamma(t)$ for $a, \alpha > 0$ and $\gamma > \alpha$. Then, for $\delta > 0$,
\begin{IEEEeqnarray*}{+rCl+x*}
g_1(t)^\delta = a^\delta t^{\alpha \delta} + \calP_{\alpha \delta, \alpha \delta}(t) + \calR_{\gamma - \alpha + \alpha \delta}(t)~.
\end{IEEEeqnarray*}
\end{enumerate}
\end{lemma}

\begin{proof}
\leavevmode
\begin{enumerate}[(a)]
\item This is straightforward.
\item For $n \in \bbN_0$, we have
\begin{IEEEeqnarray*}{+rCl+x*}
|(f_1 \cdot f_2)^{(n)}(t)| & = & \left|\sum_{k=0}^n \binom{n}{k} f_1^{(k)}(t) f_2^{(n-k)}(t) \right| \\
& \leq & \sum_{k=0}^n O(t^{\gamma_1 - k}) O(t^{\gamma_2 - (n-k)}) = O(t^{\gamma_1 + \gamma_2 - n})~.
\end{IEEEeqnarray*}
\item Consider integers $\alpha_i' > \alpha_i$ and non-integers $\beta_i' > \beta_i$ such that $t^{\alpha_i'} = \calP_{\alpha_i, \beta_i}(t)$ and $t^{\beta_i'} = \calP_{\alpha_i, \beta_i}(t)$. We now investigate all possible products of such terms:
\begin{itemize}
\item The term $t^{\alpha_1' + \alpha_2'}$ has integer power $\alpha_1' + \alpha_2' > \alpha_1 + \alpha_2$.
\item The term $t^{\alpha_1' + \beta_2'}$ has non-integer power $\alpha_1' + \beta_2' > \alpha_1 + \beta_2$.
\item The term $t^{\beta_1' + \alpha_2'}$ has non-integer power $\beta_1' + \alpha_2' > \beta_1 + \alpha_2$.
\item The term $t^{\beta_1' + \beta_2'}$ may have integer or non-integer power. In any case, due to the assumption $\alpha_i \leq \beta_i$, we have
\begin{IEEEeqnarray*}{+rCl+x*}
\beta_1' + \beta_2' & > & \beta_1 + \beta_2 \geq \min\{\alpha_1 + \beta_2, \beta_1 + \alpha_2\} \geq \alpha_1 + \alpha_2~.
\end{IEEEeqnarray*}
\end{itemize}
\item Let $g_2 \in \calR_{\gamma_2}$. We show
\begin{IEEEeqnarray*}{+rCl+x*}
|(g_2 \circ g_1)^{(n)}(t)| = O(t^{\gamma_2 \alpha - n})
\end{IEEEeqnarray*}
by induction on $n \in \bbN_0$. Since $0 < \alpha < \gamma_1$, we have $g_1(t) = \Theta(t^\alpha)$ for $t \to 0$ and $\lim_{t \to 0} g_1(t) = 0$. Therefore,
\begin{IEEEeqnarray*}{+rCl+x*}
g_2(g_1(t)) = O(g_1(t)^{\gamma_2}) = O(t^{\gamma_2 \alpha})~.
\end{IEEEeqnarray*}
Here, we used $g_1(t) = O(t^\alpha)$ in the case $\gamma_2 > 0$ and $g_1(t) = \Omega(t^\alpha)$ in the case $\gamma_2 < 0$.

For the induction step $n \to n+1$, we use $g_2' \in \calR_{\gamma_2 - 1}$ to obtain
\begin{IEEEeqnarray*}{+rCl+x*}
|(g_2 \circ g_1)^{(n+1)}(t)| & = & \left| \frac{\diff^n}{\diff t^n} (g_2' \circ g_1)(t) \cdot g_1'(t)\right| \\
& = & \left| \sum_{k=0}^n \binom{n}{k} (g_2' \circ g_1)^{(k)}(t) g_1^{(1+(n-k))}(t) \right| \\
& \leq & \sum_{k=0}^n O(t^{(\gamma_2 - 1)\alpha - k}) O(t^{(\alpha - 1 - (n-k))}) \\
& = & O(t^{\gamma_2 \alpha - (n+1)})~.
\end{IEEEeqnarray*}

\item We decompose $g_2 = p_2 + r_2$, where $p_2$ is the degree-$(N-1)$ Taylor polynomial of $g_2$ around $0$ for some $N \geq 2$ to be defined later. Then,
\begin{IEEEeqnarray*}{+rCl+x*}
p_2(g_1(t)) = g_2(0) + g_2'(0) g_1(t) + \sum_{k=2}^{N-1} \frac{g_2^{(k)}(0)}{k!} g_1(t)^k~.
\end{IEEEeqnarray*}
Since we required $\lim_{t \searrow 0} g_1(t) = 0$, $g_1$ contains no degree-zero polynomial terms. Therefore, it follows from (a)--(c) that the powers $g_1(t)^k$ for $k \geq 2$ satisfy $g_1(t)^k = \calP_{\alpha, \beta}(t) + \calR_\gamma(t)$, which shows that
\begin{IEEEeqnarray*}{+rCl+x*}
p_2(g_1(t)) = g_2(0) + g_2'(0)a t^\alpha + g_2'(0) bt^\beta + \calP_{\alpha, \beta}(t) + \calR_\gamma(t)~.
\end{IEEEeqnarray*}

Since $\lim_{t \searrow 0} g_1(t) = 0$, we know that $g_1 \in \calR_{\tilde\beta}$ for some $\tilde\beta > 0$. For the remainder function $r_2$, we could show that $r_2 \in \calR_N$ but cannot directly apply (d) since it requires a lower bound of the form $g_1(t) = \Omega(t^{\tilde\beta})$, which does not necessarily hold if $a=b=0$. However, here, the lower bound is not needed: Since we know that $r_2 \in C^\infty$, all derivatives of $r_2$ are bounded around zero, which is not the case for all $\calR_N$ functions. Hence, here is a variant of the proof of (d) for the current setting:

In order to prove that $r_2 \circ g_1 \in \calR_{N\tilde\beta}$, we want to prove the following statement by induction on $n$: 
\begin{itemize}
\item[] For all $n \in \bbN_0$, if $h_2 \in C^\infty(J)$ and $M \in \bbN_0$ such that $h_2(0) = h_2'(0) = \hdots = h_2^{(M-1)}(0) = 0$, then $|(h_2 \circ g_1)^{(n)}(t)| = O(t^{M\tilde\beta - n})$.
\end{itemize}
By choosing $M \equalDef N$ large enough such that $M\tilde\beta \geq \gamma$ and setting $h_2 \equalDef r_2$, we can then conclude that $r_2 \circ g_1 \in \calR_\gamma$.

\textbf{Base case:} Let $n=0$. By applying Taylor's theorem, we find that
\begin{IEEEeqnarray*}{+rCl+x*}
|h_2(g_1(t))| & = & \left|\sum_{k=0}^M \frac{h_2^{(k)}(0)}{k!}g_1(t)^k + o(|g_1(t)|^M)\right| \leq \left|\frac{h_2^{(N)}(0)}{N!} g_1(t)^N\right| + o(|g_1(t)|^N) \\
& = & O(|g_1(t)|^N) = O(t^{N\tilde\beta}).
\end{IEEEeqnarray*}

\textbf{Induction step $n \to n+1$:} We apply the induction hypothesis to $\tilde h_2 \equalDef h_2' \in C^\infty(J)$, which satisfies the derivative condition for $\tilde M \equalDef \max\{0, M-1\}$, and obtain
\begin{IEEEeqnarray*}{+rCl+x*}
\left|(h_2' \circ g_1)^{(k)}(t)\right| = O(t^{\tilde N \tilde\beta - k}) \leq O(t^{(N-1)\tilde\beta - k})
\end{IEEEeqnarray*}
for all $0 \leq k \leq n$. Hence,
\begin{IEEEeqnarray*}{+rCl+x*}
|(h_2 \circ g_1)^{(n+1)}(t)| & = & \left|\frac{\diff^n}{\diff t^n} (h_2' \circ g_1)(t) \cdot g_1'(t)\right| \\
& = & \left|\sum_{k=0}^n \binom{n}{k} (h_2' \circ g_1)^{(k)}(t) \cdot g_1^{(1+n-k)}(t)\right| \\
& \leq & \sum_{k=0}^n O(t^{(M-1)\tilde\beta - k}) O(t^{\tilde\beta - 1-n+k}) \\
& = & O(t^{M\tilde\beta - (n+1)})~,
\end{IEEEeqnarray*}
which completes the induction step.

\item We obtain
\begin{IEEEeqnarray*}{+rCl+x*}
g_1(t)^\delta & = & (at^\alpha + \calP_{\alpha, \alpha}(t) + \calR_\gamma(t))^\delta \\
& = & (at^\alpha \cdot (1 + t^{-\alpha}\calP_{\alpha, \alpha}(t) + t^{-\alpha}\calR_{\gamma}(t)))^\delta \\
& \stackrel{\text{(b), (c)}}{=} & a^\delta t^{\alpha \delta} (1 + \calP_{0, 0}(t) + \calR_{\gamma - \alpha}(t))^\delta \\
& \stackrel{\text{(e)}}{=} & a^\delta t^{\alpha \delta} (1 + \calP_{0, 0}(t) + \calR_{\gamma - \alpha}(t)) \\
& = & a^\delta t^{\alpha \delta} + \calP_{\alpha \delta, \alpha \delta}(t) + \calR_{\gamma - \alpha + \alpha \delta}(t)~. & \qedhere
\end{IEEEeqnarray*}

\end{enumerate}
\end{proof}

We now obtain the following rules for $\calQ_{\alpha, \beta}$, summarized in \Cref{table:rules_for_P}:

\begin{table}
\caption{Summary of the boundary computation rules in \Cref{prop:rules_for_Q}.} \label{table:rules_for_P}
\centering \footnotesize
\begin{tabular}{ccc}
\toprule
Rule & Assumptions \\
\midrule
$\calQ_{\alpha_2, \beta_2}(t) = \calQ_{\alpha_1, \beta_1}(t)$ & $\alpha_1 \leq \alpha_2, \beta_1 \leq \beta_2$ \\
$\calQ_{\alpha_1, \beta_1}(t) + \calQ_{\alpha_2, \beta_2}(t) = \calQ_{\min\{\alpha_1, \alpha_2\}, \min\{\beta_1, \beta_2\}}(t)$ & --- \\
$\calQ_{\alpha_1, \beta_1}(t) \cdot \calQ_{\alpha_2, \beta_2}(t) = \calQ_{\alpha_1 + \alpha_2, \min\{\alpha_1 + \beta_2, \alpha_2 + \beta_1\}}(t)$ & $\alpha_i \leq \beta_i$ \\
$g_2(a t^{\alpha} + b t^{\beta} + \calQ_{\alpha, \beta}(t)) = g_2(0) + g_2'(0) a t^\alpha + g_2'(0) b t^\beta + \calQ_{\alpha, \beta}(t)$ & $g_2 \in C^\infty$ (also at $0$), \\
& $\alpha \in \bbN_0$, $\beta \geq \alpha$, \\
& $a=0$ if $\alpha=0$, $b=0$ if $\beta=0$ \\
$(a t^{\alpha} + \calQ_{\alpha, \alpha}(t))^\delta = a^\delta t^{\alpha \delta} + \calQ_{\alpha \delta, \alpha \delta}(t)$ & $a, \alpha, \delta > 0$, $a t^{\alpha} + \calQ_{\alpha, \alpha}(t) > 0$ \\
$\calQ_{\delta, \delta}(a t^{\alpha} + \calQ_{\alpha, \alpha}(t)) = \calQ_{\alpha\delta, \alpha\delta}(t)$ & $a, \alpha, \delta > 0$, $a t^{\alpha} + \calQ_{\alpha, \alpha}(t) \in (0, 2)$ \\
\bottomrule
\end{tabular}
\end{table}

\begin{proposition}[Rules for $\calQ_{\alpha, \beta}$] \label{prop:rules_for_Q}
Let $J \subseteq \bbR$ be an interval. Let $f_1, f_2: [0, 2] \to \bbR$, $g_1: [0, 2] \to J$ and $g_2: J \to \bbR$ with
\begin{IEEEeqnarray*}{+rCl+x*}
f_1(t) & = & \calQ_{\alpha_1, \beta_1}(t), \\
f_2(t) & = & \calQ_{\alpha_2, \beta_2}(t)
\end{IEEEeqnarray*}
for some $\alpha_i, \beta_i \in \bbR$.
\begin{enumerate}[(a)]
\item \textbf{Inclusion}: If $\alpha_1 \leq \alpha_2$ and $\beta_1 \leq \beta_2$, then $\calQ_{\alpha_2, \beta_2} \subseteq \calQ_{\alpha_1, \beta_1}$.
\item \textbf{Sum:} We have
\begin{IEEEeqnarray*}{+rCl+x*}
f_1(t) + f_2(t) = \calQ_{\min\{\alpha_1, \alpha_2\}, \min\{\beta_1, \beta_2\}}(t)~.
\end{IEEEeqnarray*}
\item \textbf{Product:} If $\alpha_i \leq \beta_i$, we have
\begin{IEEEeqnarray*}{+rCl+x*}
f_1(t) \cdot f_2(t) = \calQ_{\alpha_1 + \alpha_2, \min\{\alpha_1 + \beta_2, \alpha_2 + \beta_1\}}(t)~.
\end{IEEEeqnarray*}
\item \textbf{Composition with $C^\infty$:} Suppose that $g_1(t) = a t^{\alpha} + b t^{\beta} + \calQ_{\alpha, \beta}(t)$ with $g_1(0) = 0$ for some $\alpha \in \bbN_0$ and $\beta \geq \alpha$. Additionally, suppose that $g_2 \in C^\infty(J)$ with $0 \in J$. Then,
\begin{IEEEeqnarray*}{+rCl+x*}
g_2(g_1(t)) = g_2(0) + g_2'(0) a t^\alpha + g_2'(0) b t^\beta + \calQ_{\alpha, \beta}(t)~.
\end{IEEEeqnarray*}
\item \textbf{Composition with power:} Suppose that $J \subseteq (0, \infty)$ with $g_1(t) = a t^{\alpha} + \calQ_{\alpha, \alpha}(t)$ for $a, \alpha > 0$. Then, for $\delta > 0$,
\begin{IEEEeqnarray*}{+rCl+x*}
g_1(t)^\delta & = & a^\delta t^{\alpha \delta} + \calQ_{\alpha \delta, \alpha \delta}(t)~.
\end{IEEEeqnarray*}
\item \textbf{Composition with $\calQ_{\delta, \delta}$:} Suppose that $J = (0, 2)$, $g_1(t) = a t^{\alpha} + \calQ_{\alpha, \alpha}(t)$ with $a, \alpha > 0$, and $g_2(t) = \calQ_{\delta, \delta}(t)$ with $\delta \geq 0$. Then, 
\begin{IEEEeqnarray*}{+rCl+x*}
g_2(g_1(t)) & = & \calQ_{\alpha \delta, \alpha \delta}(t)~.
\end{IEEEeqnarray*}
\end{enumerate}
\end{proposition}

\begin{proof}
\leavevmode
\begin{enumerate}[(a)]
\item Follows from the definition.
\item For given $\gamma \in \bbR$, write $f_i = p_i + r_i$ with $p_i \in \calP_{\alpha_i, \beta_i}$ and $r_i \in \calR_\gamma$. Then, obviously $p_1 + p_2 \in \calP_{\min\{\alpha_1, \alpha_2\}, \min\{\beta_1, \beta_2\}}$ and $r_1 + r_2 \in \calR_\gamma$, which shows the claim.
\item Using a decomposition as in the proof of (b), we write $f_1f_2 = p_1p_2 + (p_1r_2 + r_1p_2 + r_1r_2)$. Here, it follows from \Cref{prop:properties_of_P_and_R} (c) that $p_1p_2 \in \calP_{\alpha_1 + \alpha_2, \min\{\alpha_1 + \beta_2, \alpha_2 + \beta_1\}}$. Moreover, using \Cref{prop:properties_of_P_and_R} (a) we have $p_i \in \calR_{\alpha_i}$. Using \Cref{prop:properties_of_P_and_R} (b), we obtain
\begin{IEEEeqnarray*}{+rCl+x*}
p_1r_2 + r_1p_2 + r_1r_2 \in \calR_{\gamma + \min\{\alpha_1, \alpha_2\}}~,
\end{IEEEeqnarray*}
where $\gamma \in \bbR$ was arbitrary. This shows the claim.
\item Follows from \Cref{prop:properties_of_P_and_R} (e).
\item Follows from \Cref{prop:properties_of_P_and_R} (f).
\item Consider $\gamma \in \bbR$ to be chosen later. We can then write $g_2(t) = p_2(t) + r_2(t)$ with $p_2 \in \calP_{\delta, \delta}$ and $r_2 \in \calR_\gamma$. Since $p_2$ is a sum of powers, (a) and (e) show that $p_2(g_1(t)) = \calQ_{\alpha \delta, \alpha \delta}(t)$. For $r_2$, we use \Cref{prop:properties_of_P_and_R} (a), (d) to conclude that for some $\eps > 0$,
\begin{IEEEeqnarray*}{+rCl+x*}
r_2(g_1(t)) & = & r_2(at^\alpha + \calR_{\alpha+\eps}(t)) = \calR_{\alpha \gamma}(t)~.
\end{IEEEeqnarray*}
Since $\gamma$ was arbitrary, it can be chosen such that $\alpha \gamma$ can be arbitrarily large (since $\alpha > 0$). This completes the proof. \qedhere
\end{enumerate}
\end{proof}

\section{DUAL ACTIVATIONS}\label{sec:dual_activations}
\subsection{General properties}

\begin{definition}[Pseudo-derivative] \label{def:pseudo-derivative}
Let $f, g: \bbR \to \bbR$. We call $g$ a \emph{pseudo-derivative} of $f$ if $g$ is Lebesgue integrable on compact intervals and
\begin{IEEEeqnarray*}{+rCl+x*}
f(x) & = & f(0) + \int_0^x g(t) \diff t
\end{IEEEeqnarray*}
for all $x \in \bbR$. 
\end{definition}
If $f:\R\to\R$ is continuously differentiable, then $f'$ is a pseudo-derivative of $f$. As a non-differentiable example, the Heaviside theta function $\bbone_{(0, \infty)}(x)$ is a pseudo-derivative of the ReLU function. From basic Lebesgue integration theory, it follows that pseudo-derivatives are unique up to null sets.
Note that any \emph{continuous} activation function $\act:\R\to\R$ fulfilling \cref{ass:act} is pseudo-differentiable.

The following lemma generalizes the differentiation part of Lemma 11 by \cite{daniely_toward_2016} to pseudo-differentiable functions.

\begin{lemma}[Properties of pseudo-derivatives] \label{lemma:hermite_derivative}
Let $g \in \calL_2(\calN(0, 1))$ be a pseudo-derivative of $f: \bbR \to \bbR$. Then,
\begin{enumerate}[(a)]
\item $f \in \calL_2(\calN(0, 1))$,
\item for $n \geq 1$, $a_n(f) = n^{-1/2} a_{n-1}(g)$,
\item the Hermite expansion $f = \sum_{n=0}^\infty a_n(f) h_n$ converges pointwise,
\item $\hat{f}$ is differentiable on $[-1, 1]$ with $\hat{f}' = \hat{g}$ (cf.\ \Cref{definition:dual_and_rescaled_activation_function}).
\end{enumerate}

\begin{proof}
Let $\phi$ be the p.d.f.\ of the normal distribution $\calN(0, 1)$. Without loss of generality, let $x \geq 0$. Then,
\begin{IEEEeqnarray*}{+rCl+x*}
f(x) & = & f(0) + \int_0^x g(t) \diff t = f(0) + \left\langle \frac{\bbone_{[0, x]}}{\phi}, g\right\rangle_{L_2(\calN(0, 1))} \\
& = & f(0) + \sum_{n=0}^\infty a_n(g) \left\langle \frac{\bbone_{[0, x]}}{\phi}, h_n\right\rangle_{L_2(\calN(0, 1))} \\
& = & f(0) + \sum_{n=0}^\infty a_n(g) \int_0^x h_n(t) \diff t \\
& = & f(0) + \sum_{n=0}^\infty a_n(g) \int_0^x (n+1)^{-1/2} h_{n+1}'(t) \diff t \\
& = & f(0) + \sum_{n=0}^\infty (n+1)^{-1/2} a_n(g) [h_{n+1}(x) - h_{n+1}(0)] \\
& = & f(0) + \sum_{n=1}^\infty n^{-1/2} a_{n-1}(g) [h_n(x) - h_n(0)] \\
& = & \left[f(0) - \sum_{n=1}^\infty n^{-1/2} a_{n-1}(g) h_n(0)\right]h_0(x) + \sum_{n=1}^\infty n^{-1/2} a_{n-1}(g) h_n(x)~. \IEEEyesnumber \label{eq:integral_hermite}
\end{IEEEeqnarray*}
In the last step, we used $h_0(x) = 1$ for all $x$. Moreover, splitting the series in the last step is valid since the extracted series is absolutely summable:
\begin{IEEEeqnarray*}{+rCl+x*}
\sum_{n=1}^\infty |n^{-1/2} a_{n-1}(g) h_n(0)| & = & \sum_{n=1}^\infty n^{-1/2} |a_{n-1}(g)| \even(n) \frac{(n-1)!!}{\sqrt{n!}} \\
& \stackrel{\text{\Cref{lemma:double_factorial_asymptotics}}}{\leq} & \sum_{n=1}^\infty |a_{n-1}(g)| O(n^{-3/4}) \\
& = & \left\langle (|a_{n-1}(g)|)_{n \geq 1}, O(n^{-3/4}) \right\rangle_{\ell_2(\bbN_{\geq 1})} \\
& \stackrel{\text{Cauchy-Schwarz}}{<} & \infty~.
\end{IEEEeqnarray*}
Here, we have used that $(a_n(g))_{n \geq 1} \in \ell_2(\bbN_{\geq 1})$ since $g \in \calL_2(\calN(0, 1))$.
An analogous argument shows that \eqref{eq:integral_hermite} also equals $f(x)$ for $x < 0$. Since $|n^{-1/2} a_{n-1}(g)| \leq |a_{n-1}(g)|$, we know that the series representation \eqref{eq:integral_hermite} not only converges pointwise but also in $L_2(\calN(0, 1))$, and these limits must be identical. This shows (a), (b) and (c). In order to show (d), we compute
\begin{IEEEeqnarray*}{+rCl+x*}
\hat{f}(t) & = & \sum_{n=0}^\infty a_n(f)^2 t^n \stackrel{\text{(b)}}{=} a_0(f)^2 + \sum_{n=1}^\infty n^{-1} a_{n-1}(g)^2 t^n
\end{IEEEeqnarray*}
for $t \in [-1, 1]$, which implies
\begin{IEEEeqnarray*}{+rCl+x*}
\hat{f}'(t) & = & \sum_{n=1}^\infty n^{-1} a_{n-1}(g)^2 nt^{n-1} = \sum_{n=0}^\infty a_n(g)^2 t^n = \hat{g}(t)~.
\end{IEEEeqnarray*}
This also holds at the boundary $t \in \{-1, 1\}$: Thanks to Abel's theorem, $\hat f$ and $\hat g$ are continuous on $[-1, 1]$. Without loss of generality, let $t = 1$. The mean value theorem of integration yields $\hat f(1) - \hat f(1-h) = \int_{1-h}^1 \hat f'(t) \diff t = h f'(\xi_h)$ for some $\xi_h \in (1-h, 1)$, thus
\begin{IEEEeqnarray*}{+rCl+x*}
\hat f'(1) & = & \lim_{h \searrow 0} \hat f'(\xi_h) = \lim_{u \nearrow 1} \hat f'(u) = \lim_{u \nearrow 1} \hat g(u) = \hat g(1)~. & \qedhere
\end{IEEEeqnarray*}
\end{proof}
\end{lemma}

While the dualization $\act \mapsto \hat{\act}$ from \Cref{definition:dual_and_rescaled_activation_function} is a quadratic and not a linear mapping, it still interchanges nicely with the even-odd decomposition (since even and odd functions are orthogonal w.r.t.\ the corresponding \quot{inner product}):

\begin{lemma}[General properties of dual activations] \label{lemma:dual_properties}  %
	Let $\act \in \calL_2(\calN(0, 1))$. Then, 
	\begin{enumerate}[(a)]
	\item $\acteven, \actodd \in \calL_2(\calN(0, 1))$ and
	\begin{IEEEeqnarray*}{+rCl+x*}
		(\hat \act)_{\mathrm{even}} = \widehat \acteven, \qquad (\hat \act)_{\mathrm{odd}} = \widehat \actodd~.
	\end{IEEEeqnarray*}
	\item $\w\act$, $\w \acteven$ and $\w\actodd$ are nonnegative and increasing on $[0, 1]$. Moreover, if $\act$ is not almost surely constant, then $\w \act$ is strictly increasing on $[0, 1]$.
	\item We have
	\begin{align}
		\label{eq:dual_act_has_max_at_one} \abs{\w\act(t)}\le \w\act(|t|) \leq \w\act(1),
	\end{align}
	for all $t \in [-1, 1]$. If $\act$ is not almost surely constant, then $\abs{\w\act(t)} < \w\act(1)$ for $t \in (-1, 1)$. Moreover, $|\w\act(-1)| = \w\act(1)$ if and only if $\act$ is almost surely even or almost surely odd.
	\item 	We have $\w \act |_{(-1,1)} \in C^\infty((-1,1))$.
	\end{enumerate}
\end{lemma}

\begin{proof}
\leavevmode
\begin{enumerate}[(a)]
	\item It is easy to see that $\acteven, \actodd \in \calL_2(\calN(0, 1))$. As mentioned in \Cref{sec:appendix:overview}, the Hermite polynomials are even for even $n$ and odd for odd $n$. Hence,
	\begin{IEEEeqnarray*}{+rCl+x*}
		\acteven & = & \sum_{n=0}^\infty \even(n) a_n(\act) h_n \\
		\actodd & = & \sum_{n=0}^\infty \odd(n) a_n(\act) h_n~.
	\end{IEEEeqnarray*} 
	We also know that the dual activation satisfies $\hat \act(t) = \sum_{n=0}^\infty a_n(\act)^2 t^n$. 
	Since the monomials $t^n$ are also even for even $n$ and odd for odd $n$, we obtain
	\begin{IEEEeqnarray*}{+rCl+x*}
		(\hat \act)_{\mathrm{even}}(t)& = & \sum_{n=0}^\infty \even(n) a_n(\act)^2 t^n = \widehat \acteven(t) \\
		(\hat \act)_{\mathrm{odd}}(t) & = & \sum_{n=0}^\infty \odd(n) a_n(\act)^2 t^n = \widehat \actodd(t)~. \IEEEyesnumber \label{eq:dual_act_evenodd}
	\end{IEEEeqnarray*}
	Alternatively, this statement can also be proven directly using the definition of the dual activation.
	
	\item It follows from the Hermite expansions in \Cref{eq:dual_act_evenodd} that $\w\acteven$ and $\w\actodd$ are nonnegative and increasing on $[0, 1]$, and therefore this also holds for $\w\act$. If $\act$ is not almost surely constant, we have $a_n(\act) \neq 0$ for some $n \geq 1$. Therefore, the Hermite expansion $\w\act(t) = \sum_{n=0}^\infty a_n(\act)^2 t^n$ implies that $\w\act$ is strictly increasing on $[0, 1]$.
	\item For $t \in [0, 1]$, the statement follows from (b). For $t \in [-1, 0)$, (b) implies
	\begin{IEEEeqnarray*}{+rCl+x*}
	|\w\act(t)| & = & |\w\acteven(-t) - \w\actodd(-t)| \\
	& \leq & |\w\acteven(-t)| + |\w\actodd(-t)| = \w\acteven(-t) + \w\actodd(-t) = \w\act(-t) = \w\act(|t|) \leq \w\act(1)~.
\end{IEEEeqnarray*}
If $\act$ is not almost surely constant, then $\w\act$ is strictly increasing on $[0, 1]$ by (b) and the previous inequality implies $\abs{\w\act(t)} < \w\act(1)$ for $t \in (-1, 1)$. Moreover, the inequality $|\w\acteven(-t) - \w\actodd(-t)| \leq |\w\acteven(-t)| + |\w\actodd(-t)|$ is sharp iff $\w\acteven(-1) = 0$ or $\w\actodd(-1) = 0$, which is the case iff $\act$ is almost surely even or almost surely odd. 
\item By \cref{eq:hermite_dual_activation_connection} the dual activation $\w\act$ is a convergent power series on $[-1,1]$.
\qedhere
\end{enumerate}
\end{proof}	

\subsection{Dominating terms}

\begin{definition}[Double factorial] \label{def:double_factorial}
Following \cite{daniely_toward_2016}, we define the double factorial for $n \in \bbZ$ as
\begin{IEEEeqnarray*}{+rCl+x*}
n!! & \equalDef & \begin{cases}
1 &, n \leq 0 \\
n \cdot (n-2) \cdots 4 \cdot 2 &, n > 0 \text{ even} \\
n \cdot (n-2) \cdots 3 \cdot 1 &, n > 0 \text{ odd}.
\end{cases} & \qedhere
\end{IEEEeqnarray*}
\end{definition}

\begin{proposition} \label{prop:special_hermite_coefficients}
For all $k \in \bbN_0$, the reference activation $\s_k$ from \Cref{def:special_acts} satisfies $\s_k \in L_2(\calN(0, 1))$ and $\s_k$ is a pseudo-derivative of $\s_{k+1}$. Moreover, for all $n \in \bbN_0$, the Hermite coefficients satisfy
\begin{IEEEeqnarray*}{+rCl+x*}
a_n(\s_k) & = & \odd(n-k) \frac{(-1)^{\frac{\max(1, n-k)-1}{2}}(n-k-2)!!}{(k-n)!!\sqrt{2\pi n!}}
\end{IEEEeqnarray*}

\begin{proof}
The first two statements are straightforward to show, hence we only show the formula for the Hermite coefficients. We first note that for $X \sim \calN(0, 1)$, we have for odd $k$ the central absolute moments following \citep[see e.g.][]{winkelbauer_2012_moments}:
\begin{IEEEeqnarray*}{+rCl+x*}
\bbE[|X|^k] & = & 2^{k/2} \Gamma\left(\frac{k+1}{2}\right) \pi^{-1/2} = 2^{k/2} ((k-1)/2)! \pi^{-1/2} = 2^{1/2} (k-1)!! \pi^{-1/2} \\
& = & \sqrt{\frac{2}{\pi}} (k-1)!!~.
\end{IEEEeqnarray*}
This allows us to show the statement for $n=0$:
\begin{IEEEeqnarray*}{+rCl+x*}
a_0(\s_k) & = & \langle h_0, \s_k \rangle_{\calL_2(\calN(0, 1))} = \int_{\bbR} \s_k(x) \frac{1}{\sqrt{2\pi}} e^{-x^2/2} \diff x = \odd(k) \frac{1}{2k!} \int_{\bbR} |x^k| \frac{1}{\sqrt{2\pi}} e^{-x^2/2} \diff x \\
& = & \odd(k)\frac{(k-1)!!}{k!\sqrt{2\pi}} = \odd(k)\frac{1}{k!!\sqrt{2\pi}} = \odd(0-k) \frac{(-1)^{\frac{\max(1, 0-k)-1}{2}}(0-k-2)!!}{(k-0)!!\sqrt{2\pi 0!}}~, %
\end{IEEEeqnarray*}

Finally, we show the formula for all $n, k$ via induction on $k$. (We do not use induction on $n$.) For the base case $k=0$, we note that $h_0 \equiv 1$ and $\s_0 = 2^{-1/2} \sqrt{2} \bbone_{(0, \infty)} - \frac{1}{2}h_0$ almost everywhere. \cite{daniely_toward_2016} show in Section 8 that
\begin{IEEEeqnarray*}{+rCl+x*}
a_n(\sqrt{2} \bbone_{(0, \infty)}) & = & \begin{cases}
2^{-1/2} &, n = 0 \\
\frac{(-1)^{\frac{n-1}{2}}(n-2)!!}{\sqrt{\pi n!}} &, n\text{ odd} \\
0 &, 2 \leq n\text{ even.}
\end{cases}
\end{IEEEeqnarray*}
Therefore, we obtain for $n \geq 1$:
\begin{IEEEeqnarray*}{+rCl+x*}
a_n(\s_0) & = & 2^{-1/2} a_n(\sqrt{2}\bbone_{(0, \infty)}) = \odd(n) \frac{(-1)^{\frac{n-1}{2}}(n-2)!!}{\sqrt{2 \pi n!}} \\
& = & \odd(n-0) \frac{(-1)^{\frac{\max(1, n-0)-1}{2}}(n-0-2)!!}{(0-n)!!\sqrt{2\pi n!}}~,
\end{IEEEeqnarray*}
which completes the base case $k=0$ since the case $n=0$ has already been treated above. 

For the induction $k \to k+1$, we note that $s_k$ is a pseudo-derivative of $s_{k+1}$ and use \Cref{lemma:hermite_derivative} to obtain for all $n \geq 0$:
\begin{IEEEeqnarray*}{+rCl+x*}
a_{n+1}(\s_{k+1}) & = & (n+1)^{-1/2} a_n(\s_k) = \odd(n-k) \frac{(-1)^{\frac{\max(1, n-k)-1}{2}}(n-k-2)!!}{(k-n)!!\sqrt{2\pi (n+1)!}} \\
& = & \odd((n+1)-(k+1)) \frac{(-1)^{\frac{\max(1, (n+1)-(k+1))-1}{2}}((n+1)-(k+1)-2)!!}{((k+1)-(n+1))!!\sqrt{2\pi (n+1)!}}~,
\end{IEEEeqnarray*}
which completes the induction.
\end{proof}
\end{proposition}

\begin{lemma}\label{lemma:s_0_dual}
We have $\widehat{\s_0}(t) = \frac{1}{4} - \frac{1}{2\pi} \arccos(t)$.

\begin{proof}
By Section 8 in \cite{daniely_toward_2016}, the function $f \equalDef \sqrt{2} \bbone_{(0, \infty)}$ has the dual activation $\hat{f}(t) = 1 - \frac{1}{\pi} \arccos(t)$. Hence, $g \equalDef \bbone_{(0, \infty)}$ has the dual activation
\begin{IEEEeqnarray*}{+rCl+x*}
\hat{g}(t) = \frac{1}{2} - \frac{1}{2\pi} \arccos(t)~.
\end{IEEEeqnarray*}
Now, since $\s_0 = g - 1/2$ almost everywhere, we have $a_n(\s_0) = a_n(g)$ for all $n \neq 0$, and hence
\begin{IEEEeqnarray*}{+rCl+x*}
\widehat{\s_0}(t) = \sum_{n=0}^\infty a_n(\s_0)^2 t^n = C + \sum_{n=0}^\infty a_n(g)^2 t^n = C + \frac{1}{2} - \frac{1}{2\pi} \arccos(t)
\end{IEEEeqnarray*}
for some constant $C \in \bbR$. Since $\s_0$ is odd, $\widehat{\s_0}$ must be odd by \Cref{lemma:dual_properties}, which yields $C = -1/4$.
\end{proof}
\end{lemma}

\begin{lemma} \label{lemma:asymptotics_at_the_boundary_of_hat_s}
For all $k \in \bbN_0$, there exists $b_k > 0$ such that for $\tau \in \{\pm 1\}$,
\begin{IEEEeqnarray*}{+rCl+x*}
\w{\s_k}(\tau(1-t)) & = & (-\tau)^{k+1} b_k t^{k+1/2} + \calQ_{-1, k+1/2}(t)~.
\end{IEEEeqnarray*}

\begin{proof}
\textbf{Step 1: Analysis of $s_0$.} First, consider the case $\tau=1$.
We have
\begin{IEEEeqnarray*}{+rCl+x*}
\frac{\diff}{\diff t} \arccos(1-t) = \frac{1}{\sqrt{1 - (1-t)^2}} = t^{-1/2} (2-t)^{-1/2}~.
\end{IEEEeqnarray*}
It is easily seen using induction that
\begin{IEEEeqnarray*}{+rCl+x*}
\frac{\diff^n}{\diff t^n} (2-t)^{-1/2} = \frac{(2n-1)!!}{2^n} (2-t)^{-\frac{2n+1}{2}}~.
\end{IEEEeqnarray*}
Since $t \mapsto (2 - t)^{-1/2}$ is analytic in a neighborhood of $t=0$, it is equal to its Taylor expansion:
\begin{IEEEeqnarray*}{+rCl+x*}
(2-t)^{-1/2} & = & \sum_{n=0}^\infty \frac{(2n-1)!!}{2^{2n + \frac{1}{2}} n!} t^n~.
\end{IEEEeqnarray*}
Since all coefficients are positive, we can apply the monotone convergence theorem to obtain
\begin{IEEEeqnarray*}{+rCl+x*}
\arccos(1-t) & = & \arccos(1-0) + \int_0^t \frac{\diff}{\diff u} \arccos(1-u) \diff u \\
& = & \int_0^t \sum_{n=0}^\infty \frac{(2n-1)!!}{2^{2n + \frac{1}{2}} n!} u^{n-1/2} \diff u \\
& = & \sum_{n=0}^\infty \int_0^t \frac{(2n-1)!!}{2^{2n + \frac{1}{2}} n!} u^{n-1/2} \diff u \\
& = & \sum_{n=0}^\infty \frac{(2n-1)!!}{2^{2n + \frac{1}{2}} n!} \left[\frac{1}{n+1/2} u^{n+1/2}\right]_0^t \\
& = & \sum_{n=0}^\infty \frac{(2n-1)!!}{2^{2n + \frac{1}{2}} n! (n+1/2)} t^{n+1/2}~.
\end{IEEEeqnarray*}
For arbitrary $1 \leq N \in \bbN_0$, this yields using \cref{lemma:s_0_dual}, the identity
\begin{IEEEeqnarray*}{+rCl+x*}
\widehat{\s_0}(1-t) & = & \frac{1}{4} - \sum_{n=0}^\infty \frac{(2n-1)!!}{2^{2n + \frac{3}{2}} n! (n+1/2) \pi} t^{n+1/2} \\
& = & \frac{1}{4} - \frac{1}{\pi \sqrt{2}} t^{1/2} + \underbrace{\sum_{n=1}^{N-1} \frac{(2n-1)!!}{2^{2n + \frac{3}{2}} n! (n+1/2) \pi} t^{n+1/2}}_{=\calP_{0, 1/2}(t)} \\
&& ~+~ \underbrace{t^{N+1/2}}_{= \calR_{N+1/2}(t)} \underbrace{\sum_{n=N}^\infty \frac{(2n-1)!!}{2^{2n + \frac{3}{2}} n! (n+1/2) \pi} t^{n-N}}_{= \calR_0(t)}.
\end{IEEEeqnarray*}
\textbf{Step 2: Induction on $k$.} We now show the lemma via induction on $k \in \bbN_0$. For arbitrary $\gamma \in \bbR$, by choosing $N$ with $N+1/2 \geq \gamma$, we obtain using \cref{prop:properties_of_P_and_R}:
\begin{IEEEeqnarray*}{+rCl+x*}
\widehat{\s_0}(1-t) & = & (-1)^1 b_0 t^{0 + 1/2} + \calP_{-1, 1/2}(t) + \calR_{\gamma}(t)~.
\end{IEEEeqnarray*}
For the induction step $k \to k+1$, we note that since $\s_k$ is a pseudo-derivative of $\s_{k+1}$ (see \Cref{def:pseudo-derivative}), we have $\widehat{\s_{k+1}}' = \w{\s_{k+1}'} = \widehat{\s_k}$ by \Cref{lemma:hermite_derivative}, which yields
\begin{IEEEeqnarray*}{+rCl+x*}
\widehat{\s_{k+1}}(1-t) & = & \widehat{\s_{k+1}}(1) - \int_0^t \widehat{\s_{k+1}}'(1-u) \diff u \\
& = & \widehat{\s_{k+1}}(1) - \int_0^t \widehat{\s_k}(1-u) \diff u \\
& = & \widehat{\s_{k+1}}(1) - \int_0^t (c_k + (-1)^{k+1} b_k u^{k+1/2} + \calP_{0, k+1/2}(u) + \calR_\gamma(u)) \diff u \\
& = & c_{k+1} + (-1)^{k+2} b_{k+1} t^{(k+1)+1/2} + \calP_{0, (k+1)+1/2}(t) + \calR_{\gamma+1}(t)
\end{IEEEeqnarray*}
for suitable constants $c_{k+1}, b_{k+1} > 0$. This shows the claim for $\tau = 1$. For $\tau=-1$, we note that $\s_k$ is odd for even $k$ and even for odd $k$. By \cref{lemma:dual_properties}, the same holds for $\widehat{\s_k}$. This shows
\begin{IEEEeqnarray*}{+rCl+x*}
\widehat{\s_k}(-1+t) & = & (-1)^{k+1} \widehat{\s_k}(1-t) \\
& = & b_k t^{k+1/2} + \calP_{-1, k+1/2}(t) + \calR_\gamma(t)~. & \qedhere
\end{IEEEeqnarray*}
\end{proof}
\end{lemma}

We will also need the asymptotic decay of $a_n(s_k)$, which will be studied in the following two lemmas.

\begin{lemma} \label{lemma:double_factorial_asymptotics}
For $p \in \bbN_0$ and odd $m \in \bbN_0$, we have
\begin{IEEEeqnarray*}{+rCl+x*}
\frac{m!!}{\sqrt{(m+p)!}} = \Theta_{\forall m}(m^{1/4-p/2})~.
\end{IEEEeqnarray*}
\begin{proof}
For even $n = 2k \geq 2$, we have $n!! = n(n-2)\cdots 2 = (2\cdot k)(2 \cdot (k-1)) \cdots (2 \cdot 1) = 2^k k! = 2^{n/2} (n/2)!$. Using Stirling's formula, we obtain
\begin{IEEEeqnarray*}{+rCl+x*}
\frac{(n-1)!!}{\sqrt{n!}} & = & \frac{n!}{n!!\sqrt{n!}} = \frac{\sqrt{n!}}{2^{n/2} (n/2)!} \sim \frac{(2 \pi n)^{1/4} n^{n/2} e^{-n/2}}{2^{n/2} (\pi n)^{1/2} (n/2)^{n/2} e^{-n/2}} = \Theta_{\forall n}(n^{-1/4})~.
\end{IEEEeqnarray*}
By setting $n \equalDef m+1$, we obtain
\begin{IEEEeqnarray*}{+rCl+x*}
\frac{m!!}{\sqrt{(m+p)!}} & = & \Theta_{\forall m}(m^{1/2-p/2}) \frac{m!!}{\sqrt{(m+1)!}} = \Theta_{\forall m}(m^{1/2-p/2}) \Theta_{\forall m}(m^{-1/4}) \\
& = & \Theta_{\forall m}(m^{1/4-p/2})~. & \qedhere
\end{IEEEeqnarray*}
\end{proof}
\end{lemma}

\begin{lemma} \label{lemma:s_hermite_decay}
We have $|a_n(\s_k)| = \Theta_{\forall n}(\odd(n-k)(n+1)^{-3/4-k/2})$.

\begin{proof}
Since $(k-n)!! = 1$ for $n \geq k$ by our definition of the double factorial in \Cref{def:double_factorial}, we obtain
\begin{IEEEeqnarray*}{+rCl+x*}
|a_n(\s_k)| & \stackrel{\text{\Cref{prop:special_hermite_coefficients}}}{=} & \Theta_{\forall n}\left(\odd(n-k) \frac{(n-k-2)!!}{\sqrt{2\pi n!}}\right) \\
& \stackrel{\text{\Cref{lemma:double_factorial_asymptotics}}}{=} & \Theta_{\forall n}(\odd(n-k) (n+1)^{-3/4-k/2})~. & \qedhere
\end{IEEEeqnarray*}
\end{proof}
\end{lemma}

\subsection{Mix terms}

Now, we want to investigate some of the mix-terms arising in the decomposition in \Eqref{eq:dual_act_decomposition}. For convenience, we will give them a new name:

\begin{definition}\label{definition:hermitian-spline-based_power_series}
For $i, j \in \bbN_0$, define $f_{i,j}: [-1, 1] \to \bbR$ by
\begin{IEEEeqnarray*}{+rCl+x*}
f_{i,j}(t) \equalDef \sum_{n=0}^\infty a_n(s_i) a_n(s_j) t^n~.
\end{IEEEeqnarray*}
Note that $f_{i,j} \equiv 0$ for odd $i-j$ and $f_{i,j} = f_{j,i}$. Moreover, $f_{i,i} = \widehat{\s_i}$.
\end{definition}

\begin{lemma}[Recursive characterization of mix-terms] \label{lemma:cross_term_recursion}
Let $k, l \in \bbN_0$. Then, there exists a polynomial $p_{k, k+2l}$ such that for $t \in [-1, 1]$,
\begin{IEEEeqnarray*}{+rCl+x*}
f_{k,k+2l}(t) & = & (k+2)f_{k+2,(k+2)+2(l-1)}(t) - t f_{k+1,(k+1)+2(l-1)}(t) + p_{k,k+2l}(t) \\
f_{k,k+2l+1}(t) & = & 0~.
\end{IEEEeqnarray*}

\begin{proof}
Since $a_n(s_k) = 0$ for even $n-k$ and $a_n(s_{k+2l+1}) = 0$ for even $n-(k+2l+1)$, we have
\begin{IEEEeqnarray*}{+rCl+x*}
f_{k,k+2l+1}(t) & = & \sum_{n=0}^\infty a_n(s_k) a_n(s_{k+2l+1}) t^n = 0~.
\end{IEEEeqnarray*}
For the other formula, we use \Cref{prop:special_hermite_coefficients} and obtain
\begin{IEEEeqnarray*}{+rCl+x*}
a_n(s_k) a_n(s_{k+2l}) &=& \odd(n-k)\odd(n-(k+2l)) (-1)^{\frac{\max(1, n-k)-1}{2} + \frac{\max(1, n-k-2l)-1}{2}} \\
&& ~\cdot~\frac{(n-k-2)!!(n-k-2l-2)!!}{(k-n)!!(k+2l-n)!!2\pi n!}~.
\end{IEEEeqnarray*}
Here, $\odd(n-k)\odd(n-(k+2l)) = \odd(n-k)$. Moreover, for odd $n-k$ and $n \geq k + 2l$, we have
\begin{IEEEeqnarray*}{+rCl+x*}
(-1)^{\frac{\max(1, n-k)-1}{2} + \frac{\max(1, n-k-2l)-1}{2}} = (-1)^{n-k-l-1} = (-1)^{-l} = (-1)^l~.
\end{IEEEeqnarray*}

In the following, we write
\begin{IEEEeqnarray*}{+rCl+x*}
f \polyeq g
\end{IEEEeqnarray*}
if $f - g$ is a polynomial.

The terms $(k-n)!!$ and $(k+2l-n)!!$ are equal to one for $n \geq k+2l$. Therefore, we obtain for all but finitely many $n$
\begin{IEEEeqnarray*}{+rCl+x*}
a_n(s_k) a_n(s_{k+2l}) & = & \odd(n-k)(-1)^l \frac{(n-k-2)!!(n-k-2l-2)!!}{2\pi n!}~.
\end{IEEEeqnarray*}
By writing $(n-k-2)!! = -(k+2)(n-k-4)!! + n(n-k-4)!!$, we obtain:
\begin{IEEEeqnarray*}{+rCl+x*}
&& \sum_{n=0}^\infty a_n(s_k) a_n(s_{k+2l}) t^n \\
& \polyeq & -(k+2) \sum_{n=0}^\infty \odd(n-k)(-1)^l \frac{(n-k-4)!!(n-k-2l-2)!!}{2\pi n!}t^n \\
&& ~+~ \sum_{n=0}^\infty n \odd(n-k)(-1)^l \frac{(n-k-4)!!(n-k-2l-2)!!}{2\pi n!}t^n \\
& \polyeq & (k+2) \sum_{n=0}^\infty \odd(n-(k+2)) (-1)^{l-1} \frac{(n-(k+2)-2)!!(n-(k+2)-2(l-1)-2)!!}{2\pi n!} t^n \\
&& ~-~ t \cdot \sum_{n=1}^\infty \odd((n-1)-(k-1)) (-1)^{l-1} \frac{((n-1)-k-3)!!((n-1)-k-2l-1)!!}{2\pi (n-1)!} t^{n-1} \\
& = & (k+2) \sum_{n=0}^\infty \odd(n-(k+2)) (-1)^{l-1} \frac{(n-(k+2)-2)!!(n-(k+2)-2(l-1)-2)!!}{2\pi n!} t^n \\
&& ~-~ t \cdot \sum_{n=0}^\infty \odd(n-(k+1)) (-1)^{l-1} \frac{(n-(k+1)-2)!!(n-(k+1)-2(l-1)-2)!!}{2\pi n!} t^n \\
& \polyeq & (k+2)\sum_{n=0}^\infty a_n(s_{k+2}) a_n(s_{(k+2)+2(l-1)}) t^n - t \cdot \sum_{n=0}^\infty a_n(s_{k+1}) a_n(s_{(k+1)+2(l-1)}) t^n~,
\end{IEEEeqnarray*}
which completes the proof.
\end{proof}
\end{lemma}

\begin{proposition}[Boundary behavior of mix terms]\label{lemma:_product_coeffs_asymptotics}
Let $i, j \in \bbN_0$. Then, for any $\tau \in \{\pm 1\}$ and $t \in (0, 2)$,
\begin{IEEEeqnarray*}{+rCl+x*}
f_{i,j}(\tau(1-t)) = \calQ_{-1, (i+j)/2}(t)~.
\end{IEEEeqnarray*}
\begin{proof}

Recall \Cref{lemma:cross_term_recursion}.
Since $f_{i,j} = f_{j,i}$, we can assume $i \leq j$ without loss of generality. Moreover, for odd $j-i$, the statement is trivial since $f_{i,j} \equiv 0$. For the remaining cases where $j-i \geq 0$ is even, it suffices to prove the following statement by induction on $l$:
\begin{itemize}
\item[] For all $l \in \bbN_0$: For all $k \in \bbN_0$:
\begin{IEEEeqnarray*}{+rCl+x*}
f_{k, k+2l} & = & \calQ_{-1, k+l}(t)~.
\end{IEEEeqnarray*}
\end{itemize}
For $l=0$, this follows from \Cref{lemma:asymptotics_at_the_boundary_of_hat_s} since $f_{k, k} = \widehat{\s_k}$. For the induction step, we use \Cref{lemma:cross_term_recursion} to obtain for all $k \in \bbN_0$:
\begin{IEEEeqnarray*}{+rCl+x*}
&& f_{k,k+2l}(\tau(1-t)) \\
& = & \calQ_{-1, \infty}(t) + (k+2)f_{k+2, k+2+2(l-1)}(\tau(1-t)) - \tau(1-t) f_{k+1, k+1+2(l-1)}(\tau(1-t)) \\
& = & \calQ_{-1, \infty}(t) + \calQ_{-1, k+l+1}(t) - \tau(1-t)(\calQ_{-1, k+l}(t)) \\
& = & \calQ_{-1, k+l}(t)~.
\end{IEEEeqnarray*}
Here, we used that for a polynomial $p$, the map $t\mapsto p(\tau(1-t))$ is also a polynomial and thus in $\calQ_{-1, \infty}(t)$.
\end{proof}
\end{proposition}

\subsection{Results for smooth terms}

In the following, we analyze the decay of Hermite coefficients for smooth functions, which can then be used to establish the smoothness of certain components of dual activations. For smooth $f$, we could show the smoothness of $\w{f}$ by using $\w{f}' = \w{f'}$, but this approach does not work directly for mix-terms, and hence the intermediate step via coefficient decay is helpful.

\begin{lemma} \label{cor:smooth_hermite_decay}
If $f: \bbR \to \bbR$ has a $m$-fold pseudo-derivative $f^{(m)}$ with $f^{(m)} \in \calL_2(\calN(0, 1))$, then $|a_n(f)| < o_m((n+1)^{-m/2})$.

\begin{proof}
By \Cref{lemma:hermite_derivative}, we obtain for $n \geq m$:
\begin{IEEEeqnarray*}{+rCl+x*}
a_n(f) & = & \left[n(n-1)\cdot \hdots \cdot (n-m+1)\right]^{-1/2} a_{n-m}(f^{(m)})~.
\end{IEEEeqnarray*}
Since $a_{n-m}(f^{(m)}) < o_{\forall n}(1)$, we obtain $|a_n(f)| < o_{\forall n}((n+1)^{-m/2})$.
\end{proof}
\end{lemma}

\begin{lemma}[Differentiability of power series] \label{lemma:power_series_differentiability}
Let $f: [-1, 1] \to \bbR, x \mapsto \sum_{n=0}^\infty b_n x^n$ with $|b_n| = O_{\forall n}((n+1)^{-(k+1+\varepsilon)})$ for some $k \in \bbN_0$ and $\eps > 0$. Then, $f \in C^k([-1, 1])$.
\begin{proof}
We prove by induction on $k$ that $f \in C^k([-1, 1])$ and 
\begin{IEEEeqnarray*}{+rCl+x*}
f^{(k)}(t) = \sum_{n=k}^\infty n\cdot\ldots\cdot(n-k+1) b_n t^{n-k}
\end{IEEEeqnarray*}
for $t \in [-1, 1]$. For $k=0$, we have $\sum_{n=0}^\infty |b_n| < \infty$ and the result follows from Abel's theorem on power series. Now, let the statement hold for $k-1 \geq 0$. We know from the case $k=0$ that 
\begin{IEEEeqnarray*}{+rCl+x*}
g: [-1, 1] \to \bbR, t \mapsto \sum_{n=1}^\infty n b_n t^{n-1}
\end{IEEEeqnarray*}
as well as $f$ are continuous. By elementary analysis, $f' = g$ on $(-1, 1)$. Moreover,
\begin{IEEEeqnarray*}{+rCl+x*}
f(1) - f(1-h) = \lim_{h' \searrow 0} f(1-h') - f(1-h) = \lim_{h' \searrow 0} \int_{1-h}^{1-h'} g(x) \diff x = \int_{1-h}^1 g(x) \diff x = h g(\xi_h)
\end{IEEEeqnarray*}
for a suitable $\xi_h \in [1-h, 1]$ by the mean value theorem of integration. Since $g$ is continuous, it follows that $f$ is differentiable in $1$ with $f'(1) = g(1)$. An analogous calculation can be applied for $t=-1$. By applying the induction hypothesis to $g = f'$, the proof is completed.
\end{proof}
\end{lemma}

\subsection{General activation functions} \label{sec:appendix:general_activations}

Now, we want to obtain the asymptotic boundary behavior in the sense of \Cref{sec:appendix:boundary} for general activation functions $\act$ as in \Cref{ass:act}. To this end, we use a decomposition into reference functions and smooth remainders:

\begin{lemma}\label{lemma:smoothing_k_times_at_zero}
Let $\act$ be an activation function as in \Cref{ass:act} and let $m \in \bbN_0$. 
Then, using $\Delta_k(\act)$ as defined in \Cref{def:special_acts},
\begin{IEEEeqnarray*}{+rCl+x*}
\act_m \equalDef \act - \sum_{k=0}^{m-1} \Delta_k(\act) s_k
\end{IEEEeqnarray*}
is $m$ times pseudo-differentiable and the $m$-fold pseudo-derivative $\act_m^{(m)}$ is in $L_2(\calN(0, 1))$.

\begin{proof}
Note that $\Delta_m$ is a linear operator with $\Delta_m(s_k) = \delta_{mk}$. Hence, for $l \leq m-1$, we have
\begin{IEEEeqnarray*}{+rCl+x*}
\Delta_l(\act_m) = \Delta_l(\act) - \sum_{k=0}^{m-1} \Delta_k(\act) \Delta_l(\s_k) = \Delta_l(\act) - \Delta_l(\act) = 0~.
\end{IEEEeqnarray*}
We now prove by induction on $l$ with $0 \leq l \leq m$ that $g_l: \bbR \to \bbR$ defined by
\begin{IEEEeqnarray*}{+rCl+x*}
g_0 \equalDef \act_m, g_l(x) & \equalDef & \begin{cases}
\act_m^{(l)}(x) &, x \neq 0 \\
\act_m^{(l)}(0-) &, x = 0
\end{cases} \quad (l \geq 1)
\end{IEEEeqnarray*}
is a $l$-fold pseudo-derivative of $\act_m$. For $l=0$, we have $g_0 = \act_m$ by definition, hence $g_0$ is a $0$-fold pseudo-derivative of $\act_m$. For the induction step $l \to l+1$, we observe that $l \leq m-1$, hence $\Delta_l(\act_m) = 0$, which implies $\act_m^{(l)}(0-) = \act_m^{(l)}(0+)$ and therefore, $g_l$ is continuous. Moreover, $g_l$ is continuously differentiable on $(0, x)$ and $(-x, 0)$, and $g_l'$ is bounded on these intervals. Thus,
\begin{IEEEeqnarray*}{+rCl+x*}
g_l(x) & = & g_l(0) + \int_0^x g_l'(t) \diff t = g_l(0) + \int_0^x g_{l+1}(t) \diff t \\
g_l(-x) & = & g_l(0) + \int_0^{-x} g_l'(t) \diff t = g_l(0) + \int_0^{-x} g_{l+1}(t) \diff t~,
\end{IEEEeqnarray*}
which shows that $g_{l+1}$ is a pseudo-derivative of $g_l$.

It remains to show that $g_m \in L_2(\calN(0, 1))$. For $x \neq 0$, we have $s_k^{(m)}(x) = 0$ for $k < m$ and therefore $g_m(x) = \varphi^{(m)}(x)$. By \Cref{ass:act}, $\varphi|_{(0, \infty)} \in \calS^\infty((0, \infty))$ and $\varphi|_{(-\infty, 0)} \in \calS^\infty((-\infty, 0))$, which implies the desired integrability for $\varphi^{(m)}$, hence $g_m \in L_2(\calN(0, 1))$.
\end{proof}
\end{lemma}

\thmDualBoundary*

\begin{proof}

\textbf{Step 0: Proof strategy} Fix $\tau \in \{-1, 1\}$. For an integer $M_1 > m$ to be defined later, we decompose $\act$ using \Cref{lemma:smoothing_k_times_at_zero} as
\begin{IEEEeqnarray*}{+rCl+x*}
\act = g_1 + g_2, \qquad g_1 \equalDef \sum_{k=m}^{M_1 - 1} \Delta_k(\act) \s_k, \qquad g_2 \equalDef \act_{M_1}~.
\end{IEEEeqnarray*}
Here, \Cref{lemma:smoothing_k_times_at_zero} tells us that $g_2$ is $M_1$ times pseudo-differentiable and $g_2^{(M_1)} \in L_2(\calN(0, 1))$. The basic strategy is as follows: To verify $f(t) = \calQ_{-1, m+1/2}(t)$ for a suitable function, we need to show that for all $\gamma \in \bbR$, there exists a decomposition $f = q + r$ with $q \in \calP_{-1, m+1/2}$ and $r \in \calR_\gamma$. We will select $M_1$ depending on $\gamma$ to construct such a decomposition. To show $r \in \calR_\gamma$, we need to show $|r^{(N)}(t)| = O(t^{\gamma-N})$ for all $N \in \bbN_0$. To this end, we will decompose $\act$ again with an order $M_2 > M_1$ that depends on $N$ (but does not change $r$).

\textbf{Step 1: Defining the decomposition of $\w\act$.} Using \cref{definition:hermitian-spline-based_power_series}, we obtain
\begin{IEEEeqnarray*}{+rCl+x*}
\widehat{g_1}(t) & = & \sum_{n=0}^\infty \left(\sum_{k=m}^{M_1 - 1} \Delta_k(\act) a_n(\s_k)\right)^2 t^n \\
& = & \sum_{k_1=m}^{M_1-1} \sum_{k_2=m}^{M_1-1} \Delta_{k_1}(\act) \Delta_{k_2}(\act) \sum_{n=0}^\infty a_n(\s_{k_1}) a_n(\s_{k_2}) t^n \\
& = & \sum_{k_1=m}^{M_1-1} \sum_{k_2=m}^{M_1-1} \Delta_{k_1}(\act) \Delta_{k_2}(\act) f_{k_1, k_2}(t).
\end{IEEEeqnarray*}
For $k_1 = k_2 = m$, we have
\begin{IEEEeqnarray*}{+rCl+x*}
f_{m, m}(\tau(1-t)) & = & \w{\s_m}(\tau(1-t)) \stackrel{\text{\Cref{lemma:asymptotics_at_the_boundary_of_hat_s}}}{=} (-1)^{m+1}\Delta_m(\act)^2 \tau^{m+1} b_m t^{m+1/2} + \calQ_{-1, m+1/2}(t)~.
\end{IEEEeqnarray*}
For $k_1 > m$ or $k_2 > m$, \Cref{lemma:_product_coeffs_asymptotics} yields
\begin{align*}
	f_{k_1,k_2}(\tau(1-t)) = \calQ_{-1, (k_1+k_2)/2}(t) = \calQ_{-1, m + 1/2}(t).
\end{align*}
Therefore,
\begin{IEEEeqnarray*}{+rCl+x*}
\widehat{g_1}(\tau(1-t)) & = & \Delta_m(\act)^2 (-1)^{m+1} \tau^{m+1} b_m t^{m+1/2} + \calQ_{-1, m+1/2}(t)
\end{IEEEeqnarray*}
for some constant $b_m>0$.
Moreover, for $M_1 \geq 2+m$, we have $M_1/2 \geq 3/4+m/2$. Using \Cref{lemma:s_hermite_decay} and \Cref{cor:smooth_hermite_decay}, we obtain
\begin{IEEEeqnarray*}{+rCl+x*}
\widehat{\act}(t) - \widehat{g_1}(t) & = & \sum_{n=0}^\infty \left(a_n(g_2) + \sum_{k=m}^{M_1 - 1} \Delta_k(\act) a_n(\s_k)\right)^2 t^n - \sum_{n=0}^\infty \left(\sum_{k=m}^{M_1 - 1} \Delta_k(\act) a_n(\s_k)\right)^2 t^n \\
& = & \sum_{n=0}^\infty a_n(g_2) \left(a_n(g_2) + 2\sum_{k=m}^{M_1 - 1} \Delta_k(\act) a_n(\s_k)\right) t^n \\
& = & \sum_{n=0}^\infty o((n+1)^{-M_1/2}) O((n+1)^{-3/4-m/2}) t^n \\
& = & \sum_{n=0}^\infty o((n+1)^{-3/4-m/2-M_1/2}) t^n~.
\end{IEEEeqnarray*}
Let $\gamma \in \bbR$ be arbitrary. We choose $M_1$ large enough such that
\begin{equation}
3/4+m/2+M_1/2 > \lceil \gamma \rceil + 1 \label{eq:ineq_M_1}
\end{equation}
Then, \Cref{lemma:power_series_differentiability} yields
\begin{IEEEeqnarray*}{+rCl+x*}
h & \equalDef & \widehat{\act} - \widehat{g_1} \in C^{\lceil \gamma \rceil}([-1, 1])~.
\end{IEEEeqnarray*}
We now define the Taylor polynomials
\begin{IEEEeqnarray*}{+rCl+x*}
p_\tau(t) & \equalDef & \sum_{k=0}^{\lceil \gamma \rceil - 1} \frac{\frac{\diff^k}{\diff u^k} h(\tau(1-u))\vert_{u=0}}{k!} t^k = \calQ_{-1, \infty}(t)
\end{IEEEeqnarray*}
and the rest terms
\begin{IEEEeqnarray}{+rCl+x*}
r_\tau(t) & \equalDef & h(\tau(1-t)) - p_\tau(t)~. \label{eq:remainder_term_r_tau}
\end{IEEEeqnarray}
It remains to show that $r_\tau(t) = \calR_\gamma(t)$, which will then yield
\begin{IEEEeqnarray*}{+rCl+x*}
& &\widehat{\act}(\tau(1-t)) = \widehat{g_1}(\tau(1-t)) + p_\tau(t) + r_\tau(t)\\
 &=& -\Delta_m(\act)^2 \tau^{m+1} (-1)^{m+1} b_m t^{m+1/2} + \calP_{-1, m+1/2}(t) + \calR_\gamma(t)
\end{IEEEeqnarray*}
for arbitrary $\gamma \in \bbR$.

\textbf{Step 2: Analyzing the rest term.} Let $\tau \in \{-1, 1\}$ and $N \in \bbN_0$ be arbitrary. We need to show that $|r_\tau^{(N)}(t)| = O(t^{\gamma-N})$, where $r_\tau$ is defined in \Eqref{eq:remainder_term_r_tau}. For an integer $M_2 > M_1$ yet to be specified, we use \Cref{lemma:smoothing_k_times_at_zero} again to decompose
\begin{IEEEeqnarray*}{+rCl+x*}
g_2 & = & \act_{M_2} + \sum_{k=M_1}^{M_2-1} \Delta_k(\act) \s_k~. 
\end{IEEEeqnarray*}
With the index set 
\begin{IEEEeqnarray*}{+rCl+x*}
\calI \equalDef \{m, m+1, \hdots, M_2-1\}^2 \setminus \{m, m+1, \hdots, M_1-1\}^2~,
\end{IEEEeqnarray*}
we then obtain similar to the calculation above
\begin{IEEEeqnarray*}{+rCl+x*}
h(t) & = & h_1(t) + h_2(t) \\
h_1(t) & \equalDef & \sum_{(k_1, k_2) \in \calI} \Delta_{k_1}(\act)\Delta_{k_2}(\act) f_{k_1, k_2}(t) \\
h_2(t) & \equalDef & \sum_{n=0}^\infty a_n(\act_{M_2}) \left(a_n(\act_{M_2}) + 2\sum_{k=m}^{M_2-1} \Delta_k(\act) a_n(\s_k)\right) t^n \\
& = & \sum_{n=0}^\infty o((n+1)^{-M_2/2}) O((n+1)^{-3/4-m/2}) t^n = \sum_{n=0}^\infty o((n+1)^{-3/4-m/2-M_2/2}) t^n~.
\end{IEEEeqnarray*}
We now choose $M_2$ sufficiently large such that $3/4+m/2+M_2/2 > \tilde N+1$, where $\tilde N \equalDef \max \{N, \lceil \gamma \rceil\}$. \Cref{lemma:power_series_differentiability} yields
\begin{IEEEeqnarray*}{+rCl+x*}
h_2 \in C^{\tilde N}([-1, 1])~.
\end{IEEEeqnarray*} 
Since \eqref{eq:ineq_M_1} implies $(m+M_1)/2 \geq \gamma$, we obtain from \cref{lemma:_product_coeffs_asymptotics}
\begin{IEEEeqnarray*}{+rCl+x*}
h_1(\tau(1-t)) & = & \calP_{-1, (m+M_1)/2}(t) + \calR_\gamma(t) = \calP_{-1, \infty}(t) + \calR_\gamma(t)~.
\end{IEEEeqnarray*}
In other words, we can find a polynomial $p_{1, \tau}$ of degree $\leq \lceil \gamma \rceil - 1$ and a function $r_{1, \tau} \in \calR_\gamma$ such that
\begin{IEEEeqnarray*}{+rCl+x*}
h_1(\tau(1-t)) & = & p_{1, \tau}(t) + r_{1, \tau}(t)~.
\end{IEEEeqnarray*}
We therefore investigate
\begin{IEEEeqnarray*}{+rCl+x*}
r_{2, \tau}(t) & \equalDef & r_\tau(t) - r_{1, \tau}(t) = (h(\tau(1-t)) - p_\tau(t)) - (h_1(\tau(1-t)) - p_{1, \tau}(t)) \\
& = & h_2(\tau(1-t)) - p_\tau(t) + p_{1, \tau}(t)~,
\end{IEEEeqnarray*}
which satisfies $r_{2, \tau} \in C^{\tilde N}([-1, 1])$
We now distinguish two cases:
\begin{itemize}
\item If $N \geq \gamma$, we simply use the continuity of $r_{2, \tau}^{(N)}$ in $0$ to obtain
\begin{IEEEeqnarray*}{+rCl+x*}
|r_{2, \tau}^{(N)}(t)| & = & O(1) = O(t^{\gamma-N})~.
\end{IEEEeqnarray*}
\item If $N < \gamma$, we proceed differently. For $0 \leq n \leq \lceil \gamma \rceil - 1$, we have
\begin{IEEEeqnarray*}{+rCl+x*}
r_{2, \tau}^{(n)}(0) & = & r_{2, \tau}^{(n)}(0+) = r_\tau^{(n)}(0+) - r_{1, \tau}^{(n)}(0+) = 0-0 = 0~,
\end{IEEEeqnarray*}
where we used that $r_\tau^{(n)}(0+) = 0$ by construction and that $r_{1,\tau}^{(n)}(0+) = 0$ thanks to $r_{1,\tau} \in \calR_\gamma$ and $n < \gamma$.
Thus, Taylor's theorem with Peano's form of the remainder yields
\begin{IEEEeqnarray*}{+rCl+x*}
|r_{2, \tau}^{(N)}(t)| & = & \left|\sum_{k=0}^{\lceil \gamma \rceil - N} \frac{r_{2, \tau}^{(N+k)}(0)}{k!} t^k + o(t^{\lceil \gamma \rceil - N})\right| \\
& = & O(t^{\lceil \gamma \rceil - N}) = O(t^{\gamma - N})~.
\end{IEEEeqnarray*}
\end{itemize}
Since $r_{1, \tau} \in \calR_\gamma$, we obtain
\begin{IEEEeqnarray*}{+rCl+x*}
|r_\tau^{(N)}(t)| & \leq & |r_{1, \tau}^{(N)}(t)| + |r_{2, \tau}^{(N)}(t)| \leq O(t^{\gamma-N}) + O(t^{\gamma-N}) = O(t^{\gamma-N})~,
\end{IEEEeqnarray*}
which completes the proof.
\end{proof}

\section{NEURAL KERNELS} \label{sec:appendix:neural_kernel_proofs}

\subsection{Analytical formulas}\label{subsec:neural_kernels}
Recall the network's architecture in \cref{def:network}.

For the subsequent consideration concerning kernels, we compare the network behavior for two inputs $\bfx,\bar{\bfx}$. All terms $\bullet$ correspond to the inputs $\bfx$, all terms $\bar{\bullet}$ to the input $\bar{\bfx}$.
 
\begin{definition}[Neural kernels]\label{def:neural_kernels}
	Consider a network of depth $L\ge 2$ and output dimension $d_L=1$. Let $d=d_1=\dots =d_{L-1}$.
	Define the  \emph{neural network Gaussian process}-kernel $\nngp_{L}:\R^{d_0}\times \R^{d_0}\to\R$ as
	\begin{align*}
		\nngp_{L}(\bfx,\bar\bfx) \equalDef \lim_{d\to \infty} \Cov\big(\po{\bfz}{L}, \po{\bar{\bfz}}{L}\big)
	\end{align*}
	and the neural tangent kernel $\ntk_{L}:\R^{d_0}\times \R^{d_0}\to\R$ as
	\begin{align*}
	\ntk_{L}(\bfx,\bar\bfx) \equalDef \lim_{d\to\infty} \left\langle \nabla_{\bftheta}\po{\bfz}L, \nabla_{\bftheta}{\po{\bar{\bfz}}L}\right\rangle~,
	\end{align*}
	where $\nabla_{\bftheta}\po{\bfz}L$ denotes derivation of the output $\po{\bfz}L$ by all parameters $\bftheta$.
\end{definition}

\begin{lemma}\label{lemma:limit_ntk_formula}
	Let the activation function $\act:\R\to\R$ fulfill \cref{ass:act}.
	Then, the NNGP and NTK kernels introduced above converge almost surely and they are given by
	\begin{align*}
		\nngp_1(\bfx, \bar{\bfx}) \equalDef&\, \sigma_b^2 \sigma_i^2 + \sigma_w^2 \langle \bfx, \bar{\bfx}\rangle \\
		\nngp_{L}(\bfx, \bar{\bfx})  =&\, \sigma_b^2 \sigma_i^2 + \sigma_w^2 \bbE_{(u, v) \sim \bfSigma_{L-1}(\bfx, \bar{\bfx})} [\act(u) \act(v)] \\
		\ntk_{L}(\bfx,\bar{\bfx})=&\,\sigma_b^2(1-\sigma_i^2)+\nngp_{L}(\bfx,\bar{\bfx}) +\sigma_w^2
		\E_{(u,v)\sim\calN(0,\bfSigma_{L-1}(\bfx,\bar{\bfx}))} \sbra{\act'(u)\act'(v)}\ntk_{L-1}(\bfx,\bar{\bfx}) \\
		\bfSigma_{L}(\bfx, \bar{\bfx})  
		=&\, \begin{pmatrix}
			\nngp_{L}(\bfx, \bfx) & \nngp_{L}(\bfx, \bar{\bfx}) \\
			\nngp_{L}(\bar{\bfx}, \bfx) & \nngp_{L}(\bar{\bfx}, \bar{\bfx})
		\end{pmatrix}
	\end{align*}
	where we define 
	$\ntk_1(\bfx, \bar{\bfx})  \equalDef \sigma_b^2 (1 - \sigma_i^2) + \nngp_1(\bfx, \bar{\bfx}) $.
\end{lemma}
\begin{proof}
Denote $\tilde d_l\equalDef d$ if $l\in\setrange{1}{L-1}$ and $\tilde d_l= 1$ otherwise. Consider the inputs $\bfx,\bar{\bfx}\in\R^{d_0}$.
For a fixed hidden layer width $d\in\N$, the finite version of the NTK kernel is given by
\begin{align}
	 \nonumber
	 &\left\langle \nabla_{\bftheta}\bfz^{(L)},  \nabla_{\bftheta}\bar\bfz^{(L)} \right\rangle
	 = \sum_{l=1}^{L} \left\langle \nabla_{\bfW^{(l)}} \po{\bfz}{L} ,\nabla_{\bfW^{(l)}} { \po{\bar\bfz}{L}} \right\rangle
	  +\sum_{l=1}^{L} \left\langle \nabla_{\bfb^{(l)}} \po{\bfz}{L} ,\nabla_{{\bfb}^{(l)}} \po{\bar\bfz}{L} \right\rangle\,,\\
	  \label{eq:ntk_terms_1}&\nabla_{\bfb^{(l)}} \po{\bfz}{L} = \sigma_b \nabla_{\po{\bfz}{l}}\po{\bfz}{L}\,,\\
	 \nonumber  &\left\langle \nabla_{\bfW^{(l)}} \po{\bfz}{L}, \nabla_{\bfW^{(l)}} { \po{\bar\bfz}{L}} \right\rangle 
	   = \frac{\sigma_w^2}{{\tilde d_{l-1}}}\left\langle \nabla_{\po{\bfz}{l}} \po{\bfz}{L} {\bfx}^{(l-1) \top }, \nabla_{\po{\bar\bfz}{l}} \po{\bar\bfz}{L} \bar {\bfx}^{(l-1) \top } \right\rangle,
\end{align}
where inner products involving matrices are Frobenius inner products. Using the outer product structure we can rewrite the weight-matrix gradient as
\begin{align}
	\label{eq:ntk_terms_2}
	&\left\langle \nabla_{\bfW^{(l)}} \po{\bfz}{L}, \nabla_{\bfW^{(l)}} { \po{\bar\bfz}{L}} \right\rangle \\
	\nonumber=&\,  \sigma_w^2\left\langle \nabla_{\po{\bfz}{l}} \po{\bfz}{L},\nabla_{\po{\bar\bfz}{l}} \po{\bar\bfz}{L}\right\rangle  \frac{1}{{\tilde d_{l-1}}} \left \langle {\po{\bfx}{l-1}},{ \po{\bar {\bfx}}{l-1}} \right \rangle \\
	\nonumber=&\,  \frac{\sigma_w^2}{d_l}\left\langle \sqrt{d_l}\nabla_{\po{\bfz}{l}} \po{\bfz}{L},\sqrt{d_l}\nabla_{\po{\bar\bfz}{l}} \po{\bar\bfz}{L}\right\rangle  \frac{1}{{\tilde d_{l-1}}} \left \langle {\po{\bfx}{l-1}},{ \po{\bar {\bfx}}{l-1}} \right \rangle ~.
\end{align}

 The occurring backpropagation terms can be recursively unrolled as
\begin{align*}
	\po{\nabla_{\po{\bfz}{L}} \bfz }{l}&=
	\frac{\sigma_w}{\sqrt{d_l}}\bfW^{(l+1)\top} \po{\nabla_{\po{\bfz}{l+1}} \bfz }{L} \odot \act'(\po{\bfz}{l}) &&1\le l\le L-1~,\\
	\po{\nabla_{\po{\bfz}{L}} \bfz}{L}&= 1,&&
\end{align*}
where $\odot$ denotes component-wise multiplication. 

In order to rigorously calculate the NNGP and the NTK, we use the simplified \texttt{NETSOR$^\top$} program from Section 7 in \cite{yang_tensor_programs_II_ntk_for_any_architecture}. 
To this end, define the set of $d$-dimensional initial vectors $\calV \equalDef \{ \sigma_w \po{\bfW}{1}\bfx, \sigma_w \po{\bfW}1 \bar\bfx, \bfW^{(L)\top} , \sigma_i\sigma_b\po {\bfb}1,\sigma_i\sigma_b\po {\bfb}2, \dots,\sigma_i\sigma_b\po {\bfb}{L-1}\}$.
For $1\le i\le d$, 
the process
$(v_i)_{v\in\calV}$ 
is a centered Gaussian process distributed as $(\rand v )_{v\in\calV}$, where\footnote{Note that $\rand v$ is a real-valued random variable for $v\in \calV$.} $\Cov(\rand{\po{\bfW}{1}\bfy},\rand{
	\po{\bfW}1 \bar\bfy})= \scal{\bfy}{\bar\bfy}$ for $\bfy,\bar \bfy\in \{ \bfx,\bar\bfx\}$, $\Cov(\rand{\po{\bfb}{l}},\rand v) = \bbone_{\{\po{\bfb}{l}\}}(v)   $ and $\Cov\left( \rand{\bfW^{(L)\top}},v\right)= \bbone_{\{\bfW^{(L)\top}\}}(v)$. 
Define the set of $\bbR^{d\times d}$ random matrices $\calW \equalDef \{{\tilde \bfW}^{(2)},\dots,{\tilde \bfW}^{(L-1)}\}$ by ${\tilde \bfW}^{(i)}_{j k}\sim \calN(0,\sigma_w^2/d)$.
Now we recursively define the vectors 
\begin{alignat}{2}
	\label{eq:tensor_program}
	\po{\bfh}{1}&\equalDef \sigma_w \po{\bfW}{1}\bfx,&&\,\po{\bfh}{l}\equalDef \po{\tilde{\bfW}}{l} \po{\bfx}{l-1}\\
	\nonumber
	\po{\bfz}{l} &\equalDef \po{\bfh}{l} +\sigma_b\sigma_i \po{\bfb}{l},&&\,\po{\bfx}{l}\equalDef \act(\bfz^{(l)}) ,\\
	\nonumber
	\d\po{\bfz}{L-1}&\equalDef \bfW^{(L)\top} \odot\act'(\po{\bfz}{L-1}) ,\qquad\qquad\qquad\quad&& \d \po{\bfh}{l}\equalDef{\tilde{\bfW}}^{(l+1)\top}\d \po{\bfz}{l+1}, \\
	\nonumber
	\d {\bfz}^{(l)} &\equalDef \d \po{\bfh}{l} \odot \act'(\po{\bfz}{l}),&&
\end{alignat}
where the index $l$ has the range $1\le l\le L-1$ for $\po{\bfx}{l}$, $1\le l\le L-2$ for $\d \po{\bfh}{l}, \d \po{\bfz}l$ and  $2\le l\le L-1$ for $\po{\bfh}{l},\po{\bfz}{l}$. 
Also note $ \d {\bfz}^{(l)} =  \frac{1}{\sqrt {d}} \po{\nabla_{\po{\bfz}{l}} \bfz }{L}$ for $1\le l\le L-1$ due to the \glqq missing normalization\grqq in the definition of $\d \bfz^{(L-1)}$.
The individual vectors in \cref{eq:tensor_program} are given by a non-linear operation or a matrix operation as in \cite[Box 1 p. 7]{yang_tensor_programs_II_ntk_for_any_architecture}.
This enables applying Theorem 7.2 and Box 1 from \cite{yang_tensor_programs_II_ntk_for_any_architecture}, which states that limits of the kernels
\begin{align*}
	B^{(l)}(\bfx,\bar\bfx)&\equalDef \lim_{d\to\infty} \frac{1}{d} \scal{\po{\bfz}{l}}{\po{\bar\bfz}{l}}, \\
	C^{(l)}(\bfx,\bar\bfx)&\equalDef \lim_{d\to\infty} \frac{1}{d}\scal{\po{\bfx}{l}}{\po{\bar\bfx}{l}}, \\
	D^{(l,L)}(\bfx,\bar\bfx)&\equalDef \lim_{d\to\infty} \frac{1}{d}\scal{\po{\d\bfz}{l}}{\po{\d\bar\bfz}{l}},\\
	E^{(l)}(\bfx,\bar\bfx) &\equalDef \lim_{d\to \infty} \frac{1}{d}\scal{\act'(\po{\bfz}{l})}{\act'(\po{\bar\bfz}{l})}.
\end{align*}
exist almost surely and furthermore yields recursive formulas for those limits.
Note that the backpropagation terms $D^{(l,L)}$ depend on the network's depth $L$.
By the independence of the parameters in the network we have 
\begin{align*}
	&\nngp_{L}(\bfx,\bar\bfx) = \lim_{d\to \infty} \Cov\big( \po{\bfz}{L} , \po{\bar\bfz}{L}\big) \\
	&=\, \lim_{d\to \infty} \frac{\sigma_w^2}{d} \E \big[ \bfx^{(L-1)\top} \bfW^{(L)^\top} \bfW^{(L)} \po{\bar\bfx}{L-1}\big] + \sigma_i^2\sigma_b^2 \E \big[(\bfb^{(L)})^2\big]\\
	=&\, \lim_{d\to\infty} \sigma_w^2\E\frac{1}{d}  \scal{\bfx^{(L-1)}}{{\bar {\bfx}}^{(L-1)}} + \sigma_i^2\sigma_b^2\\
	=&\, \sigma_w^2 C^{(L-1)}(\bfx,\bar\bfx)+ \sigma_i^2\sigma_b^2~.
\end{align*}
Here, the convergence of the expectation follows from \citet[Proposition G.4]{yang2019tensor}.
By \cref{eq:ntk_terms_1,eq:ntk_terms_2} the NTK-kernel can be expressed as
\begin{align*}
	\ntk_{L}(\bfx,\bar\bfx )&= \sigma_w^2 \sum_{l=1}^{L}D^{(l)}(\bfx,\bar\bfx) C^{(l-1)}(\bfx,\bar\bfx) + \sigma_b^2 \sum_{l=1}^{L} D^{(l)}(\bfx,\bar\bfx)~.
\end{align*}
For a kernel $k$ denote $\Sigma_k (\bfy,\bar\bfy) \equalDef \begin{psmallmatrix}
	k(\bfy,\bfy) & k(\bfy,\bar\bfy)\\
	k(\bfy,\bar\bfy) & k(\bar\bfy,\bar\bfy)
\end{psmallmatrix}
$.
Using \cref{eq:tensor_program} and \cite{yang_tensor_programs_II_ntk_for_any_architecture}, the above kernels are recursively given by
\begin{alignat}{2}
	\label{eq:ntk_associated_kernels_recursion}
	B^{(l)} (\bfx,\bar\bfx) &= \sigma_w^2 C^{l-1} (\bfx,\bar\bfx)+ \sigma_b^2\sigma_i^2,&&\\
	\nonumber C^{(l)}(\bfx,\bar\bfx) &= \E_{(u,v)\sim\calN(0,\Sigma_{B^{(l)}}(\bfx,\bar \bfx))} \sbra{\act(u)\act(v)}, \quad &&(l\ge 1)\\
	\nonumber C^{(0)}(\bfx,\bar\bfx) &= \scal{\bfx}{\bar\bfx},&&\\
	\nonumber D^{(l,L)}(\bfx,\bar\bfx)&= \sigma_w^2 D^{{(l+1)}}(\bfx,\bar\bfx) E^{(l)}(\bfx,\bar\bfx),\quad &&l\le L-1,\\
	\nonumber D^{(L,L)}(\bfx,\bar\bfx)&=1,&&\\
	\nonumber E^{(l)} (\bfx,\bar\bfx) &= \E_{(u,v)\sim\calN(0,\Sigma_{B^{(l)}}(\bfx,\bar \bfx))} [\act'(u)\act'(v)],&&
\end{alignat}
and a recursive formula for $D^{(l,L)}$ over the depth $L$ follows as
\begin{equation*}
D^{{(l,L+1)}}= \sigma_w^{2} E^{(l)} \cdot \hdots \cdot \sigma_w^2 E^{(L)} = \sigma_w^2 E^{(L)} D^{(l,L)}~.
\end{equation*}
The recursive formula for $\nngp_{L}(\bfx,\bar\bfx)$ can directly be derived from \cref{eq:ntk_associated_kernels_recursion}. We investigate $\ntk_{L}(\bfx,\bar\bfx)$ by induction. 
For $L=1$ there is nothing to show. 
Assume that the claimed formula holds for $L\in \N$. Then we have
\begin{align*}
	\ntk_{L+1}(\bfx,\bar\bfx) &= 
	\sigma_w^2 \sum_{l=1}^{L+1} D^{(l,L+1)} (\bfx,\bar\bfx)C^{(l-1)}(\bfx,\bar\bfx) x
	+ \sigma_b^2 \sum_{l=1}^{L} D^{(l,L+1)}(\bfx,\bar\bfx)\\
	&= \sigma_w^2 E^{(L)}(\bfx,\bar\bfx)
	\bra{ \sigma_w^2\sum_{l=1}^{L} D^{(l,L)} (\bfx,\bar\bfx)C^{(l-1)}(\bfx,\bar\bfx)+\sigma_b^2  \sum_{l=1}^{L} D^{(l,L)}(\bfx,\bar\bfx) }  \\
	&~~~~+  \sigma_w^2 C^{(L)}(\bfx,\bar\bfx)   + \sigma_b^2\\
	&= \sigma_b^2(1-\sigma_i^2)+\nngp_{L+1}(\bfx,\bar{\bfx}) +\sigma_w^2
	\E_{(u,v)\sim\calN(0,\bfSigma_{L})} \sbra{\act'(u)\act'(v)}\ntk_{L}(x,x'). \quad \qedhere
\end{align*} 
\end{proof}
A dot-product kernel $k$ on the sphere $\bbS^d$ is conveniently described by a function $\kappa:[-1,1] \to \R$ by defining $\kappa(\scal xy) \equalDef k(x,y)$. We translate the NNGP- and NTK-recursion of \cref{lemma:limit_ntk_formula} to a recursion for the corresponding functions $\kappa$ describing the restriction of these kernels to the sphere.
\lemmaKernelsSphere*
	
	\begin{proof}
		As $\sigma_w^2>0$ and $\act$ is not almost surely equal to zero, $\alpha_l>0$ follows for all $l$. By a simple induction we see that $\alpha_l = \nngp_l(1)$ holds for all $l \ge 1$. The identities for the NNGP can be shown by induction on $l$ as well. The induction base $\nngp_{1}(\bfx,\bar \bfx)= \knngp_{1}(\scal{\bfx}{\bar\bfx})$ is straightforward. 		
		For the induction step
		\cref{lemma:limit_ntk_formula} yields
		\begin{align*}
			\nngp_{l}(\bfx, {\bar\bfx}) &= \sigma_b^2\sigma_i^2 +\sigma_w^2 \E_{(u,v)\sim\calN(0,\bfSigma_{l-1}(\bfx,\bar{\bfx}))}\sbra{ \act(u)\act (v)},\\
			\bfSigma_{l-1}(\bfx,\bar{\bfx})&=\begin{psmallmatrix}
				\nngp_{l-1}(\bfx,\bfx) & \nngp_{l-1}(\bfx,\bfx')\\
				\nngp_{l-1}(\bfx,\bfx')& \nngp_{l-1}(\bfx',\bfx')
			\end{psmallmatrix}
			= {\alpha_{l-1}} \cdot \begin{psmallmatrix}
				1& \knngp_{l-1}(\scal {\bfx}{\bar\bfx})/\alpha_{l-1}\\
				\knngp_{l-1}(\scal {\bfx}{\bar\bfx})/\alpha_{l-1}&1
			\end{psmallmatrix}
		\end{align*}
		and the identity $\nngp_l(\bfx,\bar{\bfx})= \knngp_l(\scal{\bfx}{\bar{\bfx}})$ follows directly from the definition of the dual activation (\ref{definition:dual_and_rescaled_activation_function}).
		A similar argument shows $\ntk_{l}(\bfx, \bar\bfx) = \kntk_{l}(\scal {\bfx}{\bar\bfx})$.
	\end{proof}

\subsection{Even and odd parts of neural kernels}
\begin{lemma}[Even/odd algebra] \label{lemma:even_odd_algebra}
	Let $f, g: \bbR \to \bbR$. Then,
	\begin{enumerate}[(a)]
		\item $\evenpart{f+g} = \feven + \geven$ and $\oddpart{f+g} = \fodd + \godd$.
		\item $\evenpart{f \cdot g} = \feven \cdot \geven + \fodd \cdot \godd$ and $\oddpart{f \cdot g} = \feven \cdot \godd + \fodd \cdot \geven$.
		\item If $g$ is odd, $\evenpart{f \circ g} = \feven \circ g$ and $\oddpart{f \circ g} = \fodd \circ g$. If $g$ is even, $f \circ g$ is even.
		\item If $f'$ is a weak derivative of $f$, then $\evenpart{f'}$ is a weak derivative of $\fodd$ and $\oddpart{f'}$ is a weak derivative of $\feven$. %
	\end{enumerate}
	\begin{proof}
		Statements (a) -- (c) are straightforward to prove. For (d), we obtain for $x \in \bbR$:
		\begin{IEEEeqnarray*}{+rCl+x*}
			\int_0^x \evenpart{f'}(t) \diff t & = & \int_0^x \frac{f'(t) + f'(-t)}{2} \diff t \\
			& = & \frac{1}{2} \left(\int_0^x f'(t) \diff t + \int_0^x f'(-t) \diff t\right) \\
			& = & \frac{1}{2} \left(\int_0^x f'(t) \diff t - \int_0^{-x} f'(t) \diff t\right) \\
			& = & \frac{1}{2} ((f(x) - f(0)) - (f(-x) - f(0))) \\
			& = & \fodd(x)~.
		\end{IEEEeqnarray*}
		This shows that $\evenpart{f'}$ is a weak derivative of $\fodd$, and a similar computation shows that $\oddpart{f'}$ is a weak derivative of $\feven$.
	\end{proof}
\end{lemma}
\begin{proposition}[Special cases for NNGP/NTK with even/odd functions] \label{prop:even_odd_kernels}
	Let the activation function $\act$ fulfill \cref{ass:act}.
	\begin{enumerate}[(a)]
		\item Let $\sigma_b^2 \sigma_i^2 = 0$. If $\act$ is even/odd, then $\knngp_l$ is even/odd for all $l \geq 2$. 
		\item Let $\sigma_b^2 \sigma_i^2 = 0$. 
		Then, $\evenpart{\knngp_{2, \act}} = \knngp_{2, \acteven}$ and $\oddpart{\knngp_{2, \act}} = \knngp_{2, \actodd}$.
	\end{enumerate}
	\begin{enumerate}[(a)]
		\item[(c)] Let $\sigma_b^2 = 0$. If $\act$ is even/odd, then $\kntk_l$ is even/odd for all $l \geq 2$.
		\item[(d)] Let $\sigma_b^2 = 0$. Then, $\evenpart{\kntk_{2, \act}} = \kntk_{2, \acteven}$ and $\oddpart{\kntk_{2, \act}} = \kntk_{2, \actodd}$.
	\end{enumerate}
Here we denote the activation function in the index to clarify the network architecture the kernels belong to.
	\begin{proof}
		\leavevmode
		\begin{enumerate}[(a)]
			\item Since $\sigma_b^2 \sigma_i^2 = 0$, $\knngp_1$ is odd. If $\act$ is even/odd, then so is $\scaledfn{\act}{\sqrt{\alpha_l}}$ and by \Cref{lemma:dual_properties} also $\widehat{\scaledfn{\act}{\sqrt{\alpha_l}}}$. The claim then follows by induction using \Cref{lemma:even_odd_algebra} (c). %
			\item We have $\knngp_{2, \act}(t) = \sigma_w^2 \widehat{\scaledfn{\act}{\sqrt{\alpha_1}}}(t)$. Note that $\alpha_1$ does not depend on $\varphi$. By \Cref{lemma:dual_properties}, we obtain
			\begin{IEEEeqnarray*}{+rCl+x*}
				\evenpart{\knngp_{2, \act}}(t) = \sigma_w^2 \widehat{\evenpart{\scaledfn{\act}{\sqrt{\alpha_1}}}}(t) = \sigma_w^2\widehat{(\acteven)_{\cdot\sqrt{\alpha_1}}}(t) = \knngp_{2, \acteven}(t)
			\end{IEEEeqnarray*}
			and similar for the odd part. 
			\item
			Since $\sigma_b^2 = 0$, $\kntk_1$ is odd. Moreover,
			\begin{IEEEeqnarray*}{+rCl+x*}
				\kntk_{l}(t) & = &  \sigma_w^2 \bra{\knngp_{l}(t) + \kntk_{l-1}(t) \cdot \widehat{(\act')_{\cdot \sqrt{\alpha_{l-1}}}}(\knngp_{l-1}(t)/\alpha_{l-1})} 
			\end{IEEEeqnarray*}
			holds and we obtain the claim for $l=2$.  
			We use \Cref{lemma:even_odd_algebra} and \Cref{lemma:dual_properties} to obtain that $\widehat{(\act')_{\sqrt{\alpha_l}}}$ is odd/even if $\act$ is even/odd, and the claim holds for $l\ge 3$ by induction since $\knngp_{l}(t)$ is even/odd by (a).
			\item Just as in (c), this follows from \Cref{lemma:dual_properties}, \Cref{lemma:even_odd_algebra} and (b). \qedhere
		\end{enumerate}
	\end{proof}
\end{proposition}
\subsection{Adapting the main theorem from \cite{bietti_deep_2021}}
Here, we derive \Cref{thm:bietti_bach_adapted}, an adapted version of the main theorem of \cite{bietti_deep_2021} that is closer to our notation. First, we restate the original theorem with minor adaptations (using $\tau$, expanding at $t=0$ instead of $t \in \{-1, 1\}$, rewriting derivatives of powers, using $\bbS^d$ instead of $\bbS^{d-1}$). Here, the notation $a_n \sim b_n$ means $\lim_{n \to \infty} a_n/b_n = 1$.

\begin{theorem}[Theorem 7 in \cite{bietti_deep_2021}, arXiv version 4]\label{thm:bietti_bach_original}
	Let $k:[-1,1]\to\R $ be a function that is $C^\infty$ on $(-1,1)$ and has for $\tau\in\set{1,-1} $ the following expansion for $t\searrow0$\,\emph :
	\begin{align}
		\label{eq:bietti_bach_main_decomposition}
		k(\tau(1-t))=p_{\tau}(t) +\sum_{j=1}^r c_{j,\tau} t^{\nu_j} + O(t^{\nu_1+1+\varepsilon}),
	\end{align}
	where $p_\tau$ are polynomials and $0<\nu_1<\dots<\nu_r$ are not integers and $0 < \varepsilon < \nu_2-\nu_1$. Assume further that the derivatives $\po ks$ have for any $s\in\N_0$ the following expressions for $t\searrow0$, where $\bra{ t^{\nu_j}}^{(s)}$ is the $s$-th derivative of $t \mapsto t^{\nu_j}$:
	\begin{align}
		\label{eq:bietti_bach_main_decomposition2}
		\po ks(\tau(1-t))= p_{s,\tau} (t) +\sum_{j=1}^{r} c_{j,\tau} \bra{ t^{\nu_j}}^{(s)} + O\bra{t^{\nu_1+1+\varepsilon-s}}
	\end{align}
	for some polynomials $p_{s,\tau}$. %
	Then, for $d \geq 1$, the eigenvalues $\mu_k = \mu_k(\kappa, d)$ as defined in \Cref{sec:preliminaries} satisfy, for an absolute constant $C(d,\nu_1)$,
	\begin{itemize}
		\item For $k$ even, if $c_{1,1} \ne -c_{1,-1}$: $\mu_k \sim (c_{1,1} +c_{1,-1})C(d,\nu_1)k^{-d-2\nu_1}$;
		\item For $k$ even, if $c_{1,1}= -c_{1,-1}$: $\mu_k= o(k^{-d-2\nu_1})$;
		\item For $k$ odd, if $c_{1,1} \ne c_{1,-1}$: $\mu_k \sim (c_{1,1} - c_{1,-1})C(d,\nu_1)k^{-\tilde d-2\nu_1}$;
		\item For $k$ odd, if $c_{1,1}= c_{1,-1}$: $\mu_k= o(k^{-d-2\nu_1})$.
	\end{itemize}
\end{theorem}

We reformulate this in terms of $\calQ$, see \Cref{def:boundary_function_classes}. This perspective allows conveniently handling the $O(t^{\nu_1+1+\varepsilon})$-term and its derivatives. 
For greater precision, we use the following variant of Lemma B.2 in \citet{haas_mind_2023} to investigate the sign of eigenvalues $\mu_i$:

\begin{lemma}[Guaranteeing strictly positive eigenvalues for kernels on spheres]\label{lemma:strictly_positive_Eigenvalues}
	Let $\kappa:[-1,1] \to \R$ be continuous, let $d\ge 1$ and consider the radial kernels 
	\begin{align*}
		k_d&: \bbS^d \times \bbS^d \to \R, k_d(x,y) \equalDef \kappa(\scal xy)\\
		k_{d+2}&: \bbS^{d+2} \times \bbS^{d+2} \to \R, k_{d+2}(x,y) \equalDef \kappa(\scal xy).
	\end{align*}
	Suppose that $k_{d+2} $ is a kernel. Then, $k_d$ is a kernel. 
	If an eigenvalue $\mu_{\hat l}$ of $k_d$ fulfills $\mu_{\hat l}>0$ and $\hat l$ is even/odd, then for all even/odd $l\le \hat l$  
	\begin{align}
		\label{eq:strictly_pos_ev}
		\mu_{l}>0
	\end{align} 
	follows. 
	Especially, if \cref{eq:strictly_pos_ev} holds for infinitely many even/odd $l$, then it holds for all even/odd $l$.
\end{lemma}

The only difference between Lemma B.2 in \citet{haas_mind_2023} and the above lemma is that the eigenvalues are described in greater detail here, the proof remains the same.

\thmBiettiBachAdapted*
\begin{proof}
Set $\gamma \equalDef \beta + 2$. Fix $\tau \in \{-1, 1\}$. Then, we can write $\kappa(\tau(1-t)) = b_\tau t^\beta + p(t) + r(t)$ for some $p \in \calP_{-1, \beta}$ and $r \in \calR_{\gamma}(t)$. To apply \Cref{thm:bietti_bach_original}, we can rewrite $b_\tau t^\beta = p_\tau(t) + \sum_{j=1}^r c_{j,\tau} t^{\nu_j}$ with $\nu_1 = \beta$ and $c_{1,\tau} = b_\tau$. The rest term $r(t)$ is covered by the $O(t^{\nu_1+1+\eps})$ term in \Cref{thm:bietti_bach_original} by setting $\eps \equalDef \frac{1}{2} \min\{1, \nu_2 - \nu_1\}$, and the expressions for the derivatives in \Cref{thm:bietti_bach_original} are satisfied by the definition of $\calR_\gamma$.

Hence, we can apply \Cref{thm:bietti_bach_original} to obtain the correct asymptotic decay rates, and it remains to show that in the cases (a) and (c), all eigenvalues $\mu_l$ are strictly positive for all $l\in\bbN_0$. This claim follows from the following considerations:
\begin{itemize}
	\item 
	$C(d, \nu_1) \neq 0$: While this is not shown in \cite{bietti_deep_2021}, it follows by considering $\tilde\kappa(t) \equalDef (1-t)^{\nu_1}$:
	This choice of $\kappa$ satisfies the assumptions of \Cref{thm:bietti_bach_original} with $c_{1, 1} = 1 \neq 0 = c_{1, -1}$, hence the cases (a) or respectively (c) apply. 
	
	Suppose $C(d, \nu_1) = 0$, then we obtain $\mu_k \sim 0$ and hence only finitely many $\mu_k$ are nonzero. Following Appendix A in \cite{bietti_deep_2021}, $\tilde\kappa$ can be expressed as
	\begin{IEEEeqnarray*}{+rCl+x*}
		\tilde\kappa(t) & = & \sum_{l=0}^\infty \mu_l N_{d,l} P_l(t)~,
	\end{IEEEeqnarray*}
	where the $P_l$ are Legendre polynomials of degree $l$ for the dimension $d+1$. 
	Hence, $\tilde\kappa$ would be a polynomial, contradicting the definition $\tilde \kappa(t) = (1-t)^{\nu_1}$.
\item Since we assumed $k_{\kappa, d}$ to be a (positive semi-definite) kernel, we cannot have negative eigenvalues, hence $(c_{1,1} + c_{1,-1})C(d,\nu_1) > 0$ or respectively $(c_{1,1} - c_{1,-1})C(d,\nu_1) > 0$ follows in the cases (a) and (c).
\item The notation $\sim$ from \Cref{thm:bietti_bach_original} allows a finite number of $\mu_{l}$ to be zero. 
However, as there are infinitely many both even and odd indices for which $\mu_l$ is strictly positive, \cref{lemma:strictly_positive_Eigenvalues} shows that $\mu_l>0$ holds for all $l\in\N_0$.

\qedhere
\end{itemize}
\end{proof}

\subsection{Eigenvalue decay} \label{sec:appendix:eigenvalue_decay}

\outcomment{
\begin{remark}
\todo{This is only to discuss the proof idea for the polynomial case.}

If $\act$ is a polynomial, then $\widehat \act$ is also a polynomial with the same even and odd degrees, call them $N_{\text{even}}$ and $N_{\text{odd}}$. The coefficients are non-negative, hence there should be no cancellations when composing or multiplying these polynomials. The rescaling should also not affect the degrees.

Now, suppose that $N_{\text{even}} > N_{\text{odd}}$. Consider first the NNGP part. Then, clearly the even degree of the $(L-1)$-fold composition of $\widehat \act$ has even degree $N_{\text{even}}^{L-1}$. However, the odd degree is more complicated. From the binomial theorem one can see that when taking a power, the even-odd degree gap remains the same (?). When taking an even power, odd degrees can become even, which can complicate things.
For the NTK part, we also need to consider multiplication of the $(L-2)$ NNGP term with the $\widehat{\act'}$ term. The latter term has now $N_{\text{odd}} > N_{\text{even}}$, but multiplying the odd term from that with the even term of the NNGP part should usually give a lower degree.

Probably the best thing is to create lemmas for multiplication and composition of such terms (first composition with a single power, then with a whole polynomial).
\end{remark}
}

Some preliminary constructions useful for both the NNGP and the NTK kernel are first established here, stressing the perspective of layerwise computation in the neural network.

We will study $\knngp_L,\kntk_L:[-1,1]\to\R$ as compositions of functions which break the kernels given by depth $L$ down to a \quot{composition} of kernels of depth $1$.

Recall the recursive form of the NNGP-kernel on the sphere given in \cref{lemma:restriction_to_the_unit_sphere} and the dual activation function $\w\act$ in \cref{definition:dual_and_rescaled_activation_function}.

\begin{numbered_notation}[NNGP- and NTK related terms]\label{eq:recursive_formulas}
Recursively define the functions $g_l, G_l,H_l: [-1, 1] \to [-1, 1]$ by
\begin{IEEEeqnarray*}{+rCl+x*}
	\nonumber\alpha_{l} & \equalDef &  \knngp_{l,\act}(1) ,\\
	\nonumber
	g_1(t) & \equalDef & \frac{\sigma_b^2 \sigma_i^2 + \sigma_w^2 t}{\alpha_1} , \\
	\nonumber g_{l}(t) & \equalDef & \frac{\sigma_b^2 \sigma_i^2 + \sigma_w^2 \widehat{\scaledfn{\act}{\sqrt{\alpha_{l-1}}}}(t)}{\alpha_{l}} ,\\
	\nonumber G_l(t) & \equalDef & g_l \circ \dots \circ g_1 (t),\\
	\nonumber H_1(t)&\equalDef & \sigma_b^2 +\sigma_w^2 t ,\\
	\nonumber H_{l}(t)&\equalDef & \sigma_b^2(1-\sigma_i^2) + \alpha_{l} G_{l}(t) + \sigma_w^2 \w{\scaledfn{\act'}{\sqrt{\alpha_{l-1}}}}\bra{G_{l-1}(t)} H_{l-1}(t)~. & \qedhere
\end{IEEEeqnarray*}
\end{numbered_notation}
By the recursive formulation of $\knngp_l,\kntk_l$ in \cref{lemma:restriction_to_the_unit_sphere} we observe 
for all $t\in [-1,1]$
\begin{align*}
	\alpha_l G_l (t)&= \knngp_l (t) ,\\
	H_l(t)&= \kntk_l (t)~. 
\end{align*}

\begin{lemma}\label{lemma:layer_nngp_contraction_property}
	Under the assumptions of \cref{lemma:restriction_to_the_unit_sphere}, we have for any $l\ge 1$
	\begin{align*}
		\max_{t\in [-1,1]}\abs{g_l(t)} = g_l(1)=1~.
	\end{align*}
	Furthermore,
	\begin{align*}
		G_l[-1,1) \subset (-1,1)
	\end{align*}
	holds for all $l\ge 1$ when $\sigma_b^2\sigma_i^2>0$, and for all $l\ge 2$ when $\act $ is neither even nor odd.
	\begin{proof}
		Follows from the definition of $g_l$  and \cref{lemma:dual_properties}.
	\end{proof}
\end{lemma}

\begin{lemma}[Behavior at $1$]\label{lemma:asymptotics_at_one}
	Let the activation function $\act$ fulfill \cref{ass:act} and
	let $m \equalDef \inf\set{k\in\N_0 \mid \po{\act}{k}(0+)\ne \po{\act}{k}(0-)} $ be its smoothness, see \cref{def:smoothness_of_an_activation_function}.
	Let $l\ge 2$.
	\begin{enumerate}
		\item 
		Let $m=0$. 
		Then we have for $t\searrow0$
		\begin{align*}
			g_l(1-t) &= 1-  \bra{c_l t^{1/2} + \calQ_{0,1/2}(t)},\\
			G_l(1-t) &= 1 - \bra{C_l t^{2^{1-l}} + \calQ_{0,2^{1-l}}(t)},
		\end{align*}
		for constants $c_l,C_l >0$. 
		\item 
		Let $m\in\N_{\ge 1}$. Then we have for $t\searrow0$
		\begin{align*}
			g_l(1-t) &= 1 -\bra{ b_l t + (-1)^{m}c_l  t^{m+1/2} +\calQ_{1,{m+1/2}}(t)},\\
			G_l(1-t)&= 1 - \bra{ B_l t + (-1)^{m}C_l t^{m+1/2} + \calQ_{1,{m+1/2}}(t)},\\
			H_l(1-t)&=A_l - \bra{(-1)^{m+1}D_l t^{m-1/2} +\calQ_{0,m-1/2}(t)}
		\end{align*}
		for constants $b_l=g_l'(1),c_l,A_l,B_l,C_l,D_l,>0$.
		Furthermore, we have for $l\ge 3$ the recursive relation $C_l\ge g'_l(1) C_{l-1}$. 
		\item Let $m=\infty$. Then we have for any $\tau \in\set{\pm 1}, q\in \R$ and $t\searrow 0$
		\begin{align*}
			g_l(\tau(1-t)) &= \calQ_{-1,q}(t),\\
			G_l(\tau(1-t))&= \calQ_{-1,q}(t),\\
			H_l(\tau(1-t))&= \calQ_{-1,q}(t).
		\end{align*}
	\end{enumerate}
	
	\begin{proof}
	\leavevmode
		\begin{enumerate}
			\item 
			The formula for $g_l$ follows from \cref{thm:smoothed_dual_activation_smooth_decomposition} and \cref{lemma:layer_nngp_contraction_property}. We obtain the formula for $G_l$ via induction, where in the base case for $l=1$ by definition $G_1(1-t)= 1 - c_1 t $ with $c_1>0$ holds. 
			For $l= 2$ we have
			\begin{align*}
				G_2(1-t) &=  g_2( G_{1} (1-t)) = g_2( 1- (1-G_{1}(1-t)))= 1 - c_2 c_1^{1/2} t^{1/2} + \calQ_{0,1/2}\bra{ c_1 t }
			\end{align*}
			as claimed. For $l\ge 3$, 
			 note that the identities $\calQ_{0,2^{2-l}}(t)= \calQ_{2^{2-l},2^{2-l}}(t)$ and $\calQ_{0,1/2}(t)=\calQ_{1/2,1/2}(t)$ hold by \cref{def:boundary_function_classes}. From \cref{prop:rules_for_Q} (e,f) we obtain
			\begin{align*}
				G_l(1-t) &=  g_l( G_{l-1} (1-t)) = g_l( 1- (1-G_{l-1}(1-t)))\\
				&= 1 - c_l \bra{ C_{l-1}t^{2^{2-l}} +\calQ_{2^{2-l},2^{2-l}} (t)}^{1/2}
				+ \calQ_{1/2,1/2}\bra{  C_{l-1} t^{2^{2-l}} +\calQ_{2^{2-l},2^{2-l}}(t)}\\
				&= 1-\bra{ C_l t^{2^{1-l}} +\calQ_{2^{1-l},2^{1-l}}(t)}
			\end{align*}			
			and the claim follows from the identity $\calQ_{2^{1-l},2^{1-l}}(t)=\calQ_{0,2^{1-l}}(t)$.
			\item 
			Define $\beta\equalDef m+1/2$.
			From \cref{thm:smoothed_dual_activation_smooth_decomposition} we obtain 
			\begin{align*}
				g_l(1-t)= (-1)^{m+1} c_l t^{\beta}+ \calQ_{-1, \beta}(t)
			\end{align*} 
			with $c_l>0$. \cref{lemma:layer_nngp_contraction_property} shows $g_l(0)=1   $ and furthermore we observe
			\begin{align*}
				\widehat{\scaledfn{\act}{\sqrt{\alpha_{l-1}}}}'(1) = \E_{u\sim \calN(0,\alpha_{l-1})} [\act'(u)^2]>0
			\end{align*}
			as $\act$ is not constant, since $1 \leq m < \infty$. From \Cref{lemma:layer_nngp_contraction_property} and using $b_l \equalDef g_l'(1) > 0$, it follows that
			\begin{align*}
				g_l(1-t) &= 1 -\bra{ b_l t + (-1)^{m}c_l  t^{m+1/2} +\calQ_{1,{m+1/2}}(t)}~.
			\end{align*} 
			Note that for the NNGP term $G_l$ we have to argue more carefully to obtain the strictly positive factor $B_l>0$ for the linear term, while this is not required for the NTK term $H_l$ where we can rely on the more precisely calculated NNGP term.
			We show the formula for $G_l$ via induction. For $l=2$ it follows directly from $G_2=g_2\circ g_1$ using $g_1(1-t) = 1 - c_1 t$.
			In the induction step, let $l\ge 3$ and define 
			\begin{align*}
				\bar G_{l-1}(t)\equalDef 1- G_{l-1}(t) =  B_{l-1} t+ (-1)^{m}C_{l-1} t^{\beta} +\calQ_{1,\beta}(t)~.
			\end{align*}  
			We rewrite
			\begin{align*}
				G_l(1-t) &=  g_l( G_{l-1} (1-t)) = g_l( 1-\bar G_{l-1}(t))\\
				&= 1 - \left( b_l \bar G_{l-1}(t)+ (-1)^m c_l \bra{\bar G_{l-1}(t)}^{\beta} + \calQ_{1,\beta}\bra{\bar G_{l-1}(t)}\right)
			\end{align*}
			and investigate the single summands. The linear summand straightforwardly yields
			\begin{align*}
				b_l \bar G_{l-1}(t) = b_l B_{l-1} t + (-1)^{m} b_l C_{l-1} t^\beta +\calQ_{1,\beta}(t).
			\end{align*}
			The power summand can be handled with \cref{prop:rules_for_Q} (e), where we write $\bar{G}_{l-1} =  b_l B_{l-1} t +\calQ_{1,1}(t)$ and consequently we have 
			\begin{align*}
				(\bar{G}_{l-1}(t))^\beta &= (b_l B_{l-1})^\beta t^\beta + \calQ_{\beta,\beta}(t) \\
				&= (b_l B_{l-1})^\beta t^\beta + \calQ_{1,\beta}(t)~.
			\end{align*}
			In order to calculate $\calQ_{1,\beta}(\bar{G}_{l-1}(t))$ we decompose $\calQ_{1,\beta}(t) = p(t) + q(t)$, where $p(t)$ is a polynomial with integer powers greater than $1$ and $q(t)\in\calQ_{\beta,\beta}(t)$.
			Using \cref{prop:rules_for_Q} (d) we have, as $p(0)= p'(0) = 0$ holds, 
			\begin{align*}
				p( \bar G_{L-1}(t)) = p( B_{l-1}t + (-1)^m C_{l-1} t^\beta +\calQ_{1,\beta}(t) ) 
				= \calQ_{1,\beta}(t)~.
			\end{align*}
			\cref{prop:rules_for_Q} (f) yields $q(t) = \calQ_{\beta,\beta}(t) = \calQ_{1,\beta}(t) $ and we obtain
			$
				\calQ_{1,\beta}\left(\bar{G}_{l-1}(t)\right) 
				= \calQ_{1,\beta}(t)
			$.
			Altogether, we see that 
			\begin{align*}
				G_l(1-t)= 1 - \big(b_lB_{l-1} t + (-1)^m (b_lC_{l-1}(b_lB_{l-1})^\beta) t^\beta +\calQ_{1,\beta}(t) \big)
			\end{align*}
			holds as desired, and furthermore we have 
			\begin{align*}
				C_l \equalDef b_lC_{l-1}(b_lB_{l-1})^\beta \ge b_l C_{l-1} = g_l'(1) C_{l-1}~.
			\end{align*}

			We move on to calculate $H_{l}$ for $l\ge 2$.
			An elementary argument yields $\kntk(1)>0$ as a consequence of $\sigma_w>0$, hence we have $A_l>0$ for all $l\ge 2$. 
			Again, the induction is based on the recursive construction in \cref{eq:recursive_formulas}, where we handle the term arising from $\w{\scaledfn{\act'}{\sqrt{\alpha_l}}}$ using \cref{thm:smoothed_dual_activation_smooth_decomposition}.
			Since $\act'$ has smoothness $m-1$ we obtain
			\begin{align*}
				g_l'(1-t)= \w{\scaledfn{\act'}{\sqrt{\alpha_l}}}(1-t)=a_l (-1)^{m} t^{\beta-1} + \calQ_{-1,\beta-1}(t)
			\end{align*}
			for some $a_l > 0$.
			 For $l=2$, the claimed form of $H_l$ now follows straightforwardly.
			 For the induction step, assume $l\ge 3$.
			 Then we have
			 \begin{align*}
			 	&  \w{\scaledfn{\act'}{\sqrt{\alpha_{l-1}}}}\bra{G_{l-1}(1-t)} 
			 	= \bar c_l+a_l (-1)^m \bra{\bar G_{l-1}(t)}^{\beta-1} +\calQ_{0,\beta-1}\bra{\bar G_{l-1}(t)}
			 \end{align*}
		 	where $\bar c_l\ge 0$ holds since we have $\w{\scaledfn{\act'}{\sqrt{\alpha_{l-1}}}}\bra{G_{l-1}(1-0)}= \w{\scaledfn{\act'}{\sqrt{\alpha_{l-1}}}}\bra{1} \ge 0 $. A similar argument as above shows 
		 	\begin{align*}
		 		\calQ_{-1,\beta-1}\bra{\bar G_{l-1}(t)}&=\calQ_{-1,\beta-1}(t),\\
		 		\bra{\bar G_{l-1}(t)}^{\beta-1}&= 
		 		B_{l-1}^{\beta-1} t^{\beta-1} +\calQ_{-1,\beta-1}(t)~.
		 	\end{align*}
	 		We define $\bar D_l\equalDef a_l B_{l-1}^{\beta-1 }$ and conclude
	 		\begin{IEEEeqnarray*}{+rCl+x*}
	 			&& \w{\scaledfn{\act'}{\sqrt{\alpha_{l-1}}}}\bra{G_{l-1}(1-t)} H_{l-1}(1-t) \\
	 			& = & \bra{\bar c_l+ (-1)^m \bar D_l t^{\beta-1} +\calQ_{0,\beta-1}(t) } \bra{A_{l-1} - \bra{(-1)^{m+1}  D_{l-1} t^{\beta-1} +\calQ_{0,\beta-1}(t)}}\\
	 			& = & A_l - \bra{ (-1)^{m+1} D_l t^{\beta-1} + \calQ_{0,\beta-1}(t)},
	 		\end{IEEEeqnarray*}
 			where $D_l>0$, and the claim follows.
 			
			\item %
			Choose a natural number $m\ge q-1/2$ and apply \cref{thm:smoothed_dual_activation_smooth_decomposition} to obtain
			\begin{align*}
				g_l(\tau (1-t)) = \calQ_{-1,m+1/2}(t) = \calQ_{-1,q}(t).
			\end{align*} 
			Using \cref{prop:rules_for_Q} c), we obtain the claim by induction. \qedhere
		\end{enumerate}
	\end{proof}
\end{lemma}

\outcomment{
\begin{proposition}\label{prop:polynomial_network_kernels}
	Let $\act$ be an activation function fulfilling \cref{ass:act}. Then, the following are equivalent
	\begin{enumerate}
		\item $\act$ is a polynomial,
		\item there exists $l\ge 2$ such that $G_l$ is a polynomial,
		\item there exists $l\ge 2$ such that $H_l$ is a polynomial,
		\item for all $l\ge 2$, $G_l$ is a polynomial,
		\item for all $l\ge 2$, $H_l$ is a polynomial~.
	\end{enumerate}
\end{proposition}
}

\begin{theorem}\label{thm:kernel_ev_rates}
Let the assumptions of \cref{lemma:restriction_to_the_unit_sphere} be satisfied. Choose a number of layers $L \geq 2$ and a parity $r \in \{0, 1\}$ of the eigenvalues to be analyzed. Define $\act^{[0]} \equalDef \acteven$ and $\act^{[1]} \equalDef \actodd$.

\begin{enumerate}
\item[(NNGP)] Define the simplified activation $\tilde\act$ by
\begin{IEEEeqnarray*}{+rCl+x*}
\tilde\act & \equalDef & \begin{cases}
\act^{[r]} &,\text{ if $\sigma_b^2 \sigma_i^2 = 0$ and ($L=2$ or $\act$ is even or odd)} \\
\act &,\text{ otherwise,}
\end{cases} 
\end{IEEEeqnarray*}
and let $s 
\equalDef \smoothness{\tilde \act}$. 
Note that in all cases, the kernels $\nngp_L$ and $\ntk_L$ refer to the \emph{original }activation $\act$, while the smoothness parameter $s$ is determined by the simplified activation $\tilde \act$.
\begin{enumerate}
\item[(1.1)] If $s = 0$, then $\mu_{2\ttt+r, d}(\nngp_L) = \Theta_{\forall \ttt} \bra{(2\ttt+r+1)^{-d-2^{2-L}}}$. 
\item[(1.2)] If $1 \leq s < \infty$, then $\mu_{2\ttt+r, d}(\nngp_L) = \Theta_{\forall \ttt}\left({(2\ttt+r+1)^{-d-2s-1}}\right)$.
\item[(1.3)] If $s = \infty$ and $\tilde\act$ is not a polynomial, then $\mu_{2\ttt+r, d}(\nngp_L) > 0$ for all $\ttt$
and $\mu_{2\ttt+r, d}(\nngp_L) =o_{\forall \ttt}((2\ttt+r+1)^{-q})$ for all $q > 0$.

\item[(1.4)] If $\tilde \act $ is a polynomial, refer to the polynomial case below.
\end{enumerate}
\item[(NTK)] 
Define the simplified activation $\tilde\act$ by
\begin{IEEEeqnarray*}{+rCl+x*}
	\tilde\act & \equalDef & \begin{cases}
		\act^{[r]} &,\text{ if $\sigma_b^2 = 0$ and ($L=2$ or $\act$ is even or odd)} \\
		\act &,\text{ otherwise,}
	\end{cases} 
\end{IEEEeqnarray*}
and let $s 
\equalDef \smoothness{\tilde \act}$.
\begin{enumerate}
\item[(2.1)] If $1 \le  s < \infty$, then $\mu_{2\ttt+r, d}(\ntk_L) 
= \Theta_{\forall \ttt}\left({(2\ttt+r+1)^{-d-2s+1}}\right)$.
\item[(2.2)] If $s = \infty$ and $\tilde\act$ is not a polynomial, then $\mu_{2\ttt+r,d}(\ntk_L) > 0$ for all $\ttt$ and $\mu_{2\ttt+r,d}(\ntk_L) < o_{\forall \ttt}((2\ttt+r+1)^{-q})$ for all $q > 0$.
\item[(2.3)] If $\tilde \act$ is a polynomial, refer to the polynomial case below.
\end{enumerate}
\end{enumerate}
\textbf{Polynomial case:}
\begin{enumerate}
	\item[(3.1)]
	Let the activation $\tilde\act(t) = \sum_{i\ge 0} \lambda_i t^i$ be a nonzero polynomial.
	Define the even/odd degree of $\tilde\act$ as the degree of $\tilde\act(t) +\tilde \act(-t) $ or of $\tilde\act(t)- \tilde \act (-t)$. For the NNGP, define $\sigma\equalDef \sigma_b^2\sigma_i^2$ and for the NTK define $\sigma\equalDef \sigma_b^2$ respectively.
	
	Define $N_{\text{even}},N_{\text{odd}}$ according to the following table. By $\hdeven{\tilde\act},\hdodd{\tilde \act}$ we denote the even/odd degree, see \cref{def:degree}, of $\tilde \act$ displayed in the Hermite basis, which is possible as $\act$ fulfills \cref{ass:act}.
	\tiny{
	\begin{center}
	{\renewcommand{\arraystretch}{1.5}
		\begin{tabular}{ccc}
			\toprule 
			& $N_{\text{even}}$ & $N_{\text{odd}}$ \\
			\midrule
			$\sigma^2 >0$, $\hdeven {\tilde \act} > \hdodd {\tilde \act}$  & $\hdeven{\tilde \act}^{L-1}$ & $\hdeven{{\tilde \act}}^{L-1} -1$\\
			$\sigma^2 >0$, $\hdeven {\tilde \act} < \hdodd {{\tilde \act}}$ & $ \hdodd {{\tilde \act}}^{L-1} -1$ & $\hdodd {{\tilde \act}}^{L-1}$ \\
			$\sigma^2  =0, \hdeven {\tilde \act} > \hdodd {\tilde \act} $ & $\hdeven{\tilde \act}^{L-1}$ & $\hdeven{{\tilde \act}}^{L-1} -\hdeven {\tilde \act} +\hdodd {\tilde \act}$ \\
			$\sigma^2  =0, \hdeven {\tilde \act} < \hdodd {\tilde \act} $ & $\hdodd{\tilde \act}^{L-1}-\hdodd {\tilde \act} +\hdeven {\tilde \act}$ & $\hdodd{{\tilde \act}}^{L-1}$ \\
			\bottomrule
		\end{tabular}
		}
	\end{center}
}
\normalsize
	Then, 
	\begin{align*}
		\mu_{2\ttt+r,d}(\nngp_L / \ntk_L) >0 \text{ if and only if } 
		\begin{cases}
			2\ttt+r\le N_{\text{even}}&, r \text{ even,}\\
			2\ttt+r\le N_{\text{odd}}&, r \text{ odd.}
		\end{cases}
	\end{align*}
	\item[(3.2)]
	If $\tilde\act=0$, then $\mu_{2\ttt+r,d}(\nngp_L) >0$ if and only if $2\ttt+r =0$ and $\sigma_b^2 \sigma_i^2 >0$, and $\mu_{2\ttt+r,d}(\ntk_L) >0$ if and only if $2\ttt+r =0$ and $\sigma_b^2>0$.
\end{enumerate}
\end{theorem}
The definition of $\tilde \act$ may look complicated.
However, in contrast to $\act$, it prevents that $c_{1,1} = -(-1)^r c_{1,-1}$ in \cref{thm:bietti_bach_adapted}, which then would not provide an exact decay rate.

\begin{proof}
	For a bounded radial kernel $k(x,y)\equalDef \kappa(\scal xy)$ on $\bbS^d$ we define $k_{\text{even}}(x,y) \equalDef \kappa_{\text{even}}(\scal xy),k_{\text{odd}}(x,y) \equalDef \kappa_{\text{odd}}(\scal xy)$. 
	The spherical harmonics of degree $l$ are even if $l$ is even and odd if $l$ is odd.
	Hence, we obtain
	\begin{align*}
		\mu_{2\ttt,d} ( k) &= \mu_{2\ttt,d} ( k_\text{even}),\\
		\mu_{2\ttt+1,d} ( k) &= \mu_{2\ttt+1,d} ( k_\text{odd})
	\end{align*}
	for all $\ttt\in \bbN_0$. For the NNGP- and NTK-kernel, \cref{prop:even_odd_kernels} shows that in the cases 
	where $\tilde \act \ne \act $ holds we have 
	\begin{alignat*}{2}
		\knngp_{L, \acteven} &= \bra{\knngp_{L,\act}}_{\text{even}},\\
		\kntk_{L, \acteven} &= \bra{\kntk_{L,\act}}_{\text{even}},\\
		\knngp_{L, \actodd} &= \bra{\knngp_{L,\act}}_{\text{odd}},\\
		\kntk_{L, \actodd} &= \bra{\kntk_{L,\act}}_{\text{odd}},
	\end{alignat*}
	hence we conclude
	\begin{alignat*}{1}
		\mu_{2\ttt+r}\bra{\nngp_{L, \tilde \act}} &= \mu_{2\ttt+r}\bra{\nngp_{L, \act}},\\
		\mu_{2\ttt+r}\bra{\ntk_{L, \tilde \act}} &= \mu_{2\ttt+r}\bra{\ntk_{L, \act}},
	\end{alignat*}
	and it suffices to investigate the kernels induced by $\tilde\act$.
	In the following, $g_l, G_l$ and $H_l$ from \cref{eq:recursive_formulas} refer to the functions induced by the activation function  $\tilde \act$.

\hspace{-7mm}\textbf{NNGP:}
\begin{enumerate}[leftmargin=*]
	\item[(1.1)]  
	\cref{lemma:asymptotics_at_one} yields
	\begin{align*}
		G_L(1-t) = C_{L,+} t^{2^{1-L}} +\calQ_{-1,2^{1-L}}(t)
	\end{align*}
	for $t\searrow	0$, where $\abs{C_{L,+}}>0$.
	By \cref{thm:bietti_bach_adapted}, it suffices to show
	\begin{align}
	\label{eq:nngp_case_1_1_goal}	
	G_L(-(1-t)) &= C_{L,-} t^{2^{1-L}} +\calQ_{-1,2^{1-L}}(t),\\
	\nonumber	C_{L,-} &\ne \begin{cases}
			-C_{L,+}& \text{, if }r=0\\
			C_{L+}& \text{, if }r=1~. 
		\end{cases}
	\end{align}
	\begin{enumerate}[(a)]
		\item Suppose $\sigma_i^2\sigma_b^2>0$. By \cref{lemma:layer_nngp_contraction_property} we have $g_1([-1,1))\subset (-1,1)$  and $G_{l}([-1,1))\subset (-1,1)$ for $l\ge 2$.
		As $g_1$ is a polynomial, $g_1(-(1-t))= \calQ_{-1, 2^{1-L}}(t)$ holds. By \cref{lemma:dual_properties} (d) we have $g_l\vert_{(-1,1)} \in C^\infty ((-1,1))$ for all $l\ge 2$, which allows leveraging \cref{prop:rules_for_Q} (d) to obtain
		$G_l(-(1-t))= g_l(G_{l-1}(-(1-t))) =\calQ_{-1, 2^{1-L}}(t)$ and \cref{eq:nngp_case_1_1_goal} follows.
		\item Suppose $\sigma_b^2\sigma_i^2=0$, $\tilde \act $ be neither even nor odd  and $L\ge 3$. 
		Then, $g_1(t)= t$ and $G_2(t)= \w{\act}(t)/\w{\act}(1)$ hold
		and from \cref{thm:smoothed_dual_activation_smooth_decomposition} we obtain 
		\begin{align*}
			G_2(-(1-t))&= C_{2,-} t^{1/2} + \calQ_{-1, 1/2}(t).
		\end{align*}
		\Cref{lemma:layer_nngp_contraction_property} yields $G_{l}([-1,1))\subset (-1,1)$ for $l\ge 2$ and we furthermore have $g_l\vert_{(-1,1)} \in C^\infty (-1,1)$ by \cref{lemma:dual_properties} (d).
		Recursively applying  \cref{prop:rules_for_Q} (d) now yields
		\begin{align*}
			G_L(-(1-t)) = g_L\left(G_{L-1}(-(1-t))\right) = C t^{1/2}+ \calQ_{-1,1/2}(t)= \calQ_{-1,2^{1-L}}(t)
		\end{align*}
		and we have \cref{eq:nngp_case_1_1_goal}.
		\item 
		Let $\sigma_b^2\sigma_i^2=0$ and let  $\tilde\act$ be odd. Then, $r=1$ follows by construction.
		\Cref{prop:even_odd_kernels}  shows that $G_l$ is odd for all $l\in \N$ and we obtain
		\begin{IEEEeqnarray*}{+rCl+x*}
			G_L(-(1-t)) & = & -G_L(1-t) = -\bra{ C_{L,+} t^{2^{1-L}} +\calQ_{-1,2^{1-L}}(t)}
		\end{IEEEeqnarray*}
		follows. As desired, we have $C_{L,+}\ne C_{L,-} = -C_{L,+} $.
		\item The case $\sigma_b^2 \sigma_i^2 = 0$ and $\tilde\act$ even is not possible since even activation functions $\tilde \act$ cannot have smoothness $s=0$.
	\end{enumerate}	
	\item[(1.2)]
	Define $\beta\equalDef s+1/2$. \cref{lemma:asymptotics_at_one} yields 
	\begin{align*}
		G_L(1-t) = C_{L,+} t^{\beta} +\calQ_{-1,\beta}(t)
	\end{align*}
	for $t\searrow	0$, where $\abs{C_{L,+}}>0$. 
	By \cref{thm:bietti_bach_adapted} it suffices to show
	\begin{align}
		\label{eq:nngp_case_1_2_goal}	
		G_L(-(1-t)) &= C_{L,-} t^{\beta} +\calQ_{-1,\beta}(t)~,\\
		\nonumber	C_{L,-} &\ne \begin{cases}
			-C_{L,+}& \text{, if }r=0\\
			C_{L+}& \text{, if }r=1~. 
		\end{cases}
	\end{align}
	\begin{enumerate}[(a)]
		\item Let $\sigma_b^2\sigma_i^2>0$. By \cref{lemma:layer_nngp_contraction_property} we have $g_1([-1,1))\subset (-1,1)$  and $G_{i}([-1,1))\subset (-1,1)$ for $i\ge 2$.
		As $g_1$ is a polynomial, $g_1(-(1-t))= \calQ_{-1, \beta}(t)$ holds. \cref{lemma:dual_properties} (d) shows $g_i\vert_{(-1,1)} \in C^\infty ((-1,1))$ for all $i\ge 2$ which allows leveraging \cref{prop:rules_for_Q} (d) to obtain
		$G_L(-(1-t))= g_L(G_{L-1}(-(1-t))) =\calQ_{-1, \beta}(t)$
		by induction. 
		\cref{eq:nngp_case_1_2_goal} follows.
		\item Let $\sigma_b^2\sigma_i^2=0$, $\tilde \act $ be neither even nor odd  and $L\ge 3$, that is, $\tilde \act =\act$. Then, $g_1(t)= t$ and 
		$G_2(t)= \w{\act}(t)/\w{\act}(1)$ hold and 
		from \cref{thm:smoothed_dual_activation_smooth_decomposition} we obtain
		\begin{align*}
			G_2(1-t) &= C_{2,+} t^{\beta} + \calQ_{-1, \beta}(t),\\
			G_2(-(1-t))&= C_{2,-} t^{\beta} + \calQ_{-1, \beta}(t),
		\end{align*}
		where $\abs{C_{2,+}} = \abs{C_{2,-}} \ne 0$. Recursively applying \cref{lemma:asymptotics_at_one} furthermore yields
		\begin{align*}
			\abs{C_{L,+}}\ge g_L'(1)\cdot \ldots \cdot g_3'(1) \cdot\abs {C_{2,+}}.
		\end{align*}
		
		\Cref{lemma:layer_nngp_contraction_property} yields $G_{l}([-1,1))\subset (-1,1)$ for $l\ge 2$ and we have $g_l\vert_{(-1,1)} \in C^\infty (-1,1)$ by \cref{lemma:dual_properties}. Iteratively applying \cref{prop:rules_for_Q} (d) yields
		\begin{align*}
			G_L(-(1-t)) = g'_L (G_{L-1}(-1)) \cdot \ldots  \cdot g'_3 (G_2(-1)) \cdot C_{2,-} t^{\beta} +\calQ_{-1, \beta}(t).
		\end{align*}
		Note that we have $g_l'(t) = \w{\scaledfn{\act'}{\sqrt{\alpha_{l-1}}}(t)}/ \alpha_l$ and by \cref{lemma:dual_properties} (c) we have
		\begin{align*}
			\abs {g_l'(t) } < \abs{g_l'(1)}
		\end{align*}
		for all $t\in (-1,1)$.
		Altogether we conclude
		\begin{align*}
			G_L(1-t) &= C_{L,+} t^{\beta} +\calQ_{-1, \beta}(t),\\
			G_L(-(1-t)) &= C_{L,-} t^{\beta}+\calQ_{-1, 	\beta}(t),
		\end{align*}
		where $\abs{C_{L,+}} > \abs{C_{L,-}}$ holds and we obtain \cref{eq:nngp_case_1_2_goal}.
		\item 
		Let $\sigma_b^2\sigma_i^2=0$ and let  $\tilde\act$ be odd, so we have $r=1$.
		By \cref{prop:even_odd_kernels}, $G_l$ is odd for all $l\in \N$ and
		\begin{IEEEeqnarray*}{+rCl+x*}
			G_L(-(1-t)) & = & -G_L(1-t) = -\bra{ C_{L,+} t^{\beta} +\calQ_{-1,\beta}(t)}
		\end{IEEEeqnarray*}
		follows. We obtain \cref{eq:nngp_case_1_2_goal} as $C_{L,+}\ne C_{L,-} = -C_{L,+} $ holds.
		\item In the case $\sigma_b^2 \sigma_i^2 = 0$ and $\tilde\act$ even \cref{prop:even_odd_kernels} yields $G_l$ even and we argue as in (c).
	\end{enumerate}	
		
		\item[(1.3)]		
		For $q>0$ choose $\tilde q>0, \tilde q\not \in \Z$ such that $ -d-2\tilde q\le -q $. 
		By \cref{lemma:asymptotics_at_one} we have for $\tau \in \{1,-1\}$:
		\begin{align*}
			G_L(\tau(1-t)) &= \calQ_{-1, \tilde q}(t)
		\end{align*}
		and \cref{thm:smoothed_dual_activation_smooth_decomposition} yields
		$\mu_{2\ttt+r, d}(\nngp_L) = o_{\forall \ttt}((2\ttt+r+1)^{-d-2\tilde q})=o_{\forall \ttt}((2\ttt+r+1)^{-q})$.
		
		We proceed to show that all Eigenvalues $\mu_{2\ttt+1,d}(\nngp_L)$ are strictly positive.
		As $\tilde \act$ is not a polynomial, its Hermite series representation has infinitely many nonzero coefficients; that is, we have $\hdegree{\tilde \act}=\infty$. See \cref{sec:degree} for more details. 
		We apply \cref{lemma:nngp_ntk_even_odd_degree_polynomial_case}.
		
		Assume even parity $r=0$. In the case $\deven{\tilde \act} =\infty$, the claim follows directly. In the case $\deven{\tilde \act} < \dodd{\tilde \act} =\infty$, we note that $\tilde\act = \act$ follows and investigate sub-cases.
		If $\sigma_b^2\sigma_i^2>0$, then \cref{lemma:nngp_ntk_even_odd_degree_polynomial_case} directly yields $\dseven{G_l}=\infty$ as desired.
		If $\sigma_b^2\sigma_i^2=0$ holds, then the definition of $\tilde \act$ implies that $\deven{\tilde\act }\ge 0$ and $L\ge 3$ hold, and again the claim follows directly from \cref{lemma:nngp_ntk_even_odd_degree_polynomial_case}. 
		The case of odd parity $r=1$ is proven analogously.

	\end{enumerate}
	\hspace{-7mm}\textbf{NTK:} 

	The argumentation for the NTK-kernel is more technical as for the NNGP-kernel, as by \cref{eq:recursive_formulas} $H_l$ follows the recursion
	\begin{align}
	\nonumber	H_1(t)&= \sigma_b^2 +\sigma_w^2 t ~,\\
	\label{eq:ntk_recursion}
		H_{l}(t)&= \sigma_b^2(1-\sigma_i^2) + \alpha_{l} G_{l}(t) + \sigma_w^2 \w{\scaledfn{\act'}{\sqrt{\alpha_{l-1}}}}\bra{G_{l-1}(t)} H_{l-1}(t)~, \quad l\ge 2~,
	\end{align}
	which involves more terms in a complicated fashion.

	\begin{enumerate}
		\item[(2.1)]
		Define $\gamma\equalDef s-1/2$. \cref{lemma:asymptotics_at_one} yields 
		\begin{align*}
			H_L(1-t) = C_{L,+} t^{\gamma} +\calQ_{-1,\gamma}(t)
		\end{align*}
		for $t\searrow	0$, where $\abs{C_{L,+}}>0$. 
		By \cref{thm:bietti_bach_adapted} we need to show
		\begin{align}
			\label{eq:ntk_case_2_1_goal}	
			H_L(-(1-t)) &= C_{L,-} t^{\gamma} +\calQ_{-1,\gamma}(t)~,\\
			\nonumber	C_{L,-} &\ne \begin{cases}
				-C_{L,+}& \text{, if }r=0\\
				C_{L+}& \text{, if }r=1~. 
			\end{cases}
		\end{align}
		In case (1.2) we saw for the NNGP-kernel that 
		\begin{align}
			\label{eq:G_l_asymptotics}
			G_l(\tau (1-t)) = c_{l,\tau} t^{\gamma+1} +\calQ_{-1,\gamma+1}(t) = \calQ_{-1,\gamma}(t) 
		\end{align}
		holds for $\tau=\pm 1,\ l\ge 1$. Hence, it suffices to show 
		\begin{align}
			\label{eq:ntk_case_2_1_simpler}
			\sigma_w^2 \w{\scaledfn{\act'}{\sqrt{\alpha_{L-1}}}}\bra{G_{L-1}(-(1-t))} H_{L-1}(-(1-t)) &= C_{L,-} t^{\gamma} +\calQ_{-1,\gamma}(t)~,\\
			\nonumber	C_{L,-} &\ne \begin{cases}
				-C_{L,+}& \text{, if }r=0\\
				C_{L+}& \text{, if }r=1~. 
			\end{cases}
		\end{align}
		in order to obtain \cref{eq:ntk_case_2_1_goal}.
		\begin{enumerate}[(a)]

\item Let $\sigma_b^2>0$.
This is the most difficult sub-case. Without loss of generality assume $\sigma_w^2 = 1$.

When analyzing $H_l(\tau(1-t))$, \cref{eq:ntk_recursion}  the constants $C_{l,\tau}$ of $t^\gamma$ stem from the summand
\begin{align}
	\label{eq:ntk_simplified_formula}	
	\w{\scaledfn{\act'}{\sqrt{\alpha_{l-1}}}}\bra{G_{l-1}(\tau(1-t))} H_{l-1}(\tau(1-t)) &= C_{l,\tau} t^{\gamma} +\calQ_{-1,\gamma}(t)
\end{align}
as the other summands in \cref{eq:ntk_recursion} are in $\calQ_{-1,\gamma}(t)$ by \cref{eq:G_l_asymptotics}.

We show by induction in $l \ge 2$ that the following statements hold:
\begin{itemize}
	\item $\sign{C_{l,+}} = (-1)^{s+1}$,
	\item $\abs{C_{l,+}} > \abs{C_{l,-}}$.
\end{itemize}
The first point is a technical tool for the induction, ensuring that the coefficients of $t^\gamma$ for $\tau=1$ suffer no annihilation because of different signs of its summands: The coefficients of $t^\gamma$ are the sum of the product of the $t^\gamma$-coefficient of one factor with the constant coefficient of the other factor in \cref{eq:ntk_simplified_formula}. 
The second point then proves this sub-case of the theorem.

By \cref{thm:smoothed_dual_activation_smooth_decomposition} we have
\begin{align}
	\label{eq:scaledfn}
	\widehat{\scaledfn{\act'}{\sqrt{\alpha_l}}}(\tau(1-t)) = (-\tau )^{s+1} d_{l} t^\gamma + \calQ_{-1,\gamma}(t)~,
\end{align}
where $d_l > 0$, and furthermore $\widehat{\scaledfn{\act'}{\sqrt{\alpha_l}}}(\cdot)$ is smooth on $(-1,1)$ by \cref{lemma:dual_properties} (d).

For the induction base $l=2$ 
we have $G_1(\tau(1-t)) = \tau(1-t), H_1(\tau (1-t)) = \sigma_b^2 + \tau -\tau t$ which yields
\begin{align*}
	\widehat{\scaledfn{\act'}{\sqrt{\alpha_{1}}} } (G_1(\tau(1-t))) H_1(\tau(1-t))
	&= ((-\tau)^{s+1} d_2 t^\gamma + \calQ_{-1,\gamma} (t)) (\sigma_b^2+\tau -\tau t)\\
	&= (-\tau)^{s+1} d_2 (\sigma_b^2+\tau) t^\gamma + \calQ_{-1,\gamma}(t)~, 
\end{align*}
i.e.\ $C_{2,+} = (-1)^{s+1} d_2 (\sigma_b^2+1)$ and $C_{2,-}= d_2(\sigma_b^2 -1)$
and the claim holds for $l=2$ since $\sigma_b^2>0$.

For the induction step, let the claim hold for $l\ge 2$.
First we 
show the auxiliary statement 
\begin{align}
	\label{eq:ntk_maximum}
	H_r(1) > \abs{ H_r(t)}\quad , t\in [-1,1),\ r\in\N~,
\end{align} 
which is in the spirit of \cref{lemma:layer_nngp_contraction_property}, by induction in $r\ge 1$. In the base case $r=1$, we have $H_1(t) = \sigma_b^2 +t$ and the claim holds.
In the induction step, let the claim hold for $r\ge 1$. We use \cref{eq:recursive_formulas} and obtain
\begin{align*}
	H_{r+1}(t) = \sigma_b^2 +   
	\w{\scaledfn{\act}{\sqrt{\alpha_{r}}}}(G_{r}(t))
	+ \w{\scaledfn{\act'}{\sqrt{\alpha_{r}}}}(G_{r}(t)) H_{r}(t)~.
\end{align*}
Now, for any $t\in [-1,1]$ we have $G_{r}(1)=1, G_{r}(t) \in [-1,1]$ by \cref{lemma:layer_nngp_contraction_property} and by \cref{lemma:dual_properties} we  obtain 
\begin{align}
\nonumber	\w{\scaledfn{\act}{\sqrt{\alpha_{r}}}}(1) &\ge \abs{ \w{\scaledfn{\act}{\sqrt{\alpha_{r}}}}(t)  }\\
\label{eq:derived_dualactivation_max}	\w{\scaledfn{\act'}{\sqrt{\alpha_{r}}}}(1) &\ge \abs{ \w{\scaledfn{\act'}{\sqrt{\alpha_{r}}}}(t)  }
\end{align}
and \cref{eq:ntk_maximum} follows.

By a Taylor expansion we have
\begin{align}
	\label{eq:G_l_taylor_expansion}
	G_{l}(\tau(1-t)) &= G_{l} (\tau ) -\tau t G_{l}'(\tau) +\calQ_{1,\gamma}(t)
\end{align}
as $G_l(\tau(1-t)) = \calQ_{-1,\gamma}(t)$ by \cref{eq:G_l_asymptotics}
and furthermore $G_{l}(t) \in [-1,1], G_{l}(1)=1$ by \cref{lemma:layer_nngp_contraction_property}.

For $\tau =1$ we have by \cref{eq:G_l_taylor_expansion} and \cref{eq:scaledfn}
\begin{align*}
	\w{\scaledfn{\act'}{\sqrt{\alpha_{l}}}}(G_{l}(1-t)) &= 
	(-1)^{s+1} d_l G_{l}'(1) t^\gamma +\calQ_{-1,\gamma}(t)~.
\end{align*}

\begin{align*}
	\w{\scaledfn{\act'}{\sqrt{\alpha_{l}}}}(G_{l}(1-t)) H_{l}(1-t) &=
	\underbrace{
		\big((-1)^{s+1} d_l G_{l}'(1) H_{l}(1) 
		+ \w{\scaledfn{\act'}{\sqrt{\alpha_{l}}}}(G_{l}(1)) C_{l,+}  \big)
	}_{=C_{l+1,+}}t^\gamma +\calQ_{-1,\gamma}(t)
\end{align*}
where furthermore 
\begin{align}
	\label{eq:bounds_on_G_l_derivatives}
	G_{l}'(1)\ge \abs{G_l'(t)}\, \quad G_{l}'(1)>0
\end{align} 
hold for all $t\in [-1,1]$ by \cref{lemma:dual_properties} and \cref{eq:recursive_formulas}.
We observe that $\sgn C_{l,+} = (-1)^{s+1}$ holds as claimed, that is, $C_{l+1,+}$ suffers no annihilation from different signs of its summands.

For $\tau =-1$ the possible cases are $G_{l}(-1) \in (-1,1)$, for which we obtain 
$\w{\scaledfn{\act'}{\sqrt{\alpha_{l}}}}(G_{l}(-(1-t))) = \calQ_{-1,\gamma}(t)$ by \cref{prop:rules_for_Q} (d) and \cref{eq:G_l_asymptotics},
and $\abs{G_{l}(-1)} =1$, for which we obtain
\begin{align*}
	\w{\scaledfn{\act'}{\sqrt{\alpha_{l}}}}(G_{l}(-(1-t))) = \tilde a_{l,\tau} d_l G_{l}'(-1)  t^\gamma +\calQ_{-1,\gamma}(t)~,
\end{align*}
$\tilde a_{l,\tau} = \pm 1$ and $\abs{G_{l}'(-1)} \le G_{l}'(1)$ by \cref{eq:bounds_on_G_l_derivatives}.
Hence, we have 
\begin{align}
\nonumber
	\abs{\w{\scaledfn{\act'}{\sqrt{\alpha_{l}}}}(G_{l}(-(1-t)))} &= a_{l,-}  t^\gamma +\calQ_{-1,\gamma}(t)\\
	\label{eq:scaled_derived_activation_at_minus}	\abs{a_{l,-}} &\le \abs{d_l G_l'(1)} 
\end{align}
This yields
\begin{align*}
	\w{\scaledfn{\act'}{\sqrt{\alpha_{l}}}}(G_{l}(-(1-t))) H_{l}(-(1-t)) &=
	\underbrace{ \big( a_{l,-}  H_{l}(-1) 
		+ \w{\scaledfn{\act'}{\sqrt{\alpha_{l}}}}(G_{l}(-1)) C_{l,-}  \big)
	}_{=C_{l+1,-}}t^\gamma +\calQ_{-1,\gamma}(t)
\end{align*}
and comparing the coefficients $C_{l+1,+}$ and $\ C_{l+1,-}$ 
we finally obtain
\begin{align*}
	\abs{C_{l+1,+}} &= \abs{ d_l G_{l}'(1)H_l(1)  } + \abs{\w{\scaledfn{\act'}{\sqrt{\alpha_{l}}}}(G_{l}(1)) C_{l,+}}\\
	&> \abs{ a_{l,-} H_l(-1)  } + \abs{\w{\scaledfn{\act'}{\sqrt{\alpha_{l}}}}(G_{l}(-1)) C_{l,-}}\\
	&\ge \abs{C_{l+1,-}}
\end{align*}
using \cref{eq:ntk_maximum} and \cref{eq:scaled_derived_activation_at_minus} to compare the first summand and 
the induction hypothesis and \cref{eq:derived_dualactivation_max} for the second summand. 

			\item 
			Suppose $\sigma_b^2=0$, $\tilde \act $ neither even nor odd  and $L\ge 3$, that is $\tilde \act=\act$.
			
			Without loss of generality assume $\sigma_w^2=1$ for simpler notation.
			We show  $\abs{C_{2,+} }\ge \abs{C_{2,-}} $, for $l\ge 3 $ we show that $\abs{C_{l-1,+} }\ge \abs{C_{l-1,-}}$ implies  $\abs{C_{l,+}}> \abs{C_{l,-}}$. That yields \cref{eq:ntk_case_2_1_goal}. Furthermore we show that $\sign(C_{l,+})= (-1)^s$ holds, which is required for the corresponding recursion. 
			
			As $G_1(t) = t$, we have for $\tau=\pm 1$ by \cref{thm:smoothed_dual_activation_smooth_decomposition}
			\begin{align*}
				\w{\scaledfn{\act'}{\sqrt{\alpha_{l-1}}}}\bra{G_{1}(\tau(1-t))} = (-\tau)^s C_2   t^{\gamma } + \calQ_{-1,\gamma}(t)~, \quad t\searrow 0~,
			\end{align*}
			where $C_2>0$.
			Now, since we have $H_1(\tau(1-t)) =\tau (1-t)$ we obtain from \cref{eq:ntk_recursion}
			\begin{align*}
			H_2(\tau(1-t)) &=  \tau(-\tau)^s C_2 t^{\gamma}  + \calQ_{-1,\gamma}(t)
			\end{align*}
			and the claim holds for $l=2$. 			
			Now, we take a close look at $C_{l,+}$ for $l\ge 3$. Recalling  $G_l(1-t) = G_{l-1}(0) - t G_{l-1}'(1)+\calQ_{1,\gamma}(t)$ from the considerations done for the NNGP above, we see that for
			\begin{align*}
				H_{l}(1-t)&= \sigma_b^2(1-\sigma_i^2) + \alpha_{l} G_{l}(t) +  \w{\scaledfn{\act'}{\sqrt{\alpha_{l-1}}}}\bra{G_{l-1}(1-t)}  H_{l-1}(1-t)
			\end{align*}
			the terms contributing to the coefficient of $t^\gamma$ only stem from the product of $\w{\scaledfn{\act'}{\sqrt{\alpha_{l-1}}}}\bra{G_{l-1}(1-t)}$ and $  H_{l-1}(1-t)$ and by \cref{thm:smoothed_dual_activation_smooth_decomposition} we have
			\begin{align*}
				\w{\scaledfn{\act'}{\sqrt{\alpha_{l-1}}}}\bra{G_{l-1}(1-t)} = (-1)^{s} G_{l-1}'(1) c_{l} t^\gamma+ \w{\scaledfn{\act'}{\sqrt{\alpha_{l-1}}}}(1) + \calQ_{0,\gamma}(t)
			\end{align*}
			for a constant $c_l>0$. As $G_{l-1}'(1)>0$ by \cref{lemma:layer_nngp_contraction_property}, the sign of the coefficient is $(-1)^s$.
			Together with 
			\begin{align*}
				H_{l-1}(1-t) = H_{l-1}(1) + C_{l-1,+} t^\gamma +\calQ_{0,\gamma}(t)
			\end{align*}
			we obtain
			\begin{align*}
				&\quad \w{\scaledfn{\act'}{\sqrt{\alpha_{l-1}}}}\bra{G_{l-1}(1-t)}  H_{l-1}(1-t) 
				\\
				&= \left( (-1)^{s} G_{l-1}'(1) c_{l} t^\gamma+ \w{\scaledfn{\act'}{\sqrt{\alpha_{l-1}}}}(1) + \calQ_{0,\gamma}(t)  \right) \cdot \left( H_{l-1}(1) + C_{l-1,+} t^{\gamma} +\calQ_{0,\gamma}(t)\right)\\
				&=  \underbrace{\left(  (-1)^{s} G_{l-1}'(1) c_{l} H_{l-1}(1) + \w{\scaledfn{\act'}{\sqrt{\alpha_{l-1}}}}(1) C_{l-1,+} \right)}_{=C_{l,+}}t^\gamma + \calQ_{-1,\gamma}(t)~.
			\end{align*}
			We observe that as $G_{l-1}'(1), c_{l}, H_{l-1}(1),\w{\scaledfn{\act'}{\sqrt{\alpha_{l-1}}}}(1)>0 $ holds we indeed have $\sign(C_{l,+})= (-1)^s$ and observe
			\begin{align}
				\label{eq:size_of_Cplus}
				\abs{C_{l,+}} >\w{\scaledfn{\act'}{\sqrt{\alpha_{l-1}}}}(1) \abs{C_{l-1,+}}~.
			\end{align}
			Investigating $C_{l,-}$ we note that as $G_{l-1}([-1,1)) \subseteq (-1, 1)$ holds by \cref{lemma:layer_nngp_contraction_property}, \cref{lemma:dual_properties} (d) yields $ \w{\scaledfn{\act'}{\sqrt{\alpha_{l-1}}}}|_{(-1, 1)}\in C^\infty((-1, 1))$ and by \cref{prop:rules_for_Q} (e) we have
			\begin{align*}
				\w{\scaledfn{\act'}{\sqrt{\alpha_{l-1}}}}\bra{G_{l-1}(-(1-t))} =  \w{\scaledfn{\act'}{\sqrt{\alpha_{l-1}}}}(G_{l-1}(-1)) + \calQ_{0,\gamma}(t)~,
			\end{align*}	
			since $G_{l-1}(-(1-t)) = \calQ_{-1,\gamma}(t)$ holds. As before, coefficients of $t^\gamma$ can only stem from the product of $\w{\scaledfn{\act'}{\sqrt{\alpha_{l-1}}}}(G_{l-1}(-(1-t)))$ and $H_{l-1}(-(1-t))$ and where we have
			\begin{align*}
					&\quad \w{\scaledfn{\act'}{\sqrt{\alpha_{l-1}}}}\bra{G_{l-1}(-(1-t))}  H_{l-1}(-(1-t)) 
				\\
				&= \left(  \w{\scaledfn{\act'}{\sqrt{\alpha_{l-1}}}}(G_{l-1}(-1)) + \calQ_{0,\gamma}(t)  \right) \cdot 
				\left( H_{l-1}(1) + C_{l-1,-} t^{\gamma} +\calQ_{0,\gamma}(t)\right)\\
				&=  \underbrace{ \w{\scaledfn{\act'}{\sqrt{\alpha_{l-1}}}}(G_{l-1}(-1)) C_{l-1,-} }_{=C_{l,-}} t^\gamma + \calQ_{-1,\gamma}(t)~.
			\end{align*}
			\cref{lemma:dual_properties} (c) yields 
			\begin{align*}
				\w{\scaledfn{\act'}{\sqrt{\alpha_{l-1}}}}(G_{l-1}(-1)) < \w{\scaledfn{\act'}{\sqrt{\alpha_{l-1}}}}(G_{l-1}(1))
			\end{align*}	
			as $G_{l-1}(-1) \in (-1,1)$ and we obtain from \cref{eq:size_of_Cplus} as desired
			\begin{align*}
				\abs{ C_{l,-} } &= \abs{\w{\scaledfn{\act'}{\sqrt{\alpha_{l-1}}}}(G_{l-1}(-1)) C_{l-1,-}} 
				< \abs{\w{\scaledfn{\act'}{\sqrt{\alpha_{l-1}}}}(G_{l-1}(1)) C_{l-1,+}}
				\\
				&< \abs{C_{l,+}}~.
			\end{align*}		
			\item 
			Let $\sigma_b^2=0$ and let  $\tilde\act$ be odd, so we have $r=1$.
			By \cref{prop:even_odd_kernels} $H_l$ is odd for all $l\in \N$ and
			\begin{IEEEeqnarray*}{+rCl+x*}
				H_L(-(1-t)) & = & -H_L(1-t) = -\bra{ C_{L,+} t^{\gamma} +\calQ_{-1,\gamma}(t)}
			\end{IEEEeqnarray*}
			follows. We obtain \cref{eq:ntk_case_2_1_goal} as $C_{L,+}\ne C_{L,-} = -C_{L,+} $ holds.
			\item In the case $\sigma_b^2 = 0$ and $\tilde\act$ even \cref{prop:even_odd_kernels} yields $H_L$ even and we argue as above.
		\end{enumerate}	
		\item[(2.2), (2.3)] The arguments work just as in the NNGP cases (1.3) or respectively (1.4).
		
	\end{enumerate}

	\textbf{Polynomial case:}
	
	\begin{enumerate}[leftmargin=*]
		\item[(3.1)] 
		This is the statement of \cref{lemma:nngp_ntk_even_odd_degree_polynomial_case}.
		\item[(3.2)] Elementary. \qedhere
	\end{enumerate}
\end{proof}

Finally, we are able to prove \cref{thm:main_result}.
\begin{proof}[Proof of \cref{thm:main_result}]\label{proof:main_result_proof}
	\cref{lemma:limit_ntk_formula} yields the convergence of the NNGP- and NTK-kernels. 
	\cref{thm:kernel_ev_rates} combined with \cref{lemma:sobolev_spherical} and a case distinction then yields the claim. The sub-cases i) and ii) are straightforward, in the sub-cases iii)
	we use \cref{thm:kernel_ev_rates} with the parity $r=0$ to obtain the eigenvalues of even index and \cref{thm:kernel_ev_rates} with the parity $r=1$ to obtain the eigenvalues of odd index, and then combine these results to obtain Eqs.\ \eqref{eq:nngp_complicated} and \eqref{eq:ntk_complicated}, where we use that the spherical harmonics of even/odd degree are even/odd.
\end{proof}

\subsection{A close look at the even/odd degree of $G_l$ and $H_l$}\label{sec:degree}

In this subsection we develop the methods required to deal with smooth activations in \cref{thm:main_result}, where the only difficulty not covered by \cref{thm:bietti_bach_adapted} is investigating which eigenvalues of the NNGP-kernel and NTK-kernel are strictly positive.

\begin{definition}\label{def:degree}
	Let $p(t)=\sum_{n\ge 0} \lambda_n t^n, q(t)= \sum_{n\ge 0} \eta_n t^n $ be \emph{power series} which converge on an interval $I\subset \R$.
	We define the degree of $p$ as $\degree{p} \equalDef   \sup \left(\{ n\in\N_0 \mid \lambda_n\ne 0 \} \cup\{-\infty\}\right) \in \N_0 \cup \{-\infty,\infty\}$ and the even/odd degree of $p$ as
	\begin{align*}
		\deven{p} & \equalDef \degree{t \mapsto p(t) + p(-t)}~,\\
		\dodd{p} & \equalDef \degree{t \mapsto p(t) - p(-t)}
	\end{align*}
	and say
	\begin{align*}
		p\dequal q
	\end{align*}
	if and only if $\deven{p} =\deven q$ and $ \dodd p = \dodd q$. 
	That is to say, we compare only the highest even/odd degrees for this notion of equality.
	
	Let $f\in \calL_2(\calN(0,1))$ be a function. Then $f$ equals its convergent \emph{Hermite} representation $ t\mapsto \sum_{n\ge 0} a_n(f) h_n(t)$ almost everywhere, and we define $\hdegree f, \hdeven f, \hdodd f $ and $f\hdequal g$ based on the Hermite coefficients $(a_n)_{n\ge 0}$
	 analogously to the above definitions.
	
	Furthermore we also use the notation $f\phdequal g$ to denote that the even/odd degree of $f$ and $g$ coincide when $f$ is  a power series and  and $g$ is displayed in the Hermite basis. %
\end{definition}
Note that $\deven p= -\infty $ if $p$ is odd and respectively $\dodd p= -\infty $ if $p$ is even, resembling the common convention for the degree of polynomials.

\begin{lemma}[Odd/even degree of power series with non-negative coefficients]\label{lemma:odd/even_pols}
	Let $f,g,\tilde f,\tilde g$ be power series  with non-negative coefficients converging on an interval $I$, such that the products and compositions below are absolutely convergent on $\bar I$. 
	We use the conventions $\infty -\infty \equalDef -\infty, 0 \cdot \infty \equalDef 0, a\vee b \equalDef \max\{a,b\}$.
	\begin{enumerate}
		\item 
		\begin{enumerate}
			\item The sum $g+f$ fulfills
			\begin{align*}
				\dseven{g+f} &= \deven g \vee \deven f,\\
				\dsodd{g+f} &= \dodd{g} \vee \dodd f~.
			\end{align*}
			\item The product $g\cdot f $ fulfills
			\begin{align*}
				\dseven{g\cdot f} &= \left(\deven{g} + \deven f\right) \vee \left(\dodd g +\dodd f\right),\\
				\dsodd{g\cdot f} &= \left(\deven {g} + \dodd f \right)\vee \left( \dodd g +\deven f\right)~.
			\end{align*}
			\item 
			Let $f\ne 0$. The composition $g\circ f$ fulfills
			\begin{IEEEeqnarray*}{+rCl+x*}
				\dseven{g\circ f} & = & \left(\deven g \deven f \right)\vee \left( \deven g \dodd f \right) \\
				&& ~\vee~ \left( \dodd g \deven f \right)\vee \left((\dodd g-1) \dodd f + \deven f\right)\\
				\dsodd{g\circ f} & = & \left(\dodd g \dodd f \right)\vee \left( (\dodd g -1) \deven f + \dodd f \right) \\
				&& ~\vee~ \left( (\deven g-1) \dodd f +\deven f \right)
				\\
				&& ~\vee~ \left((\deven g -1 ) \deven f + \dodd f\right)~.
			\end{IEEEeqnarray*}
		\end{enumerate}
		\item Let $\tilde f \dequal f, \tilde g\dequal g$. Then we have
		\begin{align*}
			g + f &\dequal \tilde g + \tilde f~,\\
			g \cdot f &\dequal \tilde g \cdot \tilde f~,\\
			g\circ f &\dequal \tilde g \circ \tilde f~.
		\end{align*}
	\end{enumerate}
\end{lemma}
\begin{proof}
			1.
			The requirement that $f,g$ have nonzero coefficients is crucial as it prevents the elimination of coefficients.

			We only prove the even-statement in iii)  as this is the most difficult statement. The other claims are similar but easier to prove.

			As a consequence of the absolute convergence we observe that $g\circ f= \lim_{n\to\infty}g_n\circ f_n$ and $\dseven{g\circ f}= \lim_{n\to\infty}\dseven{g_n\circ f_n}, \dsodd{g\circ f}= \lim_{n\to\infty}\dsodd {g_n\circ f_n}$ holds, where $f_n(t)=\sum_{i=0}^n \lambda_i t^i, g_n(t)= \sum_{j=0}^n \mu_jt^j$.
			Hence, we can assume without loss of generality  that $\deg f,\deg g<\infty$ holds and obtain
			\begin{align*}
				& g\circ f(t) = \sum_{i\ge 0}\mu_i \Big( \sum_{j\ge 0} \lambda_j t^j\Big)^i = \sum_{i\ge 0} \mu_i \Big( \sum_{j_1,\dots,j_i\ge 0} \lambda_{j_1} \cdot \hdots \cdot \lambda_{j_i} t^{j_1+\dots+j_i} \Big)\\
				=&\, \sum_{i\ge 0} \sum_{j_1,\dots,j_i\ge 0} \mu_i  \lambda_{j_1} \cdot \hdots \cdot \lambda_{j_i} t^{j_1+\dots+j_i}~.
			\end{align*}
			We note that all occurring coefficients are non-negative, as they are products of the non-negative coefficients $\lambda_j,\mu_i$. In the above expressions, there are multiple summands contributing to a power $t^r$, however, the coefficient of $t^r$ in the power series $g\circ f$ is nonzero if and only if any of the summands $\mu_i  \lambda_{j_1} \cdot \hdots \cdot \lambda_{j_i}$ with $j_1+\dots+j_i=r$ is nonzero.	
			Define 
			\begin{align*}
				I(f) &\equalDef \{i\in\N_0\mid \lambda_i \ne 0\}\cup\{-\infty\},\\
				I(g) &\equalDef \{i\in\N_0\mid \mu_i \ne 0\}\cup\{-\infty\}
			\end{align*}
			and it is immediately clear that the coefficient
				$\mu_i  \lambda_{j_1} \cdot \hdots \cdot \lambda_{j_i}$
			is nonzero if and only if
				$i\in I(g)\setminus \{-\infty\}, j_1,\dots,j_i\in I(f)\setminus\{-\infty\}~.$ 
				The reason we add $\{-\infty\}$ to $I(f),I(g)$ is that it allows dealing with cases such as $f=0$ or $f$ odd in a simple manner.
				In order to investigate $\dseven{g\circ f}$ we stepwise simplify:
				\begin{alignat}{3}
				\nonumber	\dseven{g\circ f} &= \max_{i\in I(g)} &&\max_{\substack{j_1,\dots,j_i \in I(f)\\j_1+\dots+j_i \text{ even}}} && j_1+\dots+j_i\\
				\label{eq:highest_f_degrees}	&= \max_{i\in I(g)} &&\max_{\substack{\tilde i \le i\\\tilde i \text{ even}}} && (i-\tilde i) \deven{f} + \tilde i \dodd{f}\\
				\label{eq:highest_g_degrees}	&= \max \big\{ &&\max_{\substack{\tilde i \le \deven g \\\tilde i \text{ even}}} && (\deven g-\tilde i) \deven{f} + \tilde i \dodd{f}, \\
				\nonumber & &&\max_{\substack{\tilde i \le \dodd g \\ \tilde i \text{ even}}} && (\dodd g-\tilde i) \deven{f} + \tilde i \dodd{f} \big\}\\
				\label{eq:degree_calculus_final_form}	&= \max \big\{ &&  \max \{ &&   \deven g \deven{f}, \deven g \dodd{f}\}, \\
				\nonumber	& &&  \max\{ && \dodd g \deven{f}, \deven{f} + (\dodd{g}-1) \dodd f \}\big\}~.
				\end{alignat}
				Here, \cref{eq:highest_f_degrees} is obtained by observing 
				\begin{alignat*}{2}
					j_1+\dots+j_i &\le  \deven{f}+j_2+\dots+j_i \quad &&\text{, if }j_1\text{ is even or respectively}\\
					j_1+\dots+j_i &\le  \dodd{f}+j_2+\dots+j_i \quad &&\text{, if }j_1\text{ is odd}~,
				\end{alignat*}
				where the sum on the right hand side remains even.
			Similarly, \cref{eq:highest_g_degrees} follows by observing for $m\in\N_0$ the inequality
			\begin{align*}
				(i-\tilde i) \deven{f} + \tilde i \dodd{f} &\le \max\big\{ (2m+i-(2m+\tilde i)) \deven{f} + (2m+\tilde i) \dodd{f}, \\
				&\qquad\qquad (2m+i-\tilde i) \deven{f} + \tilde i \dodd{f} \big\}
			\end{align*}
			as by assumption $\deven f\ge 0$ or $\dodd f\ge 1$ hold.
			Finally, \cref{eq:degree_calculus_final_form} follows by observing that extremal behavior in \cref{eq:highest_g_degrees} is obtained by the highest or the lowest valid $\tilde i$. 
			The equation for $\dodd{g \circ f}$ follows analogously.

		2.
		Follows from 1.
\end{proof}

In the following lemma, we investigate the even/odd degree of $\knngp_{L} $ and $\kntk_{L}$ and then deploy \cref{lemma:strictly_positive_Eigenvalues} to characterize which eigenvalues of $\nngp_L,\ntk_L$ are strictly positive. 
\begin{lemma}\label{lemma:nngp_ntk_even_odd_degree_polynomial_case}
	Let the activation function $\act$  fulfill \cref{ass:act} and let $L\ge 2$. 
	We use the conventions $\infty^0 \equalDef	 1$ and $\infty - \infty \equalDef -\infty$.
	Then, the nonzero eigenvalues of the NNGP-kernel $\nngp_L$ and the NTK-kernel $\ntk_L$ defined in \cref{eq:recursive_formulas} fulfill
	\begin{align*}
		\mu_{2k+r,d}(\nngp_L) &> 0 \text{ if and only if }
		\begin{cases}
			2k+r \le  \dseven{{G_L}} &, r=0, \\
			2k+r \le  \dsodd{{G_L}} &, r=1, 
		\end{cases}
		\\
		\mu_{2k+r,d}(\ntk_L) &> 0 \text{ if and only if }
		\begin{cases}
			2k+r \le  \dseven{{H_L}} &, r=0, \\
			2k+r \le  \dsodd{{H_L}} &, r=1~. 
		\end{cases}
	\end{align*}

		For the NNGP-term $G_L$ define $\sigma^2 \equalDef \sigma_i^2\sigma_b^2$, and for the NTK-term $H_L$ define $\sigma^2 \equalDef \sigma_b^2$. That is, the relevant row of the table is different for $G_L$ or respectively for $H_L$ in the case $\sigma_i^2=0, \ \sigma_b^2>0$.
		
		The even/odd degree of $G_L$ or respectively $H_L$ is given by
		\tiny
		\begin{center}
			\begin{tabular}{|c|c|c|}
				\hline 
				& &
				\\[-0.90em]
				& $\dseven{G_L}, \dseven{H_L}$ & $\dsodd{G_L},\dsodd{H_L}$ \\
				& &
				\\[-0.90em]
				\hline
				& &
				\\[-0.90em]
				$\sigma^2 >0$, $\hdeven \act > \hdodd \act$  & $\hdeven\act^{L-1}$ & $\hdeven{\act}^{L-1} -1$\\
				& &
				\\[-0.90em]
				& &
				\\[-0.90em]
				\hline
				& &
				\\[-0.90em]
				$\sigma^2 >0$, $\hdeven \act < \hdodd {\act}$ & $ \hdodd {\act}^{L-1} -1$ & $\hdodd {\act}^{L-1}$ \\
				& &
				\\[-0.90em]\hline
				& &
				\\[-0.90em]
				$\sigma^2 =0, \hdeven \act > \hdodd \act $ & $\hdeven\act^{L-1}$ & $\hdeven{\act}\left( \hdeven{\act}^{L-2} -1 \right) +\hdodd \act$ \\
				& &
				\\[-0.90em]
				\hline
				& &
				\\[-0.90em]
				$\sigma^2 =0, \hdeven \act < \hdodd \act $ & $\hdodd \act\left(\hdodd\act^{L-2}-1 \right) +\hdeven \act$ & $\hdodd{\act}^{L-1}$ \\
				[-0.90em]& & 
				\\
				\hline
			\end{tabular}
		\end{center}
	\normalsize 
\end{lemma}
\begin{proof}
	Let $\kappa:[-1,1] \to \bbR, \kappa(t)= \sum_{j\ge 0} \lambda_i t^j$ be a power series.
	Let $k$ be the corresponding radial kernel on $\bbS^d $ given by $k(x,y) \equalDef \kappa(\scal xy)$ and let $(\mu_{i})_{i\ge 0}$ be the decreasing series of eigenvalues of $k$ counted by algebraic multiplicity.
	
	\textbf{Reduction to polynomial degrees of $\kappa$:} We show that the largest even/odd index $i$ corresponding to a nonzero  eigenvalue $\mu_i\ne 0$ of $k$ is the same as the largest even/odd index $j$ corresponding to a nonzero coefficient $\lambda_j$ of the power series representation of $\kappa$. By convention, if there are infinitely many such indices, we say that the largest one equals $\infty$ and we say that the largest even/odd nonzero index equals $-\infty$ if all corresponding eigenvalues or respectively coefficients are equal to zero.
	
	\begin{itemize}
		\item \cite{hubbert_spherical_2015} show that the eigenvalues $\mu_i$
		have the same sign as the Legendre coefficients $b_k(\kappa)$ obtained by displaying $\kappa$ by Legendre polynomials, which form an orthonormal basis of $L_2([-1,1])$.
		\item The bases of the Legendre polynomials $(p_i)_{i\ge 0}$, the Hermite polynomials $(h_i)_{i\ge 0} $ and the monomials $(t\mapsto t^i)_{i\ge 0}$ share the properties that the $i$-th element is a polynomial of degree $i$ and is it even/odd if $i$ is even/odd. Hence, when we can display $\kappa$ as a sum in those bases, the largest even/odd index $i$ of nonzero coefficients coincide. 
		\item Hence, we see that indeed the largest even/odd index of a nonzero  eigenvalue of $k$ corresponds to the largest even/odd index of a nonzero coefficient of the power series representation of $\kappa$.
	\end{itemize}
	
	Hence, in order to obtain the largest even/odd index of a nonzero eigenvalue of $\nngp_L$ or respectively $\ntk_L$ it suffices to determine the largest even/odd index of the nonzero coefficients of the power series representation of  $G_L$ or respectively $H_L$.  
	Given a nonzero eigenvalue with even/odd index, \cref{lemma:strictly_positive_Eigenvalues} shows that all eigenvalues with smaller even/odd index are strictly positive and we can conclude the claim. 
	Hence, all that remains is to show that $G_L$ and $H_L$ can be represented as a power series and that the even/odd degrees of $G_l$ and $H_l$ are as the table claims.

	\textbf{Simplifying Notation:} By \cref{ass:act} we have $\act\in \calL_2(\calN(0,1))$, enabling us to display $\act$ in the Hermite basis as $\act(t) 
	= \sum_{i\ge 0} \eta_i h_i (t)$.
	\cref{eq:hermite_dual_activation_connection}
	shows
	\begin{align}
		\label{eq:dual_act_again}\widehat{\act} (t) = \sum_{i\ge 0} \eta_i^2 t^i ~.
	\end{align}
	Essentially, we obtain for any $\alpha>0$
	\begin{align*}
		\widehat{\scaledfn  \act {\sqrt\alpha} } \dequal \widehat\act \phdequal \act, \quad \widehat{\scaledfn  {\act'} {\sqrt\alpha} } \dequal \widehat {\act'} \phdequal  \act'~.
	\end{align*}
		This observation allows us to work with $\widehat \act$ or respectively $\widehat {\act'} $ instead of the rescaled versions occurring in $G_l$ and $H_l$, simplifying the notation. 
	By \cref{eq:dual_act_again} all coefficients of the power series of $\widehat{\scaledfn  \act {\sqrt\alpha} }, \widehat{\scaledfn  {\act'} {\sqrt\alpha} } $ are non-negative and hence the coefficients of the power series $G_l,H_l$ are non-negative as well, enabling the use of \cref{lemma:odd/even_pols} in the following. 
	
	\textbf{Calculating $G_l$:}
	We directly obtain
	\begin{align*}
		\degree{G_l} &= \degree{g_l\circ \dots\circ g_1} = \degree{\widehat{\act}}^{l-1}, \quad l\ge 2 ,\\
		g_1 (t) &\dequal t + \sigma_b^2\sigma_i^2 ,\\
		g_l(t) &\dequal {\widehat{\act}}(t) + \sigma_b^2\sigma_i^2, \quad l\ge 2~.
	\end{align*}
	
	If $\deven{{\widehat{\act}}}=\dodd{{\widehat{\act}}}$ holds, then $\deven{{\widehat{\act}}}=\dodd{{\widehat{\act}}}=\infty$ follows and the claim is trivial.

	Assume $\deven{{\widehat{\act}}}>\dodd{{\widehat{\act}}}$. We show the claim for $\dsodd{G_l}$ by induction in $l\ge 2$ and directly obtain the induction base
	\begin{align*}
		\dsodd{G_2} = \begin{cases}
			\deven{{\widehat{\act}}}-1&, \text{ if }\sigma_b^2\sigma_i^2 >0~,\\
			\dodd{{\widehat{\act}}}&, \text{ if }\sigma_b^2\sigma_i^2 >0~.
		\end{cases}
	\end{align*}
	In the induction step,  we have by \cref{lemma:odd/even_pols}
	\begin{align*}
		\dsodd{G_{l+1}} 
		&= \left((\deven{{\widehat{\act}}} -1 ) \cdot \dseven{G_{l}} + \dsodd{G_l} \right)
		\vee 
		\left((\deven{{\widehat{\act}}} -1 ) \dsodd{G_{l}} + \dseven{G_l}\right)\\
		&\quad\vee 
		\left(\dodd{{\widehat{\act}}} \dsodd{G_l} \right)
		\vee \left( (\dodd{{\widehat{\act}}}-1) \dseven{G_l} + \dsodd{G_l}\right) \\
		&= (\deven{{\widehat{\act}}} -1 ) \dseven{G_{l}} + \dsodd{G_l}
	\end{align*}
	and the claim follows.
	The case $\dodd{{\widehat{\act}}}> \deven{{\widehat{\act}}}$ is handled analogously.
	
	\textbf{Calculating $H_l$:}
	
	If $\deven{{\widehat{\act}}}=\dodd{{\widehat{\act}}}$ holds, then $\deven{{\widehat{\act}}}=\dodd{{\widehat{\act}}}=\infty$ follows and the claim is simple.
	
	By definition we have $H_1 (t)\dequal t + \sigma_b^2$.
	For $l\ge 2$ we simplify $H_l(t)$ as
	\begin{align*}
		H_{l}(t)&\equalDef \sigma_b^2(1-\sigma_i^2) + \alpha_{l} G_{l}(t) + \sigma_w^2 \w{\scaledfn{ (\act ')}{\sqrt{\alpha_{l-1}}}}\bra{G_{l-1}(t)} H_{l-1}(t)\\
		&\dequal \sigma_b^2(1-\sigma_i^2) +  \sigma_b^2\sigma_i^2 + \sigma_w^2  \left( \widehat{\act} \circ G_{l-1}(t) + \w{\scaledfn{ (\act ')}{\sqrt{\alpha_{l-1}}}}\bra{G_{l-1}(t)} H_{l-1}(t) \right)\\
		&\dequal G_{l}(t) + {\widehat{\act}}' \bra{G_{l-1}(t)} H_{l-1}(t) +\sigma_b^2~.
	\end{align*}
	A simple induction yields
	\begin{align*}
		\degree{H_l} = \degree{G_l} = \degree{{\widehat{\act}}}^{l-1}
	\end{align*}
	and it remains to calculate  $\dseven{H_l}$ for $\dodd{{\widehat{\act}}}> \deven{{\widehat{\act}}}$ and $\dsodd{H_l}$ for $\deven{{\widehat{\act}}}> \dodd{{\widehat{\act}}}$.
	
	While we usually have $\left|\deven{\w\act'}-\dodd{\w\act'}\right| = 1$, we need some case work to cover the case where one of the degrees is $-\infty$.
	
	\textbf{Case 1: $\sigma_b^2>0$.}
	If $\deven{{\widehat{\act}}}=\dodd{{\widehat{\act}}}$ holds, then $\deven{{\widehat{\act}}}=\dodd{{\widehat{\act}}}=\infty$ follows and the claim is trivial.
	
	Assume $\deven{{\widehat{\act}}} > \dodd {{\widehat{\act}}}$. 
	We show that $\dsodd{H_l} = \deven{{\widehat{\act}}}-1$ holds for all $l\ge 2$.
	For $l=2$ we obtain straightforwardly 
	\begin{align*}
		\dsodd{H_2}= \deven{{\widehat{\act}}} -1.
	\end{align*}
	By induction we obtain 
	\begin{align*}
		\dsodd{H_{l+1}} &\ge  \dsodd{({\widehat{\act}}' \circ G_{l}) \cdot H_{l} } \\
		&\ge \dseven{{\widehat{\act}}' \circ G_l} + \dsodd{H_l} =(\deven{{\widehat{\act}}}-1) \deven{{\widehat{\act}}}^{l-1} + \deven{{\widehat{\act}}}^{l-1} -1 \\
		&= \deven{{\widehat{\act}}}^{l} -1,
	\end{align*}
	and as we have $\degree{{H_{l+1}}} = \deven{{\widehat{\act}}}^l$, the claim follows.
	The case $\dodd{{\widehat{\act}}} > \deven {{\widehat{\act}}}$ works analogously.
	
	\textbf{Case 2: $\sigma_b^2=0$.}
	Assume 
	\begin{align*}
		\dodd{{\widehat{\act}}}> \deven{{\widehat{\act}}}.
	\end{align*}
	\textbf{Case 2.1: }$\dodd{{\widehat{\act}}}= \infty$\textbf{.} 
	We want to show
	\begin{align*}
		\dseven{H_L} = 
		\begin{cases}
			\infty & ,\ \text{if }\dseven{\w{\act}} > -\infty~,\\
			-\infty& ,\ \text{if }\dseven{\w\act} = -\infty~.
		\end{cases}
	\end{align*}
	In the case $\dseven{\w\act}> -\infty$, we have since $\sigma_b^2=0$
	\begin{align*}
			\dseven{H_{L}} &=  \dseven{ G_{L}} \vee \left( \dseven{ ({\widehat{\act}}' \circ G_{L-1}) \cdot H_{L}} \right) 
			\ge \dseven{G_{l+1}}  =\infty
	\end{align*}
	by the previous case.
	In the case $\dseven{\w\act} = -\infty$, $H_L$ is odd by \cref{prop:even_odd_kernels} 
	and hence we have $\dseven{H_L} = -\infty$ as desired.

	\textbf{Case 2.2: }$\dodd{{\widehat{\act}}} < \infty$\textbf{.}
	We show the claim by induction in $l$.
	For $l=2$ we straightforwardly obtain
	\begin{align*}
		\dseven{H_2}= \deven{{\widehat{\act}}}.
	\end{align*}
	In the induction step, let $l\ge 3$ and let 
	$\dseven{H_{l}}= \dodd{{\widehat{\act}}}(  \dodd{\widehat{\act}}^{l-2} - 1 ) +\deven{{\widehat{\act}}}  $ 
	hold as claimed.
	Since $\sigma_b^2 = 0$, we have
	\begin{IEEEeqnarray*}{+rCl+x*}
	\dseven{H_{l+1}} & = & \dseven{ G_{l+1}} 
		\vee \left( \dseven{ ({\widehat{\act}}' \circ G_{l}) \cdot H_{l}} \right) \\
		& = & \left( \dodd{{\widehat{\act}}} \left( \dodd{{\widehat{\act}}}^{l-1}-1\right) +\deven{{\widehat{\act}}} \right)
		\vee \left( \dseven{ {\widehat{\act}}' \circ G_{l} } + \dseven{H_{l} } \right) \\
		&& ~\vee~ \left( \dsodd{{\widehat{\act}}'\circ G_{l} } + \dsodd{H_{l}} \right)~.
	\end{IEEEeqnarray*}
	Now we show that the first term is the largest one in order to obtain the claim. 
	Firstly, we bound the second term as
	\begin{align*}
		\dseven{{\widehat{\act}}' \circ G_{l}} + \dseven{H_{l} } 
		&\le \degree{{\widehat{\act}}' \circ G_{l}} + \dseven{H_{l} } \\
		&= (\degree{{\widehat{\act}}}-1) \cdot \degree{G_{l}} + \dseven{H_{l}} \\
		&= (\degree{{\widehat{\act}}}-1) \cdot \degree{{\widehat{\act}}}^{l-1} + \degree{{\widehat{\act}}}\left(\degree{{\widehat{\act}}}^{l-2} -1\right) +\deven{{\widehat{\act}}}\\
		&= \degree{{\widehat{\act}}}\left(\degree{{\widehat{\act}}}^{l-1} -1\right) +\deven{{\widehat{\act}}} ~.
	\end{align*}
	In order to bound the third term, we recall $\dsodd{H^{l}}=\ooo^{l-1}$, hence we need to show $\dsodd{\w\act'\circ G_l}\le \ooo^l-\ooo^{l-1} +\eee-\ooo$.
	If $\ooo=1$ we can see $\w\act' \equiv c,c\in \R$ and obtain $\dsodd{\w\act\circ G_l}=-\infty$.
	Assume $\ooo\ge 3$.
	In the following we use 
	\begin{align*}
		\dsodd{\w{\act}'} \le \ooo-2 < \ooo-1 = \dseven{\w{\act}'}~.
	\end{align*}
	\Cref{lemma:odd/even_pols} (c) allows unrolling $\dsodd{{\widehat{\act}}'\circ G_{l} }$ as
		\begin{align*}
		\dsodd{{\widehat{\act}}'\circ G_{l} }
		&= \left( \dodd{{\widehat{\act}}'} \cdot \dsodd{G_l} \right)
		\\
		&\quad\vee \left( (\dodd{{\widehat{\act}}'}-1) \cdot \dseven{G_l} + \dsodd{G_l} \right) \\
		&\quad\vee \left( (\deven{{\widehat{\act}}'}-1) \cdot \dsodd{G_l} + \dseven{G_l} \right) 
		\\
		&\quad\vee \left( (\deven{{\widehat{\act}}'}-1) \cdot \dseven{G_l} + \dsodd{G_l} \right) \\
		&\le \left(\ooo-2\right) \ooo^{l-1} 
		\\
		&\quad\vee \left(\ooo-3\right) \left(\ooo^{l-1} + \eee- \ooo \right) +\ooo^{l-1}
		\\
		&\quad\vee \left(\ooo-2\right) \ooo^{l-1} + \ooo^{l-1} +\eee -\ooo\\
		&\quad\vee \left(\ooo-2 \right) \left(\ooo^{l-1} +\eee- \ooo\right) +\ooo^{l-1}\\
		&\le \ooo^l -\ooo^{l-1} +\eee-\ooo ~.
		\end{align*}

	The case $\degree{{\widehat{\act}}}=\deven{{\widehat{\act}}}> \dodd{{\widehat{\act}}}$ is handled analogously. \qedhere
\end{proof}

\section{POLYNOMIAL BOUNDEDNESS OF DERIVATIVES} \label{sec:integrability}

Here, we prove \cref{prop:act_assumptions}. Specifically, we want to find a sufficient criterion for $C^\infty$-functions $f$ such that all of their derivatives $f^{(k)}$ are polynomially bounded. To this end, we leverage that typical activation functions use symbolic expressions whose derivatives can be expressed using the same set of symbols, e.g.\ $\tanh'(x) = 1 - \tanh(x)^2$. Our goal is to find a set of \quot{base functions} whose derivatives use the same base functions and which are polynomially bounded. We then exploit that the class of polynomially bounded functions is closed under addition, multiplication and composition and therefore also contains the derivatives of the \quot{base functions}. Higher-order derivatives are then easily treated using induction. We first formally define the relevant function classes:

\begin{definition} \label{def:function_classes}
	Let $I \subseteq \bbR$ be an interval. For $m \in \bbN_0$, let
	\begin{IEEEeqnarray*}{+rCl+x*}
		\calS^{(m)}(I) & \equalDef & \{f \in C^m(I) \mid \forall 0 \leq k \leq m ~\exists a, b, q > 0 ~\forall x \in I: |f^{(k)}(x)| \leq a|x|^q + b\}~.
	\end{IEEEeqnarray*}
	We note that the class $\calS^{(\infty)}(I)$ from \Cref{def:sinfty} satisfies $\calS^{(\infty)}(I) = \bigcap_{m=0}^\infty \calS^{(m)}(I)$.
\end{definition}

Now, we formally define some \quot{base functions}:
\begin{definition}
	Define $\sigmoid, \softplus, \RBF: \bbR \to \bbR$ by
	\begin{IEEEeqnarray*}{+rCl+x*}
		\sigmoid(x) & \equalDef & (1 + e^{-x})^{-1} \\
		\softplus(x) & \equalDef & \log(1 + e^{-x}) \\
		\RBF(x) & \equalDef & e^{-x^2}~.
	\end{IEEEeqnarray*}
	Moreover, let $\Phi$ be the CDF of the normal distribution $\calN(0, 1)$. 
\end{definition}

\begin{lemma} \label{lemma:act_assumptions}
	Let $I \subseteq \bbR$ be an interval and let $m \in \bbN_0 \cup \{\infty\}$. Then,
	\begin{enumerate}[(a)]
		\item If $m \geq 1$, then $\calS^{(m)} = \{f \in C^1(I) \mid f \in \calS^{(0)}(I) \text{ and } f' \in \calS^{(m-1)}\}(I)$.
		\item $\calS^{(m)}(I)$ is closed under addition, multiplication and composition of functions.
		\item $\calS^{(m)}(I)$ contains all polynomials.
		\item $\calS^{(m)}(I)$ contains $\sigmoid, \tanh, \softplus, \sin, \cos, \RBF$ and $\Phi$.
	\end{enumerate}
	
	\begin{proof}
		We prove the statements for $m \in \bbN_0$, the case $m = \infty$ easily follows.
		\begin{enumerate}[(a)]
			\item This statement is straightforward.
			\item For addition, the statement is trivial. Now, let $f, g \in \calS^{(m)}$. Choose constants $C_1, C_2, q$ such that $|f(x)|, |g(x)| \leq C_1 + C_2 |x|^q$. Then, $|f(x)g(x)| = O(|x|^{2q})$ and $|f(g(x))| \leq C_1 + C_2 |g(x)|^q \leq C_1 + C_2 |C_1 + C_2 |x|^q|^q = O(|x|^{q^2})$. This shows $f \cdot g, f \circ g \in \calS^{(0)}$. If $m = 0$, we are done. Otherwise, we obtain
			\begin{IEEEeqnarray*}{+rCl+x*}
				(f + g)' & = & f' + g' \in \calS^{(m-1)} t\\
				(f \cdot g)' & = & f' \cdot g + f \cdot g' \in \calS^{(m-1)} \\
				(f \circ g)' & = & (f' \circ g) \cdot g' \in \calS^{(m-1)}
			\end{IEEEeqnarray*}
			since $\calS^{(m-1)}$ is closed under addition, multiplication and composition by the induction hypothesis. By (b), we have shown that $f \cdot g, f \circ g \in \calS^{(m)}$.
			\item This is trivial.
			\item It is not hard to see that all of the mentioned functions are in $\calS^{(0)}$. For the induction step, we assume that all of the mentioned functions are in $\calS^{(m-1)}$. Since all of these functions are $C^\infty$ functions, their derivatives are in $C^{m-1}(\bbR)$. Moreover, by inspecting the derivatives and using the previously proven facts as well as the induction hypothesis, it is easy to see that the derivatives are in $\calS^{(m-1)}(\bbR)$:
			\begin{IEEEeqnarray*}{+rCl+x*}
				\sigmoid'(x) & = & \sigmoid(x) \cdot (1 - \sigmoid(x)) = \sigmoid(x) \cdot (x^0 + (-x^0) \cdot \sigmoid(x)) \\
				\tanh'(x) & = & 1 - \tanh(x) \cdot \tanh(x) \\
				\softplus'(x) & = & \sigmoid(x) \\
				\sin'(x) & = & \cos(x) \\
				\cos'(x) & = & -\sin(x) \\
				\RBF'(x) & = & 2x \cdot \RBF(x) \\
				\Phi'(x) & = & \frac{1}{\sqrt{2\pi}} e^{-x^2/2} = \frac{1}{\sqrt{2\pi}} \RBF(2^{-1/2} \cdot x)~.
			\end{IEEEeqnarray*}
			Hence, by (b), we have shown that the functions are in $\calS^{(m)}$. \qedhere
		\end{enumerate}
	\end{proof}
\end{lemma}

\propActAssumptions*

\begin{proof}
\leavevmode
	\begin{enumerate}[(a)]
		\item It follows from elementary calculations that the functions $f(x) = \exp(ax)$ are in $\calS^{(\infty)}((-\infty, 0])$ for all $a \geq 0$. For the other functions, this follows from \Cref{lemma:act_assumptions}.
		\item This follows directly from (a). \qedhere
	\end{enumerate}
\end{proof}

\section{SMOOTHNESS OF GP SAMPLE PATHS} \label{sec:appendix:path_smoothness}

To prove \Cref{theorem:path_smoothness}, we need the following result, which is a variant of Thm.\ 7.4 of \citet{Milan_2001}:

\begin{theorem}[Theorem 1.2 in \citealt{steinwart_gppaths_24}]\label{thm:steinwart_gp_path_variant}
		Let $H_1,H_2$ be Hilbert spaces on a set $T$ and let $X$ be a centered Gaussian process with covariance kernel $k$ such that $H_1$ is the RKHS of $k$.
		Then the following statements hold true:
		\begin{enumerate}[a)]
			\item If $H_1 \hookrightarrow H_2$ and furthermore this embedding is a Hilbert-Schmidt operator, then there exists a version $Y$ of $X$ with $\bbP(Y \in H_2 )= 1$.
			\item Otherwise, for all versions $Y$ of $X$ we have $\bbP(Y\in H_2 )=0$.
		\end{enumerate}
\end{theorem}

Now, we can prove \cref{theorem:path_smoothness}:

\thmPathSmoothness*

\begin{proof}
	We will apply Theorem \ref{thm:steinwart_gp_path_variant} with 
	\begin{align*}
		H_1 := H^{d+\alpha}(\bbS^d)\, , \qquad \mbox{ and } \qquad H_2 := H^{d/2+\alpha+ \eps}(\bbS^d)\, ,
	\end{align*}
	where 
	$\eps \in (-\alpha, \infty)$. 
	
	Here we first note that in the case
	$\eps> d/2$ we have 
	$ H^{d+\alpha}(\bbS^d)\not\hookrightarrow H^{d/2+\alpha+\eps}(\bbS^d)$ and hence \emph{ii)} of 
	Theorem \ref{thm:steinwart_gp_path_variant} gives \emph{i)} of Theorem \ref{theorem:path_smoothness} for such $\eps$.
	
	It thus remains to consider $\eps\in (-\alpha, d/2)$. Since in this case the embedding 
	$ H^{d+\alpha}(\bbS^d)\hookrightarrow H^{d/2+\alpha+\eps}(\bbS^d)$ exists,  
	Theorem \ref{thm:steinwart_gp_path_variant} shows that we need to  investigate whether 
	this embedding   is a Hilbert-Schmidt operator.	To this end we apply \cite{brauchart_characterization_2013} where it is shown 
	that the
	Sobolev space $H^r(\bbS^d)$ is given by
	\begin{align*}
		H^r(\bbS^d) &\equalDef \{ f\in L^2(\bbS^d) \mid \norm{f}_{H^r(\bbS^d)} <\infty\},\\
		\norm f_{H^r(\bbS^d)} &\equalDef \sum_{l=0}^\infty \sum_{i=1}^{N_{l,d}} \lambda_{l,r} \hat f_{l,i}
	\end{align*}
	for a sequence 
	$\lambda_{l,r} = \Theta_{\forall l} (1+l)^{2r}$, where 
	$\hat f_{l,i}$ 
	denotes the Laplace-Fourier coefficients given by
	\begin{align*}
		\hat f_{l,i} \equalDef \int_{\bbS^d} f(x) Y_{l,i} (x) \d\mu_{\bbS^d} (x)~. 
	\end{align*}
	Furthermore $N_{l,d}=\Theta_{\forall l}((l+1)^{d-1})$ holds.
	Hence, $(\lambda_{l,r}^{-1/2} Y_{l,i})_{l\ge 0, 1\le i\le N_{l,d}}$ forms an ONB of $H^r(\bbS^d)$ that diagonalizes the embedding operator.
	The Hilbert-Schmidt norm of the embedding $H^{d+\alpha}(\bbS^d) \hookrightarrow H^{d/2+\alpha+\eps} (\bbS^d)$ is now straightforwardly given, cf. \citet[Chapter 11.3]{Birman_1987}, by
	\begin{align}
		\label{MDS_eq:hs_norm_embedding}
		\norm {H^{d+\alpha}(\bbS^d) \hookrightarrow H^{d/2+\alpha+\eps} (\bbS^d)}_{\text{HS}}
		=\sum_{l=0}^\infty \sum_{i=1}^{N_{l,d}} \frac{\lambda_{l,d/2+\alpha+\eps}}{\lambda_{l,d+\alpha}}~.
	\end{align}
	The asymptotics of $\lambda_{l,s}$ and $N_{l,d}$ 
	yield
	\begin{align*}
		\sum_{i=1}^{N_{l,d}} \frac{\lambda_{l,d/2+\alpha+\eps}}{\lambda_{l,d+\alpha}} = \Theta_{\forall l} (l+1)^{d-1} \frac{\Theta_{\forall l}(l+1)^{d+2\alpha+2\eps}}{\Theta_{\forall l}(l+1)^{2d +2\alpha}} = \Theta_{\forall l} (l+1)^{2\eps-1}
	\end{align*} 
	and consequently, we see that \eqref{MDS_eq:hs_norm_embedding} is finite if and only if $\eps<0$. This shows \emph{ii)} as well as \emph{i)} 
	in the remaining case $0\le \eps\leq  d/2$.
	\qedhere	
\end{proof}

\section{ACTIVATION QUADRATURE} \label{sec:appendix:activation_quadrature}

Here, we derive a way to numerically approximate the dual activation for activations $f$ that are smooth everywhere except possibly at zero. The dual activation involve an integral with the term $f(x)f(y)$, and in order to use Gauss-Legendre quadrature, we need to decompose the integration domain into regions where $f(x)f(y)$ is smooth. Our construction is based on separate quadratures for the four quadrants $x, y \geq 0$, $x, -y \geq 0$, $-x, y \geq 0$, and $-x, -y \geq 0$.

Let $f: \bbR \to \bbR$ and let $\rho \in [-1, 1]$. Moreover, let $X, Y \sim \calN(0, 1)$ be independent and let $\sigma_x \equalDef \sqrt{\frac{1+\rho}{2}}, \sigma_y \equalDef \sqrt{\frac{1-\rho}{2}}$. Define the (centered) random variables
\begin{IEEEeqnarray*}{+rCl+x*}
U & \equalDef & \sigma_x X + \sigma_y Y \\
V & \equalDef & \sigma_x X - \sigma_y Y
\end{IEEEeqnarray*}
Then, 
\begin{IEEEeqnarray*}{+rCl+x*}
\bbE[U^2] & = & \sigma_x^2 + \sigma_y^2 = 1 \\
\bbE[UV] & = & \sigma_x^2 - \sigma_y^2 = \rho \\
\bbE[V^2] & = & \sigma_x^2 + \sigma_y^2 = 1~.
\end{IEEEeqnarray*}
Hence, we have
\begin{IEEEeqnarray*}{+rCl+x*}
\begin{pmatrix}
U \\ V
\end{pmatrix} \sim \calN\left(\begin{pmatrix}
0 \\ 0
\end{pmatrix}, \begin{pmatrix}
1 & \rho \\
\rho & 1
\end{pmatrix}\right)~.
\end{IEEEeqnarray*}
We then want to compute
\begin{IEEEeqnarray*}{+rCl+x*}
\hat f(\rho) = \bbE[f(U, V)] & = & \bbE[f(\sigma_x X + \sigma_y Y, \sigma_x X - \sigma_y Y)]~.
\end{IEEEeqnarray*}
We approximate the integral over a quadrant using the standard normal pdf $\phi$ as
\begin{IEEEeqnarray*}{+rCl+x*}
g(\rho) & \equalDef & \int_{(0, \infty)^2} f(u, v) \diff P_{U, V}(u, v) \\
& = & \int_0^\infty \int_{-\frac{\sigma_x}{\sigma_y} x}^{\frac{\sigma_x}{\sigma_y} x} f(\sigma_x x + \sigma_y y, \sigma_x x - \sigma_y y) \phi(x) \phi(y) \diff y \diff x \\
& \approx & \int_0^{c\sigma_y} \int_{-\frac{\sigma_x}{\sigma_y} x}^{\frac{\sigma_x}{\sigma_y} x} f(\sigma_x x + \sigma_y y, \sigma_x x - \sigma_y y) \phi(x) \phi(y) \diff y \diff x \\
&& ~+~ \int_{c\sigma_y}^{c(\sigma_y+1)} \int_{-c\sigma_x}^{c\sigma_x} f(\sigma_x x + \sigma_y y, \sigma_x x - \sigma_y y) \phi(x) \phi(y) \diff y \diff x \\
& \defEqual & A + B
\end{IEEEeqnarray*}
The restricted integration domains in the second integral $B$ should not cause a large error since $\phi(x)\phi(y)$ is small for $x \geq c(\sigma_y+1) \geq c$ and also for $x \geq c\sigma_y, |y| \geq c\sigma_x$ since either $\sigma_y$ or $\sigma_x$ are $\geq \frac{1}{\sqrt{2}}$. We can now further simplify $A$ using the substitutions $x \defEqual c\sigma_y x'$, $y \defEqual \frac{\sigma_x}{\sigma_y} x y' = c\sigma_x x' y'$:
\begin{IEEEeqnarray*}{+rCl+x*}
A & = & \int_0^1 \int_{-1}^1 c\sigma_y \cdot c\sigma_x x' \cdot f(c\sigma_x \sigma_y x' (1 + y'), c\sigma_x \sigma_y x' (1 - y')) \phi(c\sigma_y x') \phi(c\sigma_x x' y') \diff y' \diff x'~.
\end{IEEEeqnarray*}
Now, we simplify $B$ using the substitutions $x \defEqual c\sigma_y + cx''$ and $y \defEqual c\sigma_x y''$:
\begin{IEEEeqnarray*}{+rCl+x*}
B & = & \int_0^1 \int_{-1}^1 c \cdot c\sigma_x \cdot f(c\sigma_x\sigma_y(1+y'')+c\sigma_x x'', c\sigma_x\sigma_y(1-y'')+c\sigma_x x'') \\
&& \qquad \qquad ~\cdot~\phi(c\sigma_y + cx'') \phi(c\sigma_x y'') \diff y'' \diff x''~.
\end{IEEEeqnarray*}
Both integrals can be approximated using Gauss-Legendre quadrature along both axes.

For the plots in \Cref{fig:deep_spectra} and \Cref{fig:shallow_spectra}, we use $c=12$ and a $50 \times 50$ grid of Gauss-Legendre points for each 2D integral. To compute the eigenvalues from the NTK, we use the Gegenbauer quadrature implementation from \citet{bordelon2020spectrum} with 1000 quadrature points. We provide code for reproducing the figures at \url{https://github.com/dholzmueller/beyond_relu}.

\end{appendices}

\end{document}